\documentclass{article}
 \pdfoutput=1

\PassOptionsToPackage{sort&compress,numbers}{natbib}
\usepackage[preprint]{neurips_2021}
\makeatletter
\renewcommand{\@noticestring}{}
\makeatother
\usepackage{makecell}

\usepackage{booktabs}
\usepackage{multirow,multicol}
\usepackage{subcaption}
\usepackage{IEEEtrantools}
\def\Ieee{IEEEeqnarray*}
\def\Ieeen{\IEEEyesnumber}

\usepackage[utf8]{inputenc} 
\usepackage[T1]{fontenc}    
\usepackage[unicode]{hyperref}       
\usepackage{url}            
\usepackage{booktabs}       
\usepackage{amsfonts}
\usepackage{amssymb} 
\usepackage{amsmath} 
\usepackage{amsthm}     
\usepackage{nicefrac}       
\usepackage{microtype}      
\usepackage[dvipsnames]{xcolor}         
\usepackage{bbm}
\usepackage[linesnumbered,ruled,vlined]{algorithm2e}
\usepackage{graphicx,enumitem}
\usepackage{caption}
\usepackage{subcaption}
\usepackage{epstopdf}
\usepackage{algpseudocode}
\usepackage{wrapfig}
\newcommand{\tmix}{\tau_{\mathsf{mix}}}

\def\A{A^*}
\def\a{\mbox{${a^*}$}}
\def\atr{\mbox{${a^{*,\top}}$}}
\newcommand{\sgdber}{\mathsf{SGD-RER}}
\newcommand{\sgd}{\mathsf{SGD}}
\newcommand{\sgder}{\mathsf{SGD-ER}}

\newcommand{\relu}{\mathsf{ReLU}}
\newcommand{\loss}{\mathcal{L}}
\newcommand{\losssq}{\mathcal{L}_{\mathsf{sq}}}
\newcommand{\lossprox}{\mathcal{L}_{\mathsf{prox}}}
\newcommand{\kl}{\mathsf{KL}}
\newcommand{\tv}{\mathsf{TV}}
\newtheorem{claim}{Claim}
\newtheorem{define}{Definiton}
\newtheorem{proposition}{Proposition}
\newtheorem{lemma}{Lemma}
\newtheorem{assumption}{Assumption}
\newtheorem{theorem}{Theorem}
\newtheorem{remark}{Remark}

\makeatletter
\newcommand{\distas}[1]{\mathbin{\overset{#1}{\kern\z@\sim}}}%
\newsavebox{\mybox}\newsavebox{\mysim}
\newcommand{\distras}[1]{%
  \savebox{\mybox}{\hbox{\kern3pt$\scriptstyle#1$\kern3pt}}%
  \savebox{\mysim}{\hbox{$\sim$}}%
  \mathbin{\overset{#1}{\kern\z@\resizebox{\wd\mybox}{\ht\mysim}{$\sim$}}}%
}
\makeatother
\newcommand{\dprime}{\prime \prime}

\newcommand{\glvar}{\mathsf{NLDS}}
\newcommand{\norm}[1]{\left\| #1 \right\|}
\newcommand{\abs}[1]{\left| #1 \right|}
\newcommand{\lmin}[1]{\lambda_{\min}\left(#1\right)}
\newcommand{\lmax}[1]{\lambda_{\max}\left(#1\right)}
\newcommand{\Xt}[2]{\mbox{${X^{#1}_{#2}}$}}

\newcommand{\Xtt}[2]{\mbox{${\tilde{X}}^{#1}_{#2}$}}
\newcommand{\Xtttr}[2]{\mbox{${\tilde{X}}^{#1,\top}_{#2}$}}

\newcommand{\Nt}[2]{\mbox{${\varepsilon^{#1}_{#2}}$}}

\newcommand{\Htt}[3]{\mbox{${\tilde{H}^{#1}_{#2,#3}}$}}
\newcommand{\Htttr}[3]{\mbox{${\tilde{H}^{#1,\top}_{#2,#3}}$}}

\newcommand{\prodHtt}[2]{\left(\prod_{s=#2}^{1}\Htt{#1-s}{0}{B-1}\right)}
\newcommand{\prodHtttr}[2]{\left(\prod_{s=1}^{#2}\Htttr{#1-s}{0}{B-1}\right)}

\newcommand{\Ptt}[2]{\left(I-2\gamma \phi'\left(\tilde \xi^{#1}_{#2}\right) \Xtt{#1}{#2}\Xtttr{#1}{#2}\right)}

\newcommand{\Ppt}[2]{\mbox{${\tilde{P}^{#1}_{#2}}$}}

\def\cc{\mathcal{C}}
\def\cd{\mathcal{D}}
\def\ce{\mathcal{E}}

\def\cR{\mathcal{R}}
\def\cA{\mathcal{A}}
\def\cV{\mathcal{V}}
\def\cVt{\tilde \cV}

\newcommand{\ind}[2]{1\left[\cd^{#1,#2}\right]}
\newcommand{\indc}[2]{1\left[\cd^{#1,#2,C}\right]}
\def\cdt{\tilde\cd}
\def\cct{\tilde \cc}

\newcommand{\indt}[2]{1\left[\tilde \cd^{#1,#2}\right]}
\newcommand{\indtc}[2]{1\left[\tilde \cd^{#1,#2,C}\right]}
\def\cdh{\hat\cd}
\newcommand{\indh}[2]{1\left[\cdh^{#1,#2}\right]}

\def\gammah{\hat{\gamma}}
\DeclareMathOperator{\tr}{Tr}
\DeclareMathOperator{\cro}{Cr}
\DeclareMathOperator{\poly}{Poly}

\DeclareMathOperator{\dg}{Dg}

\def\prbnd{\frac{1}{T^{\alpha}}}
\def\prbndH{\frac{1}{2T^{\alpha}}}
\def\prbndsq{\frac{1}{T^{\alpha/2}}}

\newcommand{\Pb}[1]{\mathbb{P}\left[#1 \right]}

\newcommand{\Ex}[1]{\mathbb{E}\left[#1 \right]}

\def\ahto{\hat a_{t_0,N} }

\def\ahtto{\hat{\tilde{a}}_{t_0,N} }
\def\ahttv{\hat{\tilde{a}}^v_{t_0,N} }
 \def\ahttb{\hat{\tilde{a}}^b_{t_0,N} }
 \renewcommand{\intercal}{\top}
\title{Near-optimal Offline and Streaming Algorithms for Learning Non-Linear Dynamical Systems}

\author{%
  Prateek Jain \\
  Google AI Research Lab,\\
  Bengaluru, India 560016 \\
  \texttt{prajain@google.com} \\
    \And
  Suhas S Kowshik \\
  Department of EECS\\
  MIT,\\
  Cambridge, MA 02139 \\
  \texttt{suhask@mit.edu} \\
  \AND
  Dheeraj Nagaraj \\
  Department of EECS\\
  MIT,\\
  Cambridge, MA 02139 \\
  \texttt{dheeraj@mit.edu} \\
  \And
  Praneeth Netrapalli \\
  Google AI Research Lab,\\
  Bengaluru, India 560016 \\
  \texttt{pnetrapalli@google.com} \\
  }

\begin{document}

\maketitle

\begin{abstract}

	We consider the setting of vector valued non-linear dynamical systems $X_{t+1} = \phi(\A X_t) + \eta_t$, where $\eta_t$ is unbiased noise and $\phi : \mathbb{R} \to \mathbb{R}$ is a known link function that satisfies certain {\em expansivity property}.
	The goal is to learn $\A$ from a single trajectory $X_1,\cdots,X_T$ of {\em dependent or correlated} samples.
	While the problem is well-studied in the linear case, where $\phi$ is identity, with optimal error rates even for non-mixing systems, existing results in the non-linear case hold only for mixing systems. In this work, we improve existing results for learning nonlinear systems in a number of ways: a) we provide the first offline algorithm that can learn non-linear dynamical systems without the mixing assumption, b) we significantly improve upon the sample complexity of existing results for mixing systems, 
	c) in the much harder one-pass, streaming setting we study a SGD with Reverse Experience Replay ($\sgdber$) method, and demonstrate that for mixing systems, it achieves the same sample complexity as our offline algorithm, d) we justify the expansivity assumption by showing that for the popular ReLU  link function --- a non-expansive but easy to learn link function with i.i.d. samples --- any method would require exponentially many samples (with respect to dimension of $X_t$) from the dynamical system. We validate our results via simulations and  demonstrate that a naive application of SGD can be highly sub-optimal. Indeed, our work demonstrates that for correlated data, specialized  methods designed for the dependency structure in data can  significantly outperform  standard SGD based methods. 	

\end{abstract}

\section{Introduction}

Non-linear dynamical systems (NLDS) are commonly used to model the data in a variety of domains like control theory, time-series analysis, and reinforement learning (RL) \cite{vidyasagar2002nonlinear,chen1989representations,elman1990finding,jordan1997serial}. Standard NLDS models the data points $(X_0,X_1,\dots,X_T)$ as: 
\begin{equation}\label{eq:master_equation}
	X_{t+1} = \phi(\A X_t) + \eta_t,
\end{equation}
where $X_{t} \in \mathbb{R}^d$ are the states, $\eta_t \in \mathbb{R}^d$ are i.i.d. noise vectors, $\A \in \mathbb{R}^{d\times d}$ and $\phi :\mathbb{R} \to \mathbb{R}$ is an increasing function called the `link function'. Here, $\phi$ is supposed to act component wise over $\mathbb{R}^d$.

{\em System identification} problem is a foundational problem for NLDS, i.e., given $(X_0,X_1,\dots,X_T)$ generated from \eqref{eq:master_equation}, the goal is to estimate $\A$ accurately from a single trajectory $(X_0,X_1,\dots,X_T)$. The system identification problem is heavily studied in control theory \cite{ljung1999system,aastrom1971system,campi2002finite,vidyasagar2006learning} as well as time-series analysis \cite{hall2016inference}. For instance, the non-linear dynamical system considered here has an application in modeling non-linear distortions in power amplifiers \cite{wood2014behavioral}.
The problem is challenging as data points $X_0,X_1,\dots,X_T$ are not i.i.d. as usually encountered in machine learning, but form a Markov process. If the mixing time $\tmix$ of the process is finite ($\tmix<\infty$), then we can make the data approximately i.i.d. by considering only the points separated by $\tilde{O}(\tmix)$ time. While this allows using standard techniques for i.i.d. data,  it reduces the effective number of samples to $O(\frac{T}{\tmix})$, which typically gives an error of the order $O(\frac{\tmix}{T})$. In fact, even the state-of-the-art results   have error bounds which are sub-optimal by a factor of $\tmix$. 

Interestingly, for the special case of linear systems, i.e., when $\phi(x)=x$, the results are significantly stronger. For example, \cite{simchowitz2018learning,sarkar2019near} showed that the matrix $\A$ can be estimated with an error $O(1/T)$ even when the mixing time $\tmix > T$. But these results rely on the fact that for linear  systems, the estimation problem reduces to an ordinary least squares (OLS) problem for which a closed form expression is available and can be analyzed effectively. 

On the other hand, NLDS do not admit such closed form expressions. In fact the existing techniques mostly rely on mixing time arguments to induce i.i.d. like behavior in a subset of the points which leads to sub-optimal rates by $\tmix$ factor. Similarly, a direct application of uniform convergence results \cite{shalev2009stochastic} to show that the minimizer of the empirical risk is close to the population minimizer still gives sub-optimal rates as off-the-shelf concentration inequalities (cf. \cite{paulin2015concentration}) incur an additional factor of mixing time. Finally, existing results are mostly focused on offline setting, and  do not apply to the case where the data points are streaming which is critical in several practical problems like reinforcement learning (RL) and control theory. 

In this work, we provide algorithms and their corresponding error rates for the NLDS system identification problem in both offline and online setting, assuming the link function to be expansive (Assumption~\ref{assump:3}). The main highlight of our results is that the error rates are {\em independent} of the mixing time $\tmix$, which to the best of our knowledge is first such result for any non-linear system identification in any setting. In fact, for offline setting, our analysis holds even for systems which do not mix within time $T$ and even for marginally stable systems which do not mix at all. Furthermore, we analyze SGD-Reverse Experience Replay (SGD-RER) method, we provide the first streaming method for NLDS identification with  error rate that is independent of $\tmix$ (in the leading order term) while still ensuring small space and time complexity. This algorithm was first discovered in the experimental RL setting in \cite{rotinov2019reverse} based on Hippocampal reverse replay observed in biological networks \cite{ambrose2016reverse,haga2018recurrent,whelan2021robotic}. It was introduced independently in \cite{jain2021streaming} for the case of linear systems and efficiently unravels the complex dependency structure present in the problem. Finally, through a lower bound for ReLU---a non-expansive function---we provide strong justification for why expansivity might be necessary for a non-trivial result. 

Instead of mixing time arguments, our proofs for learning NLDS without mixing use a natural exponential martingale of the kind considered in the analysis of self normalized process (\cite{abbasi2011online,pena2008self}). For streaming setting, while we do use mixing time arguments (proof of Theorems~\ref{thm:newton_ub_heavy_tail} and~\ref{thm:sgd_rer_ub}), we combine them with a delicate stability analysis of the specific algorithm and the machinery developed in \cite{jain2021streaming} to obtain strong error bounds. See  Section~\ref{sec:ideas_behind_proofs} for a description of these techniques.

{\bf Our Contributions.} Key contributions of the paper are summarized below:\vspace*{-5pt} 
\begin{enumerate}[leftmargin=*]
\item Assuming  expansive and monotonic link function $\phi$ and sub-Gaussian noise, we show that the offline Quasi Newton Method (Algorithm~\ref{alg:modified_newton_method}) estimates the parameter $\A$ with near optimal errors of the order $O(1/T)$, even when the dynamics does not mix within time $T$.
\item Assuming mixing NLDS, finite fourth moment on the noise, and expansive monotonic link function, we show that offline Quasi Newton Method again estimates the parameter $\A$ with near-optimal error of $O(1/T)$, independent of mixing time $\tmix$. 
\item  We give a one-pass, streaming algorithm inspired by $\sgdber$ method by \cite{jain2021streaming}, and show that it achieves near-optimal error rates under the assumption of sub-Gaussian noise, NLDS stability (see section \ref{subsec:assumptions} for the definition), uniform expansivity and second differentiability of the link function.  
\item We then show that learning with $\relu$ link function, which is non-expansive but is known to be easy to learn with if data points are all i.i.d. \cite{diakonikolas2020approximation}, requires exponential (in $d$) many samples. 
\end{enumerate} 

We believe that the techniques developed in this work can be extended to provide efficient algorithms for learning with dependent data in more general settings. 

\paragraph{Related Works.}
NLDS has been studied in a variety of domains like time-series and recurrent neural networks (RNN). \cite{hall2016inference} studies specific NLDS models from a time series perspective and establishes non-asymptotic convergence bounds for natural estimators; their error rates suffer from mixing time factor $\tmix$.   \cite{chen1990non} considers asymptotic learning of NLDS via neural networks trained using SGD, whereas \cite{allen2018convergence} shows that overparametrized LSTMs trained with SGD learn to memorize the given data.  \cite{bahmani2019convex,oymak2019stochastic,miller2018stable} consider learning dynamical systems of the form $h_{t+1} = \phi(\A h_t + B^{*}u_t)$ for states $h_t$ and inputs $u_t$; this setting is different from standard NLDS model we study. \cite{mania2020active} considers the non linear dynamical systems of the form $x_{t+1}=A\phi(x_t,u_t)+\eta_t$ which $\phi$ is a known non-linearity and matrix $A$ is to be estimated. \cite{sarker2020parameter,mao2020finite} consider essentially linear dynamics but allow for certain non-linearities that can be modeled as process noise. All these again differ from the model we consider. 

Standard NLDS identification \eqref{eq:master_equation} has received a lot of attention recently, with results by \cite{sattar2020non,foster2020learning} being the most relevant. \cite{sattar2020non} uses uniform convergence results via. mixing time arguments to  obtain  parameter estimation error for offline SGD. \cite{foster2020learning} obtains similar bounds via. the analysis of the GLMtron algorithm \cite{kakade2011efficient}. However, both these works suffer from sub-optimal dependence on the mixing time. We refer to Table~\ref{tb:summary-table} for a comparison of the results. \cite{gao2021improved}, which appeared after the initial manuscript of this work, considers the question from a perspective of time series forecasting. This work considers sparsity in $\A$ and an unknown link function which is estimated with isotonic regression.  Their recovery guarantees eschew the mixing time dependence. However, the setting, assumptions and the error rates are incomparable to our setting. 

\cite{foster2020learning} also obtains within sample prediction error in the case when $\phi$ is not uniformly expansive along with parameter recovery bounds when $\phi$ is the $\relu$ function and the driving noise is Gaussian. However, the parameter estimation bounds for $\relu$ suffer from an exponential dependence on the dimension $d$ and mixing time $\tmix$. In Theorem~\ref{thm:relu_exp_lb} we establish that indeed we cannot improve the exponential dependence in the dimension $d$ for the case of parameter estimation. We note that the exponential dependence arises due to the dynamics present in the system since $\relu$ regression with isotropic i.i.d. data in well specified case has only a polynomial dependence in $d$ \cite{diakonikolas2020approximation}. 



Linear system identification (LSI) literature has been well studied with strong minimax optimal bounds \cite{simchowitz2018learning,simchowitz2019learning,sarkar2019finite}. These results primarily consider the (convex) empirical square loss which has a closed form solution. However, the square loss in the non-linear case is non-convex. Under the assumption that the link function is increasing, we  consider a convex proxy loss which is widely used in generalized linear regression literature \cite{kalai2009isotron,kakade2011efficient,diakonikolas2020approximation}. Similarly, GLMtron algorithm for learning NLDS in\cite{foster2020learning} (see Equation~\eqref{eq:proxy_loss_def}) also considers a similar proxy loss. In \cite{boffi2020reflectron}, the authors consider a family of GLMtron-like algorithms call Reflectron under the i.i.d. data setting. But they compare the performance of these algorithms experimentally on an NLDS similar to one considered in this work under low rank assumption on the system matrix. 

Finally, streaming setting for LSI has been recently studied in different model settings \cite{jain2021streaming,nagaraj2020least}. These methods observe that by exploiting techniques like experience replay (\cite{lin1992self}) along with squared loss error, one can obtain strong error rates. $\sgdber$ method studied in this work is inspired by a similar method by \cite{jain2021streaming} which was primarily studied for the linear case. 

 
\begin{table*}[t]
\vskip 0.15in
\centering
			\begin{sc}
				\resizebox{\columnwidth}{!}{\begin{tabular}{|c|c|c|c|c|c|}
					\hline	
					{\bf Paper} & {\bf Guarantee} & {\bf Link Function} & {\bf System} & {\bf Noise} & {\bf Algorithm} \\
					\hline					
					\makecell{\cite{sattar2020non} \\ Theorem 6.2}& ${\color{red} \frac{d^2\tmix}{T}}$&
					\makecell{\scriptsize Increasing,Lipschitz \\ \scriptsize Expansive} & \scriptsize{\color{red}Mixing}& \scriptsize{\color{red}Sub-Gaussian} &\scriptsize{\color{red} Offline}\\
					\hline
					\makecell{\cite{foster2020learning} \\ Theorem 2 }& ${\color{red} \frac{d^2\tmix}{T}}$ & \shortstack{\begin{tabular}{@{}c@{}}\scriptsize Increasing,Lipschitz \\
					\scriptsize Expansive \end{tabular}} & \scriptsize{\color{red}Mixing}&\scriptsize{\color{red}Sub-Gaussian}&\scriptsize{\color{red} Offline} \\
					\hline 
					\makecell{\textbf{This paper} \\ Theorem~\ref{thm:newton_ub_without_mixing}} & {\color{Green}$\frac{d^2\sigma^2}{T\lmin{\hat{G}}}$} &\shortstack{\begin{tabular}{@{}c@{}}\scriptsize Increasing,Lipschitz \\
					\scriptsize Expansive  \end{tabular}}&\scriptsize{\color{Green}Non-Mixing} &  \scriptsize{\color{red}Sub-Gaussian}  &\scriptsize{\color{red} Offline}\\
\hline
					\makecell{\textbf{This paper} \\ Theorem~\ref{thm:newton_ub_heavy_tail}} & {\color{Green}$\frac{d^2\sigma^2}{T\lmin{G}}$} &\shortstack{\begin{tabular}{@{}c@{}}\scriptsize Increasing,Lipschitz \\
					\scriptsize Expansive \end{tabular}} & \scriptsize{\color{red}Mixing} &  \scriptsize{\color{Green}$4$-th Moment}&\scriptsize{\color{red} Offline}\\
					
\hline
					\makecell{\textbf{This paper} \\ Theorem~\ref{thm:sgd_rer_ub}} & {\color{Green}$\frac{d^2\sigma^2}{T\lmin{G}}$} &\shortstack{\begin{tabular}{@{}c@{}}\scriptsize Increasing,Lipschitz \\
					\scriptsize Expansive \\
					\scriptsize {\color{red}Bounded Second Derivative} \end{tabular}} &\scriptsize{\color{red}Mixing} &  \scriptsize{\color{red}Sub-Gaussian} &\scriptsize{\color{Green} Streaming}  \\
					\hline
				\end{tabular}}
			\end{sc}

	\caption{Comparison of our results with existing results in terms of mixing time $\tmix$, stablility and number of samples $T$. Here, we take $\tmix = \tilde{\Omega}(\frac{1}{1-\|\A\|_{op}})$ as a proxy for the mixing time. Note that $\lmin{G} \geq \sigma^2$ in the worst case, and hence our bounds are better by a factor of $\tmix$.}
	\label{tb:summary-table} 
	\vskip -0.1in
\end{table*}

\section{Problem Statement}
\label{sec:problem_statement}
Let $\phi :\mathbb{R} \to \mathbb{R}$ be an increasing, $1$-Lipschitz function such that $\phi(0) = 0$. Suppose $X_0 \in \mathbb{R}^d$ is a random variable and $\A \in \mathbb{R}^{d\times d}$. We consider the following non-linear dynamical system (NLDS): 
\begin{equation}\label{eq:evolution_equation}
X_{t+1} = \phi(\A X_t) + \eta_t, 
\end{equation}
where the noise sequence $\eta_0,\dots,\eta_T$ is i.i.d random vectors independent of $X_0$. The noise $\eta_t$ is such that  $\mathbb{E}\eta_t = 0$, $\mathbb{E}\eta_t\eta_t^{\intercal} = \sigma^2 I$ for some $\sigma > 0$. We will also assume that $M_4:= \mathbb{E}\|\eta_t\|^4 < \infty$.  Let $\mu$ be the law of noise $\eta$. We denote the model above as $\glvar(\A,\mu,\phi)$. Whenever a stationary distribution exists for the process, we will denote it by $\pi(\A,\mu,\phi)$ or just $\pi$ when the process is clear from context. We will call the trajectory $X_0,X_1,\dots,X_T$ `stationary' if $X_0$ is distributed according to the measure $\pi(\A,\mu,\phi)$. Unless specified otherwise, we take $X_0 = 0$ almost surely.

The goal is to estimate $\A$ given a single trajectory $X_0,X_1,\dots,X_T$. A natural approach would be to minimize the empirical square loss, i.e,
$\losssq(A;X) := \frac{1}{T}\sum_{t=0}^{T-1}\|\phi(AX_t) - X_{t+1}\|^2$. However, when the link function $\phi$ is not linear, then this would be non-convex and hard to optimize. Instead, we use a convex proxy loss given by: 
\begin{equation}\label{eq:proxy_loss_def}
\lossprox(A;X) = \frac{1}{T}\sum_{t=0}^{T-1}\sum_{i=1}^{d}\bar{\phi}(\langle a_i,X_{t}\rangle ) - \langle e_i,X_{t+1}\rangle \langle a_i,X_t\rangle \,,
\end{equation}
where $\bar{\phi}$ is the indefinite integral of the link function $\phi$ and $a_i$ is the $i$-th row of $A$. Note that the gradient of $\lossprox(A;X)$ with respect to $A$ is given by:
\begin{equation}\label{eq:proxy_loss_grad}
\nabla\lossprox(A;X) = \frac{1}{T}\sum_{t=0}^{T-1}\left(\phi(AX_{t}) - X_{t+1}\right) X_t^{\intercal} \,.
\end{equation}
When the model is clear from context and the stationary distribution exists, we will denote the second moment matrix under the stationary distribution by $G := \mathbb{E} [X_t X_t^{\top}]$.  Note that $G \succeq \mathbb{E}[\eta_t\eta_t^{\intercal}] = \sigma^2 I$. Also, the empirical second moment matrix is denoted by $\hat{G}:= \frac{1}{T} \sum_{t=0}^{T-1}X_tX_t^{\intercal}$.

\subsection{Assumptions}
\label{subsec:assumptions}
We now state the assumptions below and use only a subset of the assumptions for each result. 
\begin{assumption}[Lipschitzness and Uniform Expansivity]
\label{assump:3}
$\phi$ is 1-Lipschitz and $|\phi(x)-\phi(y)| \geq \zeta|x-y|$, for some $\zeta > 0$.
\end{assumption}
Note that when $\phi$ is only weakly differentiable but satisfy Assumption~\ref{assump:3}, with a slight abuse of notation, we will write down $\phi(x) - \phi(y) = \phi^{\prime}(\beta)(x-y)$ for some $\phi^{\prime}(\beta) \in [\zeta,1]$.
\begin{assumption}[Bounded 2nd Derivative]\label{assump:8}
$\phi$ is twice continuously differentiable and $|\phi^{\dprime}|$ is bounded. 
\end{assumption}

\begin{assumption}[Noise Sub-Gaussianity]\label{assump:7}
 For any unit norm vector $x \in \mathbb{R}^d$, we have $\langle\eta_t,x\rangle$ to be sub-Gaussian with variance proxy $C_{\eta}\sigma^2$.
\end{assumption}

Next, we extend the definition of exponential stability in \cite{sattar2020non} to `exponential regularity' to allow unstable systems.
\begin{assumption}[Exponential Regularity]
\label{assump:4}
Let $X_{T} = h_{T-1}(X_0,\eta_0,\dots,\eta_T)$ be the function representation of $X_T$. We say that $\glvar(\A,\mu,\phi)$ is $(C_{\rho},\rho)$ exponentially regular if for any choice of $T\in \mathbb{N}$ and $X_0,X_0^{\prime},\eta_0,\dots,\eta_T \in \mathbb{R}^d$:
$$\|h_T(X_0,\eta_0,\dots,\eta_T)-h_T(X^{\prime}_{0},\eta_0,\dots,\eta_T)\|_2 \leq C_\rho \rho^{T-1}\|X_0 - X_0^{\prime}\|_2\,.$$

When $\rho < 1$, we will call the system stable. When $\rho = 1$ we will call it `possibly marginally stable' and when $\rho > 1$, we will call it `possibly unstable'.
\end{assumption}
Note that when Assumption~\ref{assump:4} holds with $\rho < 1$, the system necessarily mixes and converges to a stationary distribution as $T \to \infty$. Such systems forget their initial conditions in time scales of the order $\tau_{\mathsf{mix}} =O(\frac{1+\log C_{\rho}}{\log \tfrac{1}{\rho}}) = O\left(\frac{1+\log C_{\rho}}{1-\rho}\right)$, and hence we use this as a proxy for the mixing time. In what follows, when we say `the system does not mix' we either mean that it does not mix within time $T$ or it does not converge to a stationary distribution (ex: $\rho \geq 1$).
\begin{assumption}[Norm Boundedness]
\label{assump:6}
$\|\A\|_{\mathsf{op}} = \rho < 1$
\end{assumption}
That is, if $\A$ satisfies Assumption~\ref{assump:6}, we have for arbitrary $X,X^{\prime} \in \mathbb{R}^d$:  $\norm{\phi(\A X)-\phi(\A X')}\leq \rho \norm{X-X'}$ and $\norm{(\phi\circ \A)^k (X) }\leq \rho^k \norm{X}$. Hence, for such $\A$, $\glvar$ is {\em necessarily stable}.

\section{Offline Learning with Quasi Newton Method}
\label{sec:offline_learning}
\begin{algorithm}[t!]
	\DontPrintSemicolon
	\SetKwInOut{Input}{Input}
	\SetKwInOut{Output}{Output}
	\SetKwFunction{RN}{ReadNext}\SetKwFunction{LN}{LeaveNext}
	\Input{Offline data $\{X_0,\dots,X_T\}$, horizon $T$, no. of iterations $m$, link function $\phi$, step size $\gamma$}
	\Output{Estimate $ A_{m}$}
	\Begin{
		$A^0_0=0$ \textsf{/*Initialization*/}\;
		$\hat G \leftarrow \frac{1}{T}\sum_{t=0}^{T-1}X_tX_t^{\intercal}$; If $\hat G$ is not invertible, then \textbf{return} $A_m = 0$\;

			\For{$i\leftarrow 0$ \KwTo $m-1$}{
				\nl $A_{i+1}\leftarrow A_i - 2 \gamma \left(\nabla\lossprox(A_i;X)\right)\hat{G}^{-1} $
			}
		}
	\caption{Quasi Newton Method}
	\label{alg:modified_newton_method}
\end{algorithm}
In this section we consider estimating $\A$ using a single trajectory $(X_1,\dots,X_T)$ from $\glvar(\A,\mu,\phi)$. To this end, we study an offline Quasi Newton Method (Algorithm~\ref{alg:modified_newton_method}) where the iterates descend in the directions of the gradient of $\lossprox$ normalized by the inverse of the empirical second moment matrix $\hat{G} := \frac{1}{T}\sum_{t=0}^{T-1}X_tX_t^{\intercal}$. That is, the iterates follow an approximation of the standard Newton update. 

We now present analysis of Algorithm~\ref{alg:modified_newton_method} in two settings: a) Theorem~\ref{thm:newton_ub_without_mixing} provides estimation error for possibly unstable systems with sub-Gaussian noise that is close to the minimax optimal error incurred in the linear system identification case, b) Theorem~\ref{thm:newton_ub_heavy_tail} provides similarly tight estimation error for mixing systems but with heavy-tailed noise.

\begin{theorem}[Learning Without Mixing]\label{thm:newton_ub_without_mixing}
 Suppose Assumptions~\ref{assump:3},~\ref{assump:7} and~\ref{assump:4} hold with expansivity factor $\zeta$ and regularity parameters $(C_\rho,\rho)$. Let $\bar{C},\bar{C}_3$  be constants depending only on $C_{\eta}$, and let $\delta \in (0,\tfrac{1}{2})$. 
 Let $R^{*}:=C_{\rho}^2 C_{\eta}d\sigma^2\left(\sum_{t=1}^{T-1}\rho^t\right)^2 \log(\tfrac{4Td}{\delta}) $, and assume 
\begin{enumerate}
 \item The number samples $T \geq \bar{C}_3 \left(d\log\left(\frac{R^{*}}{\sigma^2}\right) + \log\tfrac{1}{\delta}\right)$
 \item Step size $\gamma=\tfrac{1}{4}$
 \item $m\geq\frac{10}{\zeta}\cdot \log \left(\frac{\|A_0 -\A\|^2_{\mathsf{F}}\cdot TR^{*}}{\sigma^2d^2}\right)$
\end{enumerate}  
 Then, the output $A_m$ of Algorithm~\ref{alg:modified_newton_method} after $m$ iterations and $\lmin{\hat{G}}$ satisfy with probability at-least $\geq 1-\delta$:

$$\|A_{m}-\A\|^2_{\mathsf{F}} \leq  \tfrac{\bar{C}\sigma^2}{T\zeta^2\lmin{\hat{G}}}\left[d^2 \log\left(1+\tfrac{R^{*}}{\sigma^2}\right) + d\log\left(\tfrac{2d}{\delta}\right)\right] \,,$$
$$ \lmin{\hat{G}} \geq \frac{\sigma^2}{2}.$$

\end{theorem}

Note that as $\lambda_{\min}(\hat{G})\gtrsim \sigma^2$, the error rate scales as $\approx d^2/T$, independent of $\tmix\approx 1/(1-\rho)$. The theorem also holds for non-mixing or possibly unstable systems as long as $\rho < 1 + \tfrac{C}{T} $. Furthermore, the error bound above is similar to the {\em minimax optimal bound} by \cite{simchowitz2018learning} for the {\em linear} setting, i.e., when $\phi(x)=x$. As the link function $\phi$ tends to decrease the information in $x$, intuitively lower bound for linear setting should apply for NLDS as well, which would imply our error rate to be optimal; we leave further investigation into lower bound of NLDS identification for future work. Interestingly, in the linear case whenever the smallest singular value $\sigma_{\min}(\A) > 1+ \epsilon$, it can be show than $\lambda_{\min}(\hat{G})$ grows exponentially with $T$, leading to an exponentially small error. It is not clear how to arrive at such a growth lower bound in the non-linear case. 

The computational complexity of the algorithm scales as $m\cdot T$ which depends only logarithmically on $\tmix$. Interestingly, the algorithm is almost hyperparameter free, and does not require knoweldge of parameters $\sigma, \tmix, \zeta$. 

 Also note that the stationary points of Algorithm~\ref{alg:modified_newton_method} and GLMtron (~\cite{foster2020learning}) are the same. So, the stronger error rate in the result above compared to the result by \cite{foster2020learning} is due to a sharper analysis. However, in dynamical systems of the form~\ref{eq:master_equation}, the squared norm of the iterates grow as $\frac{d}{1-\rho}$ even in the stable case. Hence, the GLMtron algorithm requires step sizes to be $\approx \frac{1-\rho}{d}$ which implies significantly slower convergence rate for large $\tmix=1/(1-\rho)$. In contrast, convergence rate for Algorithm~\ref{alg:modified_newton_method} depends at most logarithmically on $\tmix$.

\begin{theorem}[Learning with Heavy Tail Noise]\label{thm:newton_ub_heavy_tail}
Suppose Assumptions~\ref{assump:3} and~\ref{assump:4} hold. In Assumption~\ref{assump:4}, let $\rho < 1$. Let $X_0,\dots,X_T$ be a stationary trajectory drawn from $\glvar(\A,\mu,\phi)$ and $A_m$ be the $m$-th iterate of Algorithm~\ref{alg:modified_newton_method}. For some universal constants $C,C_1,C_0 > 0$, whenever $\delta \in (0,\tfrac{1}{2})$, $R^{*} :=  \frac{4TdC_{\rho}^2\sigma^2}{(1-\rho)^2 \delta}$ and

\begin{enumerate}
\item  $T \geq Cd\log(\tfrac{1}{\delta})\log(\tfrac{R^{*}}{\sigma^2}) \max\left(\frac{4C_{\rho}^6M_4}{(1-\rho)^4(1-\rho^2)\sigma^4}, \tfrac{\log\left( \tfrac{R^{*} C_1 C_{\rho}}{\sigma^2}\right)}{\log\left(\tfrac{1}{\rho}\right)}\right)$
\item Step size $\gamma = \frac{1}{4}$
\item $m\geq\frac{10}{\zeta}\cdot \log \left(\frac{\|A_0 -\A\|^2_{\mathsf{F}}\cdot TR^{*}}{\sigma^2d^2}\right)$
\end{enumerate}
There exists an event $\mathcal{W} \in \sigma(X_0,\eta_0,\dots,\eta_{T-1})$ with $\mathbb{P}(\mathcal{W}) \geq 1-\delta$ and :
$$\mathbb{E}\left[\|A_m - \A\|^2_{\mathsf{F}}\mathbbm{1}(\mathcal{W})\right] \leq   \frac{C_0d^2\sigma^2}{\zeta^2 T \lmin{G}}.$$
\end{theorem}
\paragraph{Obtaining High Probability Bounds:}
The bound above shows that the expectation of the error restricted to a high probability set is small. This, along with Markov inequality, shows that we can have an error of at-most $ \frac{Cd^2\sigma^2}{\zeta^2 T\lmin{G}}$ with probability at-least $\frac{2}{3}$. This can be boosted to a high probability bound by splitting the horizon $T$ into $K$ contiguous segments with a gap of $O(\tmix \log T)$ to maintain approximate independence (see Section~\ref{subsec:coupling}). We then run the Quasi Newton method on each of these `split' data sets to obtain nearly independent estimates $\hat{A}_1,\dots,\hat{A}_K$,  which each have error at-most $ \frac{Cd^2\sigma^2K}{\zeta^2 T\lmin{G}}$ with probability at-least $\frac{2}{3}$. Using a standard high-dimensional median of means estimator (see \cite[Algorithm 3]{hsu2016loss}) for $\hat{A}_1,\dots,\hat{A}_K$, we obtain error bounds of the order $\frac{Cd^2\sigma^2K}{\zeta^2 T\lmin{G}}$ with probability at least $1-e^{-\Omega(K)}$.

We refer to Section~\ref{sec:ideas_behind_proofs} for a high-level exposition of the key ideas in the analysis and Section~\ref{sec:gen_newton_analysis} for the full proof of Theorems~\ref{thm:newton_ub_without_mixing} and~\ref{thm:newton_ub_heavy_tail}.

\section{Streaming Learning with SGD-RER}
\label{sec:sgd_rer_results}

\begin{figure}[t!]

  \centering
  \includegraphics[width = .7\linewidth]{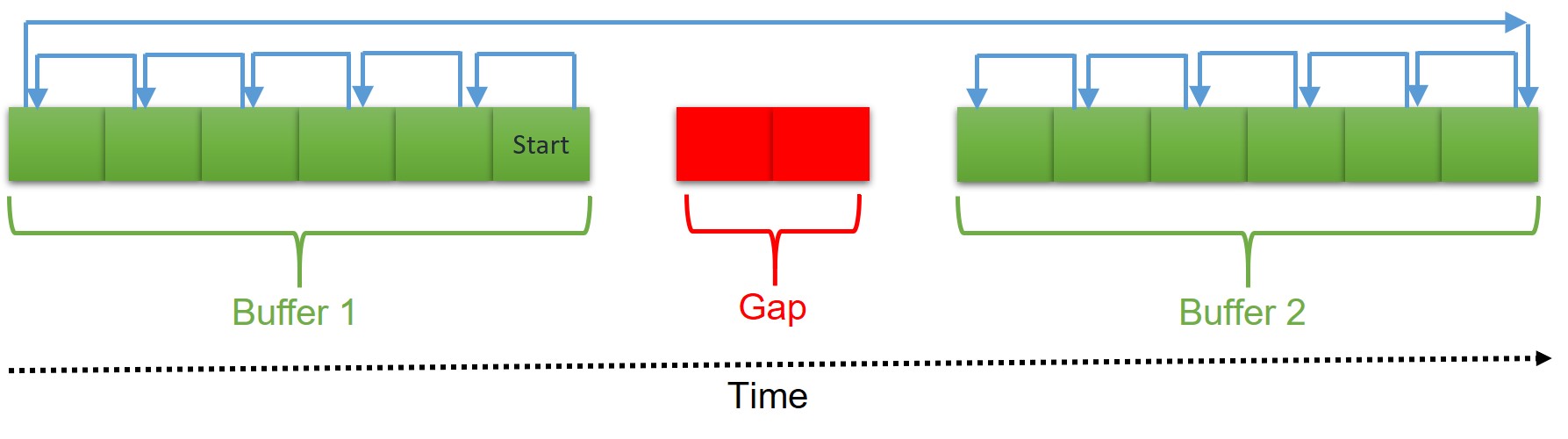}\vspace*{-10pt}
  \caption{ Data order in $\sgdber$, where each block represents a data point. Blue arrows indicate the data processing order. The gaps ensure approximate independence between successive buffers.}
  \label{fig:data_order_sgd_rer}
\end{figure}
In this section, we consider the one-pass, streaming setting, where the data points are presented in a streaming fashion. The goal is to continuously produce better estimates of $\A$ while also ensuring that the space and the time complexity of the algorithm is small. This disallows approaches that would just store all the observed points and then apply offline Algorithm~\ref{alg:modified_newton_method} to produce strong estimation error. Such one-pass streaming algorithms  are critical in a variety of settings like large-scale and online time-series analysis \cite{hamilton1994time,kuznetsov2016time}, TD learning in RL \cite{sutton1998introduction}, econometrics. 

To address this problem, we consider $\sgdber$ (Algorithm~\ref{alg:sgd_rer}) which was introduced in \cite{jain2021streaming} in the context of {\em linear system identification} (LSI). We apply the method for NLDS identification as well. $\sgdber$ uses SGD like updates, but the data is processed in a different order than it is received from the dynamical system. This algorithm is based on the observation made in \cite{jain2021streaming} that for LSI, when SGD is run on the least squares loss in the forward order, there are spurious correlations which prevent the algorithm's convergence to the optimum parameter $\A$. Surprisingly, considering the data in the reverse order \emph{exactly} unravels these correlations to resolve the problem.
\begin{algorithm}[t!]
	\DontPrintSemicolon
	\SetKwInOut{Input}{Input}
	\SetKwInOut{Output}{Output}
	\SetKwFunction{RN}{ReadNext}\SetKwFunction{LN}{LeaveNext}
	\Input{Streaming data $\{X_\tau\}$, horizon $T$, buffer size $B$, buffer gap $u$, bound $R$, tail start: $t_0 \leq N/2$, link function $\phi$, step size $\gamma$}
	\Output{Estimate $\hat A_{t_0,t}$, for all $t_0\leq t\leq N-1$; $N=T/(B+u)$}
	\Begin{
		Total buffer size: $S\leftarrow B+u$, Number of buffers: $N\leftarrow T/S$\;
		$A^0_0=0$ \textsf{/*Initialization*/}\;
		\For{$t\leftarrow 1$ \KwTo $N$}{
			Form buffer $\textsf{Buf}^{t-1}=\{X^{t-1}_0, \dots, X^{t-1}_{S-1}\}$, where, $X^{t-1}_{i}\leftarrow X_{(t-1)\cdot S+i}$\;
			If $\exists i,\ s.t.,\ \norm{X^{t-1}_i}^2>R$, then \textbf{return} $\hat A_{t_0,t}=0$\;
			
			\For{$i\leftarrow 0$ \KwTo $B-1$}{
				\nl $A^{t-1}_{i+1}\leftarrow A^{t-1}_i - 2 \gamma \left[\phi(A^{t-1}_i X^{t-1}_{S-i-1}) - X^{t-1}_{S-i}\right] X^{t-1,\top}_{S-i-1}$
			}
			$A^{t}_0=A^{t-1}_B$\;
			If $t\geq t_0+1$, then $\hat A_{t_0,t}\leftarrow \frac{1}{t-t_0}\sum_{\tau=t_0+1}^{t} A^{\tau-1}_B$
		}
	}
	
	\caption{$\sgdber$}
	\label{alg:sgd_rer}
\end{algorithm}
 Reverse order traversal of data, even though one pass, does not give a streaming algorithm. Hence, we divide the data into multiple buffers of size $B$ and leave of size $u$ between the buffers (See Figure~\ref{fig:data_order_sgd_rer}). The data \emph{within} each buffer is processed in the reverse order whereas the buffers themselves are processed in the order received. See Figure~\ref{fig:data_order_sgd_rer} for an illustration of the processing order. The gaps $u$ are set large enough so that the buffers behave approximately independently. Setting $B \geq  10u$ we note that this simple strategy improves the sample efficiency compared to naive data dropping since we use \emph{most} of the samples for estimating $\A$.  We now present the main result for streaming setting. 
\begin{theorem}[Streaming Algorithm]\label{thm:sgd_rer_ub}
Suppose Assumptions
~\ref{assump:3},~\ref{assump:8},~\ref{assump:7} and ~\ref{assump:6} hold and that the data points are stationary. Set $\alpha=100$, $R= \frac{16(\alpha+2)d C_{\eta}\sigma^2\log T }{1-\rho}$, $u\geq\frac{2\alpha\log{T}}{\log (\tfrac{1}{\rho})}$, $B \geq \left(\bar{C}_1 \frac{d}{(1-\rho)(1-\rho^2)}\log\left(\frac{d}{1-\rho}\right), 10 u\right) $  for a global constant $\bar{C}_1$ dependent only $C_{\eta}$ and $\alpha$. Let $N=T/(B+u)$ be the number of buffers. Finally, set step-size $\gamma = \frac{C}{T^{\nu}}$ where $\nu = 6.5/7$ and let $T$ be large enough such that $\gamma \leq \min\left(\frac{\zeta}{4BR(1+\zeta)},  \tfrac{1}{2R}\right)$. If $N/2>t_0> c_1\frac{\log T}{\zeta\gamma B\lambda_{\min}(G)}=\Theta(T^{\nu}\log T)$ for some large enough constant $c_1>0$, then output $\hat{A}_{t_0,N}$ of Algorithm~\ref{alg:sgd_rer} satisfies: 
\begin{\Ieee}{LLL}
\Ex{\|\hat{A}_{t_0,N} -\A\|_{\mathsf{F}}^2} \leq C \frac{d^2\sigma^2 \log{T}}{T\lmin{G}\zeta^2} + \text{ Lower Order Terms }\Ieeen
\end{\Ieee}
where $C$ is a constant dependent on $C_{\eta}$, $\alpha$.
\end{theorem}
\begin{remark}
The lower order terms are of the order $\poly(R,B,\beta,1/\zeta,1/\lambda_{\min},\norm{\phi''})\gamma^{7/2}T^2 + \gamma^2 R \sigma^2 d \frac{1}{T^{\alpha/2-2}} + \norm{A_0-\A}^2_{\mathsf{F}}\left[\frac{e^{-c_2\zeta\gamma B\lambda_{\min}t_0}}{T\zeta\gamma \lambda_{\min}}\right] $. We refer to the proof in Section~\ref{sec:sgd_rer_proof} for details.
\end{remark}
\begin{remark}
Although the bound in Theorem~\ref{thm:sgd_rer_ub} is given for the algorithmic iterate at the end of the horizon, the proof shows that in fact we can bound the error of the iterates at the end of each buffer after $(1+c)t_0$ i.e. if $t\geq (1+c)t_0$ for some $c>0$ then we obtain 
$$\Ex{\|\hat{A}_{t_0,t} -\A\|_{\mathsf{F}}^2} \leq C \frac{d^2\sigma^2 \log{T}}{(tB) \lmin{G}\zeta^2} + \text{ Lower Order Terms } $$
\end{remark}
Note that the estimation error above matches the error by offline method up to log factors (see Theorem~\ref{thm:newton_ub_heavy_tail}). 
Furthermore, while the method requires NLDS to be mixing, i.e., $\rho<1$, but the leading term in error rate does not have an explicit dependence on it. Moreover, the space complexity of the method is only $B \cdot d$ which scales as $d^2/((1-\rho)(1-\rho^2))$, i.e., it is $1/((1-\rho)(1-\rho^2))\sim \tmix^2$ factor worse than the obvious lower bound of O($d^2$) to store $A$. We leave further investigation into   space complexity optimization or tightening the lower bound for future work. Also, note that $u\leq B/10$, so $\sgdber$ wastes only about 10\% of the samples. Finally, the algorithm requires a reasonable upper bound on $\rho$ to set up various hyperparameters like $R, u, B$. However, it is not clear how to estimate such an upper bound only using the data, and seems like an interesting open question. 

See Section~\ref{sec:ideas_behind_proofs} for an explanation of the elements involved in the analysis of the algorithm and to Section~\ref{subsec:proof_strategy} for a detailed overview of the proof.
\section{Exponential Lower Bounds for Non-Expansive Link Functions}
The previous results showed that we can efficiently recover the matrix $\A$ given that the link function is uniformly expansive. We now consider non-expansive functions and show that parameter recovery is hard in this case. In particular, we show that even for the case of $\phi = \relu$, the noise being $\mathcal{N}(0,I)$, and $\|\A\| \leq \tfrac{1}{2}$, the error has an information theoretic lower bound which is exponential in the dimension. We note that this is consistent with Theorem 3 in \cite{foster2020learning} which too has an exponential dependence on the dimension (since the matrix $K \succeq I$). 

Before stating the results, we introduce some notation. Consider any algorithm $\mathcal{A}$, with accepts input $(X_0,\dots,X_T) $ and outputs an estimate $\hat{A} \in \mathbb{R}^{d\times d}$. For simplicity of calculation, we will assume that $X_0 = 0$ and $X_{t+1} = \relu(\A X_t) + \eta_t$. Since the mixing time is $O(1)$, similar results should hold for stationary sequences. We define the loss $\loss(\mathcal{A},T,\A) = \mathbb{E}\|\hat{A}-\A\|^2_F$, where the expectation is over the randomness in the data and the algorithm. By $\Theta(\tfrac{1}{2})$, we denote all the the elements of $B \in \mathbb{R}^{d\times d}$ such that $\|B\| \leq \tfrac{1}{2}$. The minimax loss is defined as: $$\loss(\Theta(\tfrac{1}{2}),T) := \inf_{\mathcal{A}}\sup_{\A \in \Theta(\tfrac{1}{2})} \loss(\mathcal{A},T,\A)\,.$$

\begin{theorem}[ReLU Lower Bound]\label{thm:relu_exp_lb}
For universal constants $c_0,c_1 >0$, we have:
$$\loss(\Theta(\tfrac{1}{2}),T) \geq c_0\min\left(1,\frac{\exp(c_1d)}{T}\right) \,.$$

\end{theorem}
We prove the theorem above using the two point method. We find a family of $\A$ in $\Theta(\tfrac{1}{2},T)$ such that $\langle X_{t},e_d\rangle = \langle \eta_{t-1},e_d\rangle$ with probability at-least $1-\exp(-\Omega(d))$. Therefore, with a large probability, we only observe noise in the last co-ordinate and hence do not obtain any information regarding the last row of $\A$ (i.e, $a_d^{*}$). We refer to Section~\ref{sec:relu_lb} for a full proof.


\section{Proof Sketch}
\label{sec:ideas_behind_proofs}
\paragraph{Quasi Newton Method.} 
Let $a^{*}_i$ be the $i$-th row of $\A$ and $a_i(l)$ be the $i$-th row of $A_l$ both in column vector form. The proofs of Theorems~\ref{thm:newton_ub_without_mixing} and~\ref{thm:newton_ub_heavy_tail} follow once we consider the lyapunov function $\Delta_{l,i} = \|\hat{G}^{1/2}(a_i(l)-a_i^{*})\|$ and show that 
\begin{equation}\label{eq:newton_lyapunov}
\Delta_{l+1,i} \leq (1-2\gamma \zeta)\Delta_l + \gamma \|\hat{G}^{-1/2}\hat{N}_i\|
\end{equation}

Where $\hat{N}_i := \frac{1}{T}\sum_{t=0}^{T-1}\langle e_i,\eta_t\rangle X_t$. In the case of Theorem~\ref{thm:newton_ub_without_mixing}, we use the sub-Gaussianity of the noise sequence and a martingale argument to obtain a high probability upper bound on $\sum_{i=1}^{d}\|\hat{G}^{-1/2}\hat{N}_i\|^2$ (see Lemma~\ref{lem:normalized_noise}). In the heavy tailed case considered in Theorem~\ref{thm:newton_ub_heavy_tail}, we use mixing time arguments along with Payley-Zygmund inequality to show the high probability lower isometry $\hat{G} \succeq c_0 G$ for some universal constant $c_0$ (see Lemma~\ref{lem:well_conditioned_grammian}). Using this lower isometry, we can replace $\hat{G}$ in Equation~\eqref{eq:newton_lyapunov} with $G$. The upper bounds follow once we note that $\mathbb{E}\hat{N}_i^{\intercal}G^{-1}\hat{N}_i = \frac{\sigma^2d}{T}$.\vspace*{-3pt}
\paragraph{SGD-RER.} 
Due to the observations made in Section~\ref{sec:sgd_rer_results}, we can split the analysis into the following parts, which are explained in detail below. 
\begin{enumerate}[leftmargin=*]
\item Analyze the reverse order SGD \emph{within} the buffers.
\item Treat successive buffer as independent samples.
\item Give a bias-variance decomposition similar to the case of linear regression.
\item Use algorithmic stability to control 'spurious' coupling introduced by non-linearity in the bias-variance decomposition.\vspace*{-3pt}
\end{enumerate}
\paragraph{Coupled Process.} 
We deal with the dependence \emph{between} buffers using a fictitious coupled process, constructed just for the sake of analysis (see Definition~\ref{def:coupled_proc}). Leveraging the gap $u$, this process $(\tilde{X}_{\tau})$ is constructed such that $\tilde{X}_{\tau} \approx X_{\tau}$ with high probability and the `coupled buffers' containing data $\tilde{X}$ instead of $X$ are \emph{exactly} independent. Since $\tilde{X}_{\tau} \approx X_{\tau}$, the output of $\sgdber$ run with the fictitious coupled process should be close output of $\sgdber$ run with the actual data points. We then use the strategy outlined above to analyze $\sgdber$ with the coupled process. In the analysis given for $\sgdber$, all the quantities with $\tilde{\cdot}$ involve the coupled process $\tilde{X}$ instead of the real process $X$.\vspace*{-3pt}
\paragraph{Non-Linear Bias Variance Decomposition.}
We use the mean value theorem to linearize the non-linear problem. This works effectively when the step size $\gamma$ is a vanishing function of the horizon $T$. Observe that the a single SGD/ $\sgdber$ step for a single row can be written as:
\begin{align}
a^\prime_i - a_i^{*} &= a_i - a_i^{*} - 2\gamma(\phi(\langle a_i,X_{\tau}\rangle) - \phi(\langle a^{*}_i,X_{\tau}\rangle) )X_{\tau} + 2\gamma \langle\eta_{\tau},e_i\rangle X_{\tau} \nonumber \\
& = \left(I-2\gamma \phi^{\prime}(\beta_{\tau})X_{\tau}X_{\tau}^{\intercal} \right)(a_i - a_i^{*}) + 2\gamma \langle\eta_{\tau},e_i\rangle X_{\tau} \label{eq:non_linear_bias_var}
\end{align}
In the second step, we have used the mean value theorem. Equation~\eqref{eq:non_linear_bias_var} can be interpreted as follows: the matrix $\left(I-2\gamma \phi^{\prime}(\beta_{\tau})X_{\tau}X_{\tau}^{\intercal} \right)$ 'contracts' the distance between $a_i$ and $a_i^{*}$ whereas the noise $2\gamma \langle\eta_{\tau},e_i\rangle X_{\tau}$ is due to the inherent uncertainty. This gives us a bias-variance decomposition similar to the case of SGD with linear regression. We refer to Section~\ref{subsec:bias_variance_decomp} for details on unrolling the recursion in Equation~\eqref{eq:non_linear_bias_var} to obtain the exact bias-variance decomposition. 

{\bf Algorithmic Stability:} Unfortunately, non-linearities result in a `coupling' between the contraction matrices through the iterates via the first derivative $\phi^{\prime}(\beta_\tau)$ due to reverse order traversal. This is an important issue since unrolling the recursion in ~\eqref{eq:non_linear_bias_var}, we encounter terms such as $\langle \eta_{\tau},e_i\rangle(I - 2\gamma\phi^{\prime}(\beta_{\tau -1})X_{\tau-1}X_{\tau-1}^{\intercal})X_{\tau}$, which have zero mean in the linear case. However, in the non-linear case, $\beta_{\tau-1}$ depends on $\eta_{\tau}$ due to reverse order traversal. We show that such dependencies are `weak' using the idea of algorithmic stability (\cite{bousquet2002stability,hardt2016train}). In particular, we establish that the output of the algorithm is not affected too much if we re-sample the \emph{entire} data trajectory by independently re-sampling a single noise co-ordinate ($\eta_{\tau}$ becomes $\eta^{\prime}_{\tau}$ and $\beta_{\tau -1}$ becomes $\beta^{\prime}_{\tau -1}$) when the step size $\gamma$ is small enough (in other words, the output is stable under small perturbations). Via second derivative arguments, we show that $\beta_{\tau - 1} \approx \beta^{\prime}_{\tau-1}$.

Now observe that resampling noise $\eta_{\tau}$ does not affect the past value of data i.e, $X_{\tau},X_{\tau-1}$ and is independent of $\beta_{\tau -1}^{\prime}$ by construction. Therefore
 $$ 0 = \mathbb{E}\langle \eta_{\tau},e_i\rangle\left(I - 2\gamma\phi^{\prime}(\beta^{\prime}_{\tau -1})X_{\tau-1}X^{\intercal}_{\tau-1}\right)X_{\tau} \approx \mathbb{E}\langle \eta_{\tau},e_i\rangle\left(I - 2\gamma\phi^{\prime}(\beta_{\tau -1})X_{\tau-1}X^{\intercal}_{\tau-1}\right) X_{\tau} $$
Such a resampling procedure is also explored in~\cite{nagaraj2019sgd} for the analysis of SGD with random reshuffling. 

We put together all the ingredients above in order to prove the error bounds given in Theorem~\ref{thm:sgd_rer_ub}.

\section{Experiments}
\label{sec:expts}
In this section, we compare performance of our methods $\sgdber$  and Quasi Newton method on synthetic data against the performance of standard baselines $\sgd$ (called `Forward SGD' here), GLMtron, along with the $\sgder$ method that applies standard experience replay technique i.e, the points from a buffer are sampled {\em randomly} instead of the reverse order. Since GLMtron and Quasi Newton Method are offline and $\sgdber$, $\sgd$ and $\sgder$ are streaming, we compare the algorithms by plotting parameter error measured by the Frobenius norm with respect to the compute time. We also compare error vs. number of iterations for the streaming algorithms. We show the results of additional experiments by considering various buffer sizes and heavy tailed noise in Section~\ref{sec:add_exp}. 
\begin{figure}[t!]
	\begin{subfigure}{ 0.5\textwidth}
	\centering
	\includegraphics[width=\linewidth, height = 0.6\linewidth]{{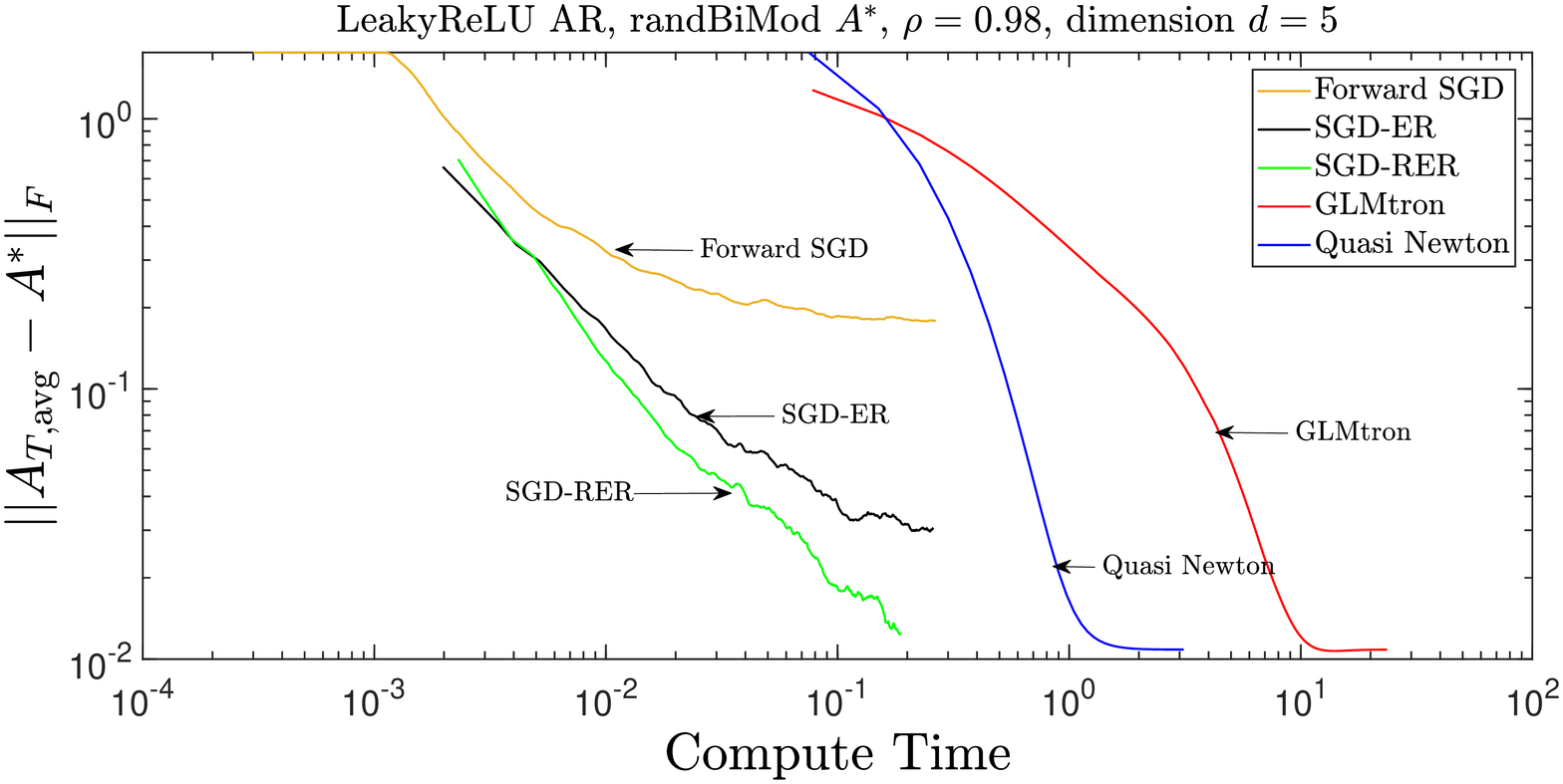}}
	\caption{Error vs. Computation}
	\label{fig:leaky_relu_1}
	\end{subfigure}
	\begin{subfigure}{.5\textwidth}
		\centering
		\includegraphics[width=\linewidth, height = 0.6\linewidth]{{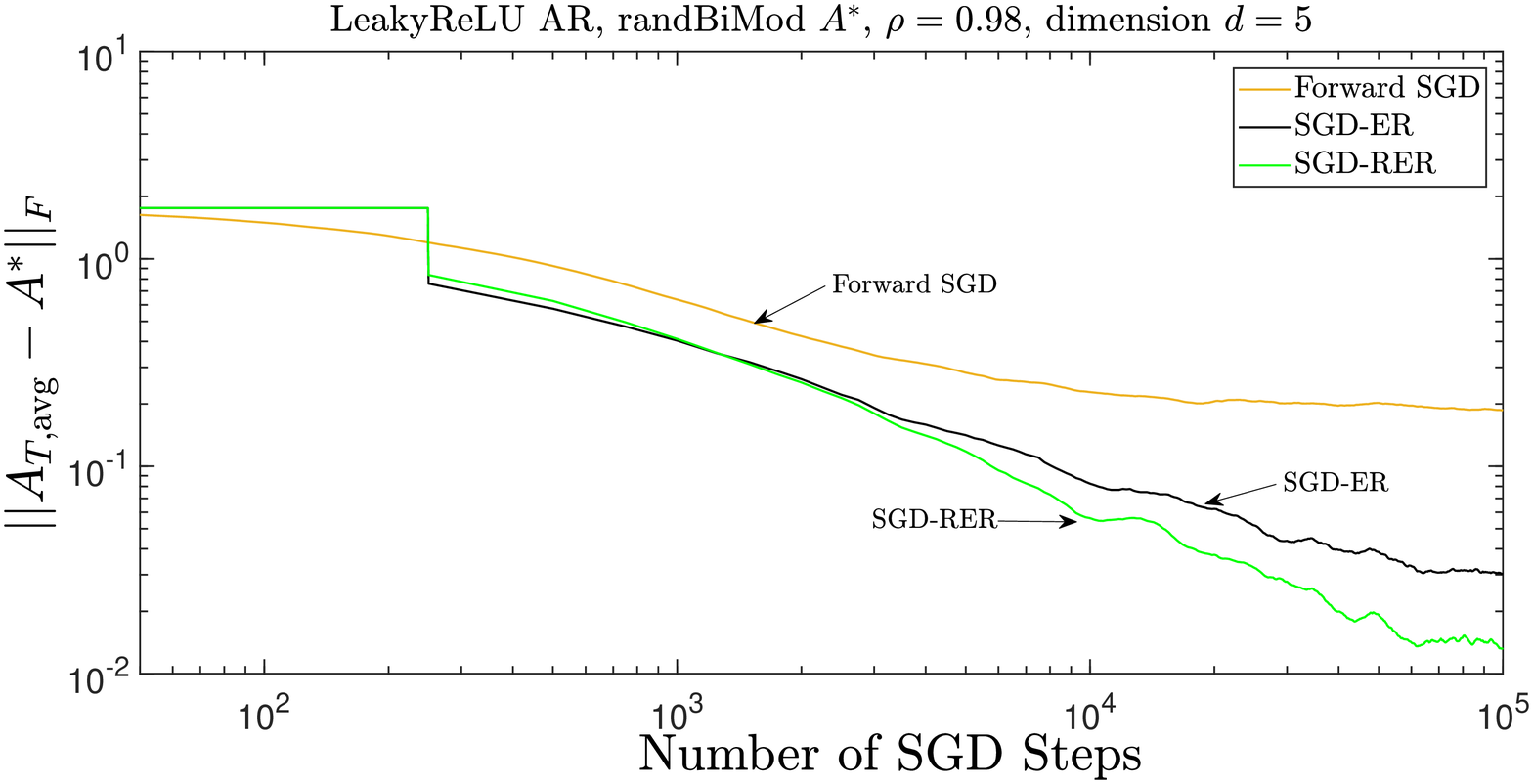}}  
		\caption{Error vs. SGD updates}
		\label{fig:leaky_relu_2}
	\end{subfigure}

	\caption{Performance of various algorithms for the case of $ \phi = \mathsf{LeakyReLU}$}
	\label{fig:simulation_comparison}
\end{figure}

\textbf{Synthetic data}: We sample data from $\glvar(\A,\mu,\phi)$ where $\mu\distas{}\mathcal{N}(0,\sigma^2 I)$ and $\A\in \mathbb{R}^{d\times d}$ is generated from the "RandBiMod" distribution. That is, $\A=U\Lambda U^{\top}$ with random orthogonal $U$, and $\Lambda$ is diagonal with $\lceil d/2\rceil$ entries on diagonal being $\rho$ and the remaining diagonal entries are set to $\rho/3$. $\phi$ is the leaky $\relu$ function given by $\phi(x) = 0.5 x \mathbbm{1}(x <0)+ x \mathbbm{1}(x \geq 0)$. We set $d=5$, $\rho=0.98$ and $\sigma^2=1$. We set a horizon of $T = 10^5$.

\textbf{Algorithm Parameters} We set $B = 240$ and $u = 10$ for the buffer size and gap size respectively for both $\sgdber$ and $\sgder$ and use full averaging (i.e, $\theta = 0$ in Algorithm~\ref{alg:sgd_rer} ). We set the step size $\gamma = \frac{5\log T}{T}$ for $\sgd$, $\sgdber$, and $\sgder$ and  $\gamma_{\mathsf{newton}} = 0.2$ and $\gamma_{\mathsf{GLMtron}} = 0.017 $. 

From Figure~\ref{fig:simulation_comparison} observe that $\sgder$ and $\sgd$ obtain sub-optimal results compared $\sgdber$, Quasi Newton Method and GLMtron. After a single pass, the performance of $\sgdber$ almost matches that of the offline algorithms. The step sizes for GLMtron have to be chosen to be small in-order to ensure that the algorithm does not diverge as noted in Section~\ref{sec:offline_learning}, which slows down its convergence time compared to the Quasi Newton method. We set the step size to be as large as possible without obtaining divergence to infinity.


\section{Conclusion}
In this work, we studied the problem of learning non-linear dynamical systems of the form~\eqref{eq:master_equation} from a single trajectory and analyzed offline and online algorithms to obtain near-optimal error guarantees. In particular we showed that mixing time based arguments are not necessary for learning certain classes of non-linear dynamical systems. Even though we show that one cannot hope for efficient parameter recovery with non-expansive link functions like $\mathsf{ReLU}$, we do not deal with the problem of minimizing the `prediction' error - where we output a good predictor based on samples, without parameter recovery. We believe that this problem would require significantly different set of techniques than the ones established in this work and hope to investigate this in future work. Presently our work only deals with specific kinds of Markovian time evolution. It would also be interesting to understand in general, the kind of structures which allow for learning without mixing based arguments. Another related direction is to design simple and efficient algorithms like $\sgdber$ for learning with various models of dependent data by unraveling the dependency structure present in the data. 
\begin{ack}
D.N. was supported in part by NSF grant DMS-2022448.\\
S.S.K was supported in part by Teaching Assistantship (TA) from EECS, MIT.
\end{ack}
\clearpage
\bibliographystyle{unsrt}  

\bibliography{refs}

\begin{thebibliography}{10}

\bibitem{vidyasagar2002nonlinear}
Mathukumalli Vidyasagar.
\newblock {\em Nonlinear systems analysis}.
\newblock SIAM, 2002.

\bibitem{chen1989representations}
Sheng Chen and Steve~A Billings.
\newblock {Representations of non-linear systems: the NARMAX model}.
\newblock {\em International journal of control}, 49(3):1013--1032, 1989.

\bibitem{elman1990finding}
Jeffrey~L Elman.
\newblock Finding structure in time.
\newblock {\em Cognitive science}, 14(2):179--211, 1990.

\bibitem{jordan1997serial}
Michael~I Jordan.
\newblock {Serial order: A parallel distributed processing approach}.
\newblock In {\em Advances in psychology}, volume 121, pages 471--495.
  Elsevier, 1997.

\bibitem{ljung1999system}
Lennart Ljung.
\newblock System identification.
\newblock {\em Wiley encyclopedia of electrical and electronics engineering},
  pages 1--19, 1999.

\bibitem{aastrom1971system}
{{\AA}str{\"o}m, Karl Johan and Eykhoff, Peter}.
\newblock System identification—a survey.
\newblock {\em Automatica}, 7(2):123--162, 1971.

\bibitem{campi2002finite}
Marco~C Campi and Erik Weyer.
\newblock Finite sample properties of system identification methods.
\newblock {\em IEEE Transactions on Automatic Control}, 47(8):1329--1334, 2002.

\bibitem{vidyasagar2006learning}
Mathukumalli Vidyasagar and Rajeeva~L Karandikar.
\newblock A learning theory approach to system identification and stochastic
  adaptive control.
\newblock In {\em Probabilistic and randomized methods for design under
  uncertainty}, pages 265--302. Springer, 2006.

\bibitem{hall2016inference}
Eric~C Hall, Garvesh Raskutti, and Rebecca Willett.
\newblock Inference of high-dimensional autoregressive generalized linear
  models.
\newblock {\em arXiv preprint arXiv:1605.02693}, 2016.

\bibitem{wood2014behavioral}
John Wood.
\newblock {\em Behavioral modeling and linearization of RF power amplifiers}.
\newblock Artech house, 2014.

\bibitem{simchowitz2018learning}
Max Simchowitz, Horia Mania, Stephen Tu, Michael~I Jordan, and Benjamin Recht.
\newblock {Learning without mixing: Towards a sharp analysis of linear system
  identification}.
\newblock {\em arXiv preprint arXiv:1802.08334}, 2018.

\bibitem{sarkar2019near}
Tuhin Sarkar and Alexander Rakhlin.
\newblock Near optimal finite time identification of arbitrary linear dynamical
  systems.
\newblock In {\em International Conference on Machine Learning}, pages
  5610--5618. PMLR, 2019.

\bibitem{shalev2009stochastic}
Shai Shalev-Shwartz, Ohad Shamir, Nathan Srebro, and Karthik Sridharan.
\newblock {Stochastic Convex Optimization.}
\newblock In {\em COLT}, 2009.

\bibitem{paulin2015concentration}
Daniel Paulin et~al.
\newblock {Concentration inequalities for Markov chains by Marton couplings and
  spectral methods}.
\newblock {\em Electronic Journal of Probability}, 20, 2015.

\bibitem{rotinov2019reverse}
Egor Rotinov.
\newblock Reverse experience replay.
\newblock {\em arXiv preprint arXiv:1910.08780}, 2019.

\bibitem{ambrose2016reverse}
R~Ellen Ambrose, Brad~E Pfeiffer, and David~J Foster.
\newblock Reverse replay of hippocampal place cells is uniquely modulated by
  changing reward.
\newblock {\em Neuron}, 91(5):1124--1136, 2016.

\bibitem{haga2018recurrent}
Tatsuya Haga and Tomoki Fukai.
\newblock Recurrent network model for learning goal-directed sequences through
  reverse replay.
\newblock {\em Elife}, 7:e34171, 2018.

\bibitem{whelan2021robotic}
Matthew~T Whelan, Tony~J Prescott, and Eleni Vasilaki.
\newblock A robotic model of hippocampal reverse replay for reinforcement
  learning.
\newblock {\em arXiv preprint arXiv:2102.11914}, 2021.

\bibitem{jain2021streaming}
Prateek Jain, Suhas~S Kowshik, Dheeraj Nagaraj, and Praneeth Netrapalli.
\newblock {Streaming Linear System Identification with Reverse Experience
  Replay}.
\newblock {\em arXiv preprint arXiv:2103.05896}, 2021.

\bibitem{abbasi2011online}
{Abbasi-Yadkori, Yasin and P{\'a}l, D{\'a}vid and Szepesv{\'a}ri, Csaba}.
\newblock Online least squares estimation with self-normalized processes: An
  application to bandit problems.
\newblock {\em arXiv preprint arXiv:1102.2670}, 2011.

\bibitem{pena2008self}
Victor~H Pe{\~n}a, Tze~Leung Lai, and Qi-Man Shao.
\newblock {\em {Self-normalized processes: Limit theory and Statistical
  Applications}}.
\newblock Springer Science \& Business Media, 2008.

\bibitem{diakonikolas2020approximation}
Ilias Diakonikolas, Surbhi Goel, Sushrut Karmalkar, Adam~R Klivans, and Mahdi
  Soltanolkotabi.
\newblock Approximation schemes for relu regression.
\newblock In {\em Conference on Learning Theory}, pages 1452--1485. PMLR, 2020.

\bibitem{chen1990non}
Sheng Chen, Stephen~A Billings, and PM~Grant.
\newblock Non-linear system identification using neural networks.
\newblock {\em International journal of control}, 51(6):1191--1214, 1990.

\bibitem{allen2018convergence}
Zeyuan Allen-Zhu, Yuanzhi Li, and Zhao Song.
\newblock On the convergence rate of training recurrent neural networks.
\newblock {\em arXiv preprint arXiv:1810.12065}, 2018.

\bibitem{bahmani2019convex}
Sohail Bahmani and Justin Romberg.
\newblock Convex programming for estimation in nonlinear recurrent models.
\newblock {\em arXiv preprint arXiv:1908.09915}, 2019.

\bibitem{oymak2019stochastic}
Samet Oymak.
\newblock Stochastic gradient descent learns state equations with nonlinear
  activations.
\newblock In {\em Conference on Learning Theory}, pages 2551--2579. PMLR, 2019.

\bibitem{miller2018stable}
John Miller and Moritz Hardt.
\newblock Stable recurrent models.
\newblock {\em arXiv preprint arXiv:1805.10369}, 2018.

\bibitem{mania2020active}
Horia Mania, Michael~I Jordan, and Benjamin Recht.
\newblock Active learning for nonlinear system identification with guarantees.
\newblock {\em arXiv preprint arXiv:2006.10277}, 2020.

\bibitem{sarker2020parameter}
Arnab Sarker, Joseph~E Gaudio, and Anuradha~M Annaswamy.
\newblock {Parameter Estimation Bounds Based on the Theory of Spectral Lines}.
\newblock {\em arXiv preprint arXiv:2006.12687}, 2020.

\bibitem{mao2020finite}
Yanbing Mao, Naira Hovakimyan, Petros Voulgaris, and Lui Sha.
\newblock {Finite-Time Model Inference From A Single Noisy Trajectory}.
\newblock {\em arXiv preprint arXiv:2010.06616}, 2020.

\bibitem{sattar2020non}
Yahya Sattar and Samet Oymak.
\newblock Non-asymptotic and accurate learning of nonlinear dynamical systems.
\newblock {\em arXiv preprint arXiv:2002.08538}, 2020.

\bibitem{foster2020learning}
Dylan Foster, Tuhin Sarkar, and Alexander Rakhlin.
\newblock Learning nonlinear dynamical systems from a single trajectory.
\newblock In {\em Learning for Dynamics and Control}, pages 851--861. PMLR,
  2020.

\bibitem{kakade2011efficient}
Sham Kakade, Adam~Tauman Kalai, Varun Kanade, and Ohad Shamir.
\newblock Efficient learning of generalized linear and single index models with
  isotonic regression.
\newblock {\em arXiv preprint arXiv:1104.2018}, 2011.

\bibitem{gao2021improved}
Yue Gao and Garvesh Raskutti.
\newblock Improved prediction and network estimation using the monotone single
  index multi-variate autoregressive model.
\newblock {\em arXiv preprint arXiv:2106.14630}, 2021.

\bibitem{simchowitz2019learning}
Max Simchowitz, Ross Boczar, and Benjamin Recht.
\newblock Learning linear dynamical systems with semi-parametric least squares.
\newblock In {\em Conference on Learning Theory}, pages 2714--2802. PMLR, 2019.

\bibitem{sarkar2019finite}
Tuhin Sarkar, Alexander Rakhlin, and Munther~A Dahleh.
\newblock Finite-time system identification for partially observed lti systems
  of unknown order.
\newblock {\em arXiv preprint arXiv:1902.01848}, 2019.

\bibitem{kalai2009isotron}
Adam~Tauman Kalai and Ravi Sastry.
\newblock {The Isotron Algorithm: High-Dimensional Isotonic Regression.}
\newblock In {\em COLT}. Citeseer, 2009.

\bibitem{boffi2020reflectron}
Nicholas~M Boffi, Stephen Tu, and Jean-Jacques~E Slotine.
\newblock {The Reflectron: Exploiting geometry for learning generalized linear
  models}.
\newblock {\em arXiv preprint arXiv:2006.08575}, 2020.

\bibitem{nagaraj2020least}
Dheeraj Nagaraj, Xian Wu, Guy Bresler, Prateek Jain, and Praneeth Netrapalli.
\newblock {Least Squares Regression with Markovian Data: Fundamental Limits and
  Algorithms}.
\newblock {\em Advances in Neural Information Processing Systems}, 33, 2020.

\bibitem{lin1992self}
Long-Ji Lin.
\newblock Self-improving reactive agents based on reinforcement learning,
  planning and teaching.
\newblock {\em Machine learning}, 8(3-4):293--321, 1992.

\bibitem{hsu2016loss}
Daniel Hsu and Sivan Sabato.
\newblock Loss minimization and parameter estimation with heavy tails.
\newblock {\em The Journal of Machine Learning Research}, 17(1):543--582, 2016.

\bibitem{hamilton1994time}
James~Douglas Hamilton.
\newblock {\em Time series analysis}.
\newblock Princeton university press, 1994.

\bibitem{kuznetsov2016time}
Vitaly Kuznetsov and Mehryar Mohri.
\newblock Time series prediction and online learning.
\newblock In {\em Conference on Learning Theory}, pages 1190--1213. PMLR, 2016.

\bibitem{sutton1998introduction}
Richard~S Sutton, Andrew~G Barto, et~al.
\newblock {\em Introduction to reinforcement learning}, volume 135.
\newblock MIT press Cambridge, 1998.

\bibitem{bousquet2002stability}
{Bousquet, Olivier and Elisseeff, Andr{\'e}}.
\newblock Stability and generalization.
\newblock {\em The Journal of Machine Learning Research}, 2:499--526, 2002.

\bibitem{hardt2016train}
Moritz Hardt, Ben Recht, and Yoram Singer.
\newblock {Train faster, generalize better: Stability of stochastic gradient
  descent}.
\newblock In {\em International Conference on Machine Learning}, pages
  1225--1234. PMLR, 2016.

\bibitem{nagaraj2019sgd}
Dheeraj Nagaraj, Prateek Jain, and Praneeth Netrapalli.
\newblock {SGD without replacement: Sharper rates for general smooth convex
  functions}.
\newblock In {\em International Conference on Machine Learning}, pages
  4703--4711. PMLR, 2019.

\bibitem{boucheron2013concentration}
{Boucheron, St{\'e}phane and Lugosi, G{\'a}bor and Massart, Pascal}.
\newblock {\em {Concentration inequalities: A nonasymptotic theory of
  independence}}.
\newblock Oxford university press, 2013.

\bibitem{dieuleveut2020bridging}
Aymeric Dieuleveut, Alain Durmus, Francis Bach, et~al.
\newblock Bridging the gap between constant step size stochastic gradient
  descent and markov chains.
\newblock {\em Annals of Statistics}, 48(3):1348--1382, 2020.

\bibitem{vershynin2019high}
Roman Vershynin.
\newblock High-dimensional probability, 2019.

\bibitem{vershynin2010introduction}
Roman Vershynin.
\newblock Introduction to the non-asymptotic analysis of random matrices.
\newblock {\em arXiv preprint arXiv:1011.3027}, 2010.

\end{thebibliography}

\clearpage
\appendix


\section{Additional Experiments}
\label{sec:add_exp}
Based on the feedback from reviewers, we perform the following additional experiments which explore the robustness of the choice of buffer size in $\sgdber$, choice of step sizes for GLMtron and the behavior of the said algorithms with heavy tailed noise with a similar setup as in Section~\ref{sec:expts}. 

We first do an experimental study about the robustness of $\sgdber$ to the choice of buffer size in Figure~\ref{fig:buf_size_variation}. Notice that the performance remains the same for a large range of buffer sizes ( 100 from  to 2000). However the performance degrades when the buffer size is too large ( $\approx$ 10000). We believe this is the case since the number of buffers decreases as the buffer size increases and the output is averaged over too few number of iterates (In the case of B = 10000, the final output is just an average of 10 iterates). 

Next, we consider the the range of step sizes which allow GLMtron to converge in Figure~\ref{fig:glm_tron_steps} in order to supplement the discussion in Sections~\ref{sec:expts} and~\ref{sec:offline_learning} regarding GLMtron requiring smaller step-sizes. 
In smooth convex optimization, it is typically the case that the iterates diverge to infinity if the step size is chosen to be too large. Theoretically, this largest step-size is $\frac{2}{L}$ where $L$  is the largest eigenvalue of the Hessian. In the case of GLMtron, it was experimentally observed that if the step size was chosen to be about 1.5 times the step size reported in Section~\ref{sec:expts}, the iterates diverged. Quasi Newton method essentially normalizes the gradient with the inverse of the Hessian (or rather an approximation of the Hessian) in order to let it converge faster with large step sizes.

\begin{figure}[t!]
	\begin{subfigure}{.5\textwidth}
	\centering
	\includegraphics[width=\linewidth]{{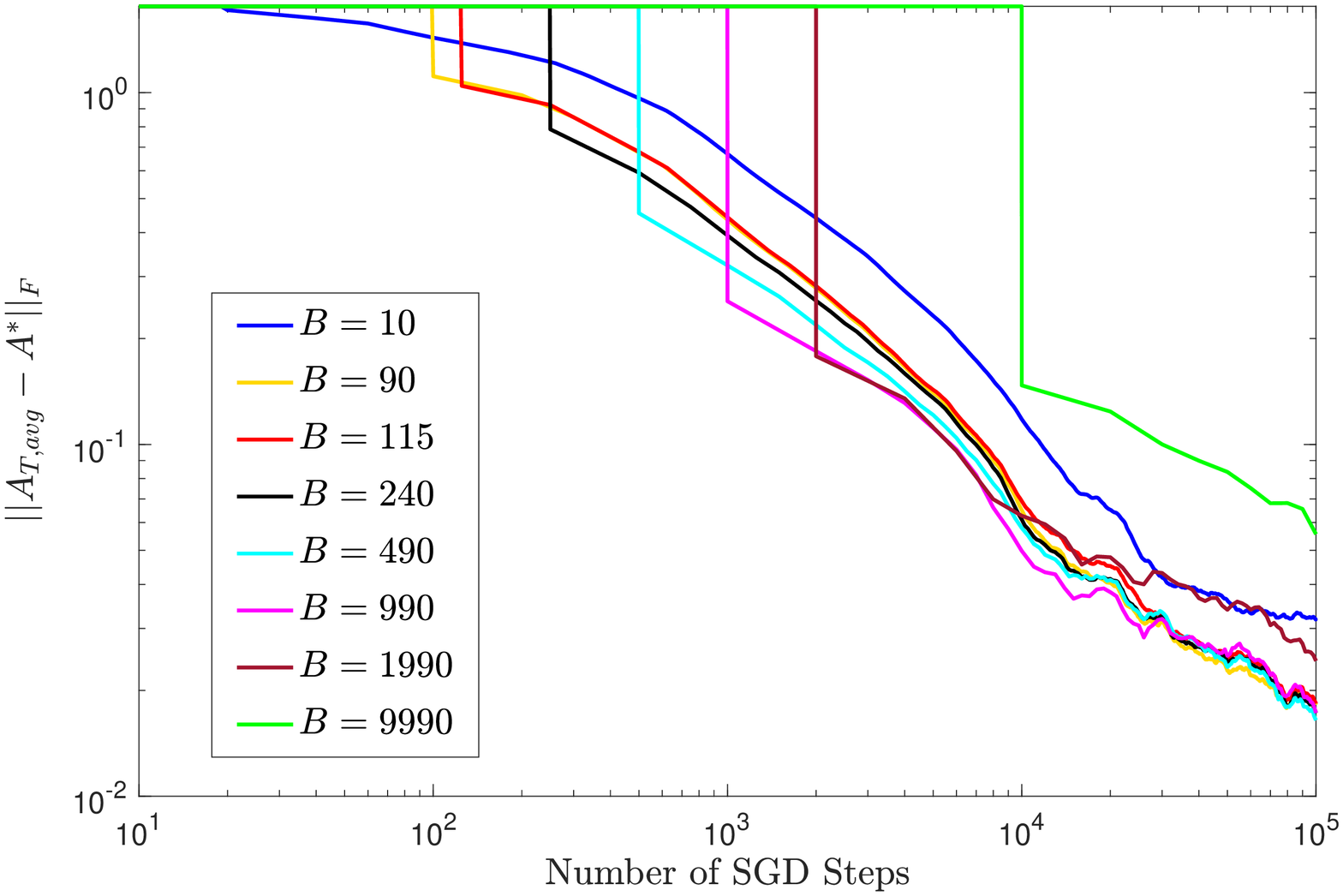}}
	\caption{Varying the Buffer Sizes}
	\label{fig:buf_size_variation}
	\end{subfigure}
	\begin{subfigure}{.5\textwidth}
		\centering
		\includegraphics[width=\linewidth]{{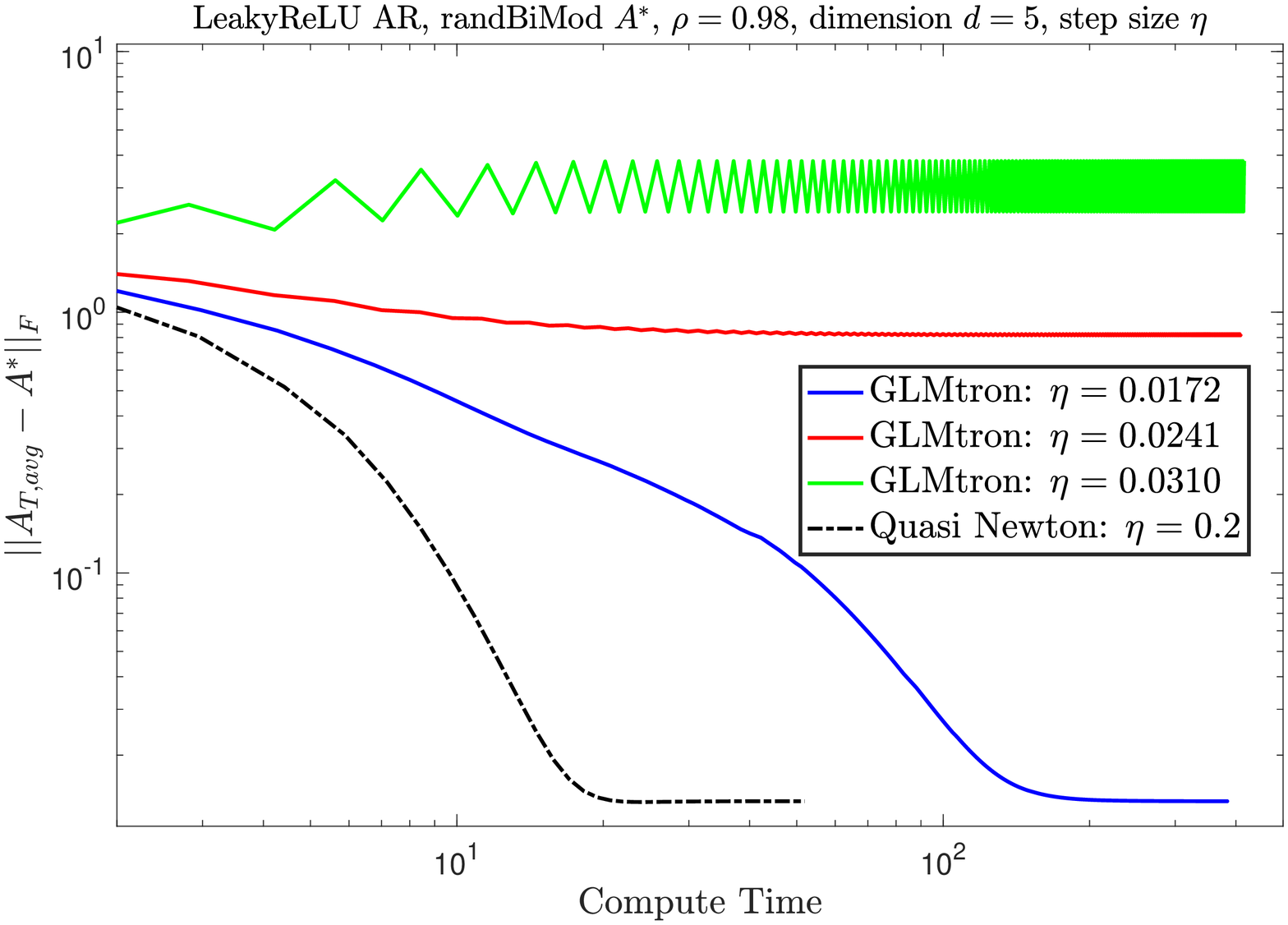}}  
		\caption{GLMtron step sizes}
		\label{fig:glm_tron_steps}
	\end{subfigure}
\vspace*{-5pt}
	\caption{Varying Parameters in the Algorithms}
	\label{fig:hyper_parameter_comparison}
\end{figure}


\begin{figure}[t!]
	\begin{subfigure}{.5\textwidth}
	\centering
	\includegraphics[width=\linewidth]{{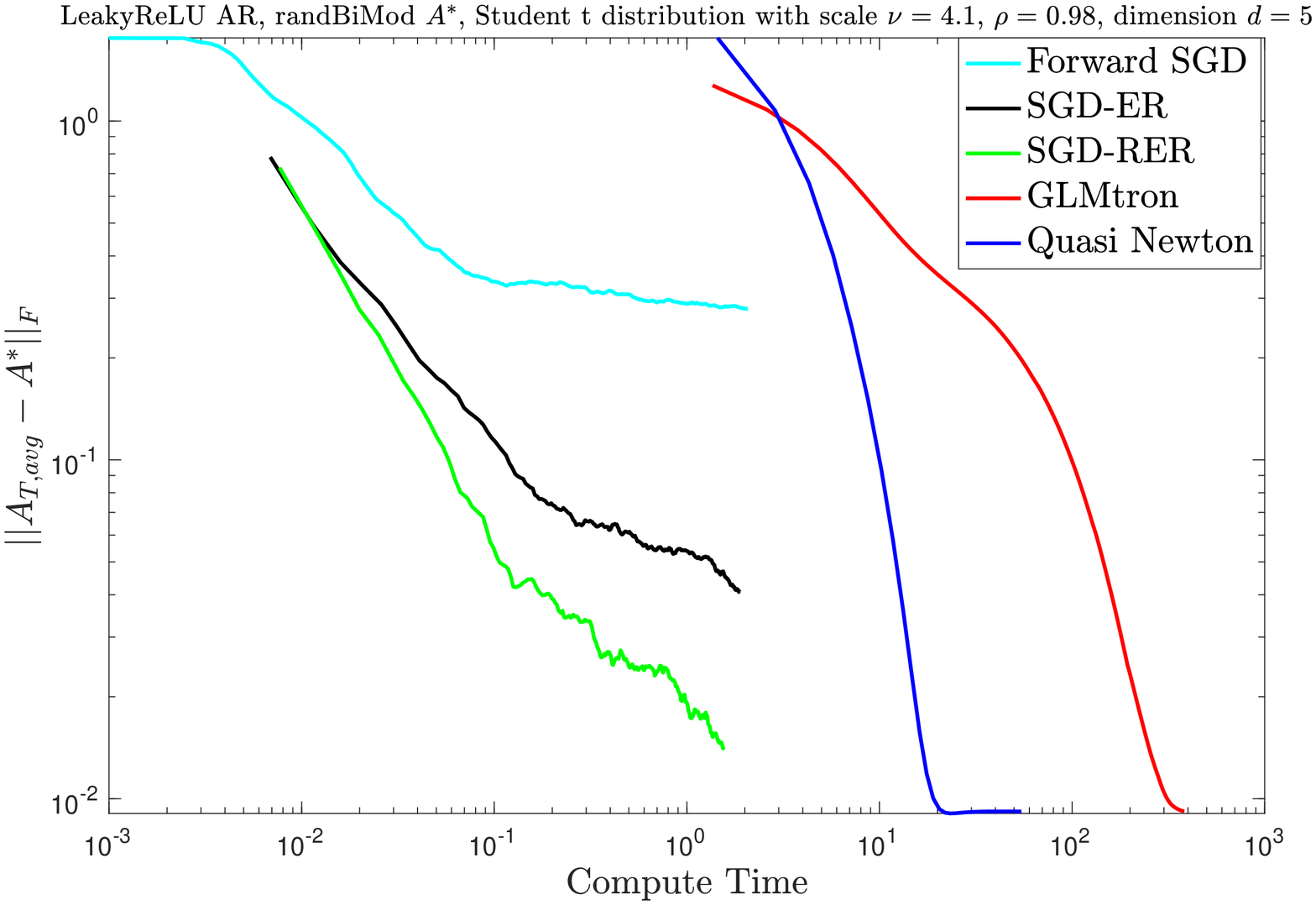}}
	\caption{Error vs. Computation}
	\label{fig:heavy_t_compute}
	\end{subfigure}
	\begin{subfigure}{.5\textwidth}
		\centering
		\includegraphics[width=\linewidth]{{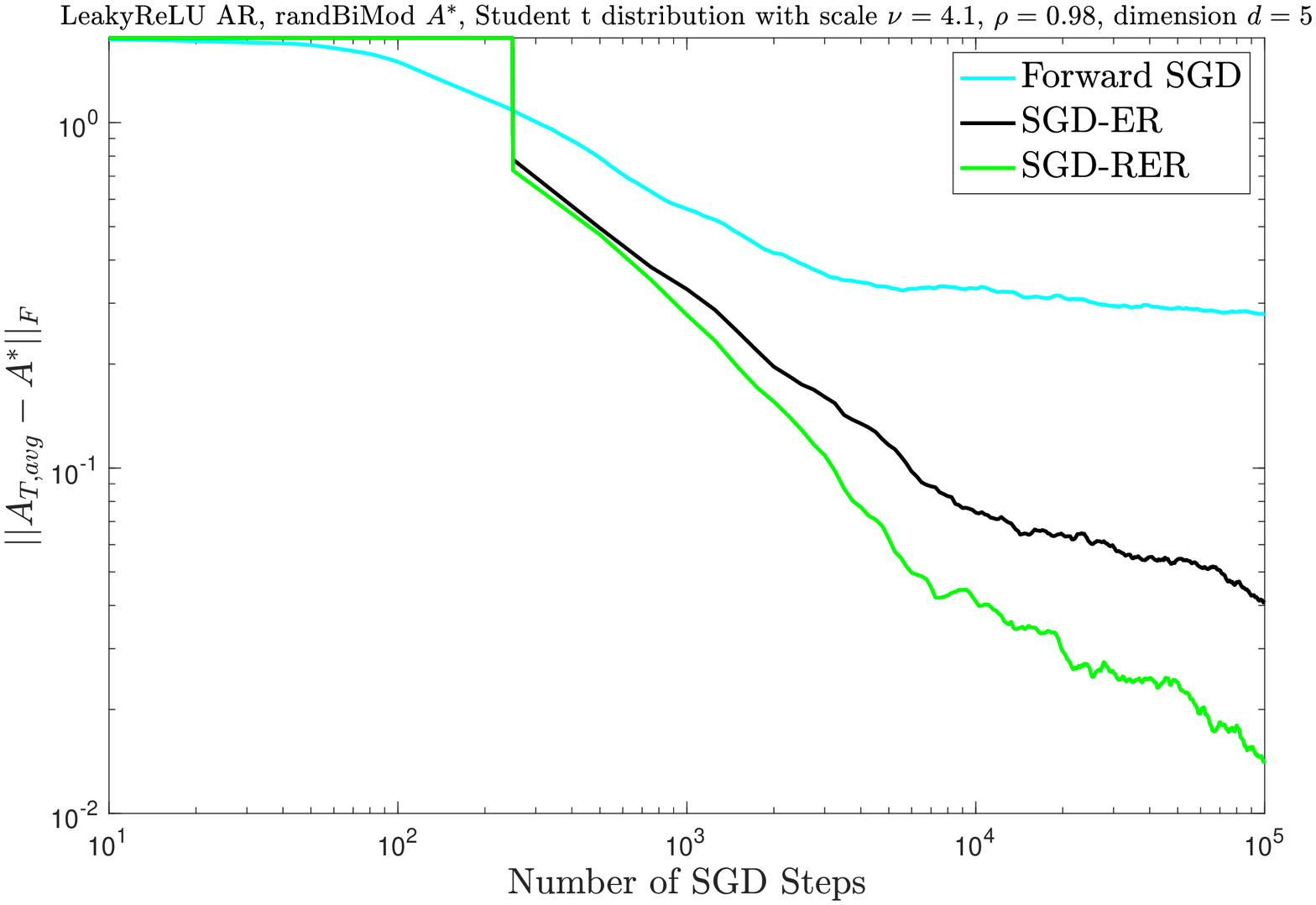}}  
		\caption{Error vs. SGD updates}
		\label{fig:heavy_t_steps}
	\end{subfigure}
\vspace*{-5pt}
	\caption{Performance of various algorithms with heavy tailed noise}
	\label{fig:ht_simulation}
\end{figure}

In Figure~\ref{fig:ht_simulation}, we consider the same system but with heavy tailed noise given by the student t distribution (scale $\nu = 4.1$) so that the $4$- th moment exists but higher moments do not. The typical behavior of Forward SGD, SGD-ER, SGD-RER and Quasi Newton methods seems to be similar to that observed in the Sub-Gaussian noise case. However, GLMtron requires much smaller step sizes to ensure convergence and hence it takes much longer. We believe that the reason for this is related to the explanation given above for GLMtron step sizes. The largest eigenvalue of the Hessian  depends on the quantity $X_t X_t^{\top}$, which can be much larger than in the sub-gaussian case and hence we need to pick much smaller step sizes. However, further research is needed to confirm this phenomenon. We also note that we did not provide theoretical guarantees for SGD-RER in the heavy tailed noise case. But it is still seen to typically perform very well.

\section{Analysis of the Quasi Newton Method}
\label{sec:gen_newton_analysis}
In this Section, we give the proofs of Theorems~\ref{thm:newton_ub_without_mixing} and~\ref{thm:newton_ub_heavy_tail}. 
Let $e_1,\dots,e_d$ be the standard basis vectors for $\mathbb{R}^d$. We will analyze the Quasi Newton method row by row. 
\begin{define}\label{def:row_decomp}
Given a matrix $A=[a_1,a_2,\cdots, a_d]^{\top}$, let $\mathcal{R}(A)=\{a_1,\cdots,a_d\}$ denote the set of vectors that are (transposes of) rows of the matrix $A$. We use $a^\top$ to represent a generic row of $A$. 
\end{define}
Follow Defintion~\ref{def:row_decomp}, we will consider the estimation of the $i$-th row $a^{*}_i$. Consider the gradient $\nabla\lossprox^{(i)} :\mathbb{R}^{d} \to \mathbb{R}^{d}$ given by:
$$\nabla\lossprox^{(i)}(a):= \frac{1}{T}\sum_{t=0}^{T-1}\left(\phi(\langle a,X_t\rangle) -\langle e_i,X_{t+1}\rangle\right)X_t \,.$$
We can write 
\begin{align}
\nabla\lossprox^{(i)}(a) &= \frac{1}{T}\sum_{t=0}^{T-1}\left(\phi(\langle a,X_t\rangle) -\phi(\langle a_i^{*}, X_t\rangle)\right)X_t - \langle\eta_t,e_i\rangle X_t\nonumber \\
&= \frac{1}{T}\sum_{t=0}^{T-1}\phi^{\prime}(\beta_t)\langle a- a_i^{*}, X_t\rangle)X_t - \langle\eta_t,e_i\rangle X_t\nonumber \\
&= \hat{K}_{a,i} (a-a_i^{*})- \hat{N}_i \label{eq:gradient_decomposition}
\end{align}
Where $\beta_t$ exist because of the mean value theorem. We can make sense of $\beta_t$ even when $\phi$ is only weakly differentiable and check that the proof below still follows.  Here, $\hat{K}_{a,i} := \frac{1}{T}\sum_{t=0}^{T-1}\phi^{\prime}(\beta_t)X_tX_t^{\intercal} $ and $\hat{N}_i := \frac{1}{T}\sum_{t=0}^{T-1}\langle\eta_t,e_i\rangle X_t$. In the first step we have used the dynamics in Equation~\eqref{eq:master_equation} to write down $X_{t+1}$ in terms of $X_t$ and $\eta_t$. \\
We now define $\hat{G} \in \mathbb{R}^{d\times d}$ by $\hat{G} := \frac{1}{T}\sum_{t=0}^{T-1}X_tX_t^{\intercal}\,.$ From the fact that $\zeta \leq\phi^{\prime}(\beta_t)\leq 1$, we note that for every $a \in \mathbb{R}^d$:
\begin{equation}\label{eq:PSD_ordering_grammian}
 \hat{G} \succeq \hat{K}_{a,i} \succeq \zeta\hat{G}
\end{equation}

Now consider the Quasi Newton Step given in Algorithm~\ref{alg:modified_newton_method}.

$$a_i(l+1) = a_i(l) - 2\gamma \hat{G}^{-1}\nabla\lossprox^{(i)}(a_i(l)) $$
Denoting $\hat{K}_{a_i(l),i}$ by $\hat{K}_{l,i}$ we use Equation~\eqref{eq:gradient_decomposition} to conclude:

\begin{align}
a_i(l+1)-a_i^{*} &= a_i(l)-a_{i}^{*} - 2\gamma \hat{G}^{-1}\hat{K}_{l,i}(a_i(l+1)-a_i^{*}) + 2\gamma \hat{G}^{-1}\hat{N}_i\nonumber \\
\implies
\sqrt{\hat{G}}(a_i(l+1)-a_i^{*}) &= (I-2\gamma\hat{G}^{-1/2}\hat{K}_{l,i}\hat{G}^{-1/2})\sqrt{\hat{G}}(a_i(l)-a_i^{*}) +2\gamma \hat{G}^{-1/2}\hat{N}_i \label{eq:generalized_difference}
\end{align}

Picking $\gamma < \frac{1}{2}$, we conclude from Equation~\eqref{eq:PSD_ordering_grammian} that:
$$(1-2\gamma)I \preceq I-2\gamma\hat{G}^{-1/2}\hat{K}_{l,i}\hat{G}^{-1/2} \preceq (1 - 2\gamma \zeta)I$$

We use the equation above in Equation~\eqref{eq:generalized_difference} along with triangle inequality to conclude:
$$\|\sqrt{\hat{G}}(a_i(l+1)-a_i^{*})\| \leq (1-2\gamma\zeta)\|\sqrt{\hat{G}}(a_i(l)-a_i^{*})\| + 2\gamma\|\hat{G}^{-1/2}\hat{N}_i\| $$

Unrolling the recursion above, we obtain that:
\begin{align}
\|\sqrt{\hat{G}}(a_i(m)-a_i^{*})\| &\leq (1-2\gamma\zeta)^{m}\|\sqrt{\hat{G}}(a_i(0)-a_i^{*})\| + \sum_{l = 0}^{m-1}2\gamma(1-2\gamma\zeta)^{l}\|\hat{G}^{-1/2}\hat{N}_i\| \nonumber\\
&\leq (1-2\gamma\zeta)^{m}\|\sqrt{\hat{G}}(a_i(0)-a_i^{*})\| + \frac{1}{\zeta}\|\hat{G}^{-1/2}\hat{N}_i\| \nonumber
\end{align}

Letting $A_{m}$ be the matrix with rows $a_i(m)$, we conclude:
\begin{equation}\label{eq:main_bound_newton}
\|A_{m}-\A\|^2_{\mathsf{F}} \leq 2\tfrac{\lmax{\hat{G}}}{\lmin{\hat{G}}} (1-2\gamma\zeta)^{2m}\|A_0 -\A\|^2_{\mathsf{F}} + \tfrac{2}{\zeta^2\lmin{\hat{G}}} \sum_{i=1}^{d}\|\hat{G}^{-1/2}\hat{N}_i\|^2
\end{equation}

Proof of Theorems~\ref{thm:newton_ub_without_mixing} and~\ref{thm:newton_ub_heavy_tail} follow once we provide high probability bounds for various terms in Equation~\eqref{eq:main_bound_newton}. We will first define some notation. Let $S(\rho,T) := \sum_{t=0}^{T}\rho^{T-t}$.
For $R,\kappa > 0$, we define the following events
\begin{enumerate}
\item $\mathcal{D}_{T}(R) := \{\sup_{0\leq t\leq T}\|X_t\|^2 \leq R\}$ 
\item $\mathcal{E}_T(\kappa) := \{\hat{G} \succeq \frac{\sigma^2I}{\kappa}\}$
\item $\mathcal{D}_T(R,\kappa):= \mathcal{D}_T(R)\cap\mathcal{E}_T(\kappa)$
\end{enumerate}
\begin{lemma}\label{lem:normalized_noise}
Under the Assumptions of Theorem~\ref{thm:newton_ub_without_mixing},
suppose $\delta \in (0,1/2)$ and take $R = C^2_{\rho}C_{\eta}(S(\rho,T))^2d\sigma^2\log (\tfrac{2T}{\delta}) $, $\kappa = 2$, and $T \geq \bar{C}_3 \left(d\log\left(\frac{R}{\sigma^2}\right) + \log\tfrac{1}{\delta}\right)$

$$\mathbb{P}\left(\sum_{i=1}^{d}\|\hat{G}^{-\tfrac{1}{2}}\hat{N}_i\|^2 \leq \frac{\bar{C}\sigma^2}{T}\left[d^2 \log\left(1+\tfrac{R}{\sigma^2}\right)  + d\log (\tfrac{1}{\delta})\right]\cap \mathcal{D}_T(R,\kappa)\right) \geq 1-2d\delta$$

Where $\bar{C},\bar{C}_3$ are constants depending only on $C_{\eta}$.
\end{lemma}
We refer to Section~\ref{sec:self_normalized_noise} for the proof. 

\begin{lemma}\label{lem:well_conditioned_grammian}
Suppose the Assumptions  of Theorem~\ref{thm:newton_ub_heavy_tail} hold. There exist universal constants $C,C_1,c_0 > 0$ such that whenever $\delta \in (0,\tfrac{1}{2})$, $R :=  \frac{4TdC_{\rho}^2\sigma^2}{(1-\rho)^2 \delta}$ and $T \geq Cd\log(\tfrac{1}{\delta})\log(\tfrac{R}{\sigma^2}) \max\left(\frac{4C_{\rho}^6M_4}{(1-\rho)^4(1-\rho^2)\sigma^4}, \frac{\log\left( \tfrac{R C_1 C_{\rho}}{\sigma^2}\right)}{\log\left(\tfrac{1}{\rho}\right)}\right)$, we have with probability at-least $1-\delta$:
\begin{enumerate}
\item $$\hat{G} \succeq c_0 G \,.$$
\item $$\lmax{\hat{G}} \leq R$$
\end{enumerate}

Where $G:= \mathbb{E}X_tX_t^{\intercal}$ 

\end{lemma}
We give the proof of this lemma in Section~\ref{sec:technical_proofs}.

\subsection{Proof of Theorem~\ref{thm:newton_ub_without_mixing}}

\begin{proof}
Note that $R$ in Lemma~\ref{lem:normalized_noise} is the same as $R^{*}$ in the statement of Theorem~\ref{thm:newton_ub_without_mixing}. We combine the result of Lemma~\ref{lem:normalized_noise} with the Equation~\eqref{eq:main_bound_newton}. Under the event $\mathcal{D}_T(R^{*},\kappa)$, we have  $\tfrac{\lmax{\hat{G}}}{\lmin{\hat{G}}} \leq \tfrac{2R}{\sigma^2}$ almost surely. Hence the result follows. 

\end{proof}

\subsection{Proof of Theorem~\ref{thm:newton_ub_heavy_tail}}
\begin{proof}

We again begin with Equation~\eqref{eq:main_bound_newton}. 
$$
\|A_{m}-\A\|^2_{\mathsf{F}} \leq 2\tfrac{\lmax{\hat{G}}}{\lmin{\hat{G}}} (1-2\gamma\zeta)^{2m}\|A_0 -\A\|^2_{\mathsf{F}} + \tfrac{2}{\zeta^2\lmin{\hat{G}}} \sum_{i=1}^{d}\|\hat{G}^{-1/2}\hat{N}_i\|^2$$

Let the event described in Lemma~\ref{lem:well_conditioned_grammian} be $\mathcal{W}$. Under this event, $\lmin{\hat{G}} \geq c_0\lmin{G}$, $\lmax{G} \leq R$ and $\hat{G}^{-1} \preceq \frac{1}{c_0}G^{-1}$. Therefore, we conclude:
$$
\|A_{m}-\A\|^2_{\mathsf{F}} \leq 2\tfrac{R}{c_0\lmin{G}} (1-2\gamma\zeta)^{2m}\|A_0 -\A\|^2_{\mathsf{F}} + \tfrac{2}{\zeta^2c_0^2\lmin{G}} \sum_{i=1}^{d}\hat{N}_i^{\intercal}G^{-1}\hat{N}_i$$
\end{proof}
Now, consider $\mathbb{E}\|A_m - \A\|^2_{\mathsf{R}}\mathbbm{1}(\mathcal{W})$. To conclude the result of the Theorem, we will show that $\mathbb{E}\hat{N}_i^{\intercal}G^{-1}\hat{N}_i \mathbbm{1}(\mathcal{W}) \leq \frac{d\sigma^2}{T} $. Indeed:
\begin{align}
\mathbb{E}\hat{N}_i^{\intercal}G^{-1}\hat{N}_i \mathbbm{1}(\mathcal{W}) &\leq \mathbb{E}\hat{N}_i^{\intercal}G^{-1}\hat{N}_i \nonumber \\
 &= \frac{1}{T^2}\sum_{t,s=0}^{T-1} \mathbb{E} \langle\eta_t,e_i\rangle\langle\eta_s,e_i\rangle X_t^{\intercal}G^{-1}X_s \label{eq:matrix_product_expectation}
\end{align}
Observe that whenever $t > s$, $\eta_t$ has mean zero and is independent of $X_t,X_s$ and $\eta_s$. Therefore, $\mathbb{E} \langle\eta_t,e_i\rangle\langle\eta_s,e_i\rangle X_t^{\intercal}G^{-1}X_s = 0$. When $t = s$, we conclude using the fact that $\mathbb{E}X_tX_t^{\intercal} = G$ that $\mathbb{E} \langle\eta_t,e_i\rangle^2 X_t^{\intercal}G^{-1}X_t = d \sigma^2$. 
Using the calculations above and Equation~\eqref{eq:matrix_product_expectation}, we conclude that:

$$\mathbb{E}\|A_m - \A\|^2_{\mathsf{R}}\mathbbm{1}(\mathcal{W}) \leq   \frac{2d^2\sigma^2}{c_0^2\zeta^2 T \lmin{G}} + 2\tfrac{R}{c_0\lmin{G}} (1-2\gamma\zeta)^{2m}\|A_0 -\A\|^2_{\mathsf{F}} $$

\section{Analysis of SGD-RER}
\label{sec:sgd_rer_analysis}
In this section we consider the following $(X_0,X_1,\dots,X_T)$ to be a stationary sequence from $\glvar(\A,\mu,\phi)$. We make Assumptions
~\ref{assump:3},~\ref{assump:8},~\ref{assump:7} and ~\ref{assump:6}. We aim to analyze Algorithm~\ref{alg:sgd_rer} and then prove Theorem~\ref{thm:sgd_rer_ub}. 

The data is divided into buffers of size $B$ and the buffers have a gap of size $u$ in between them. Let $S = B+u$. The algorithm runs SGD with respect to the proxy loss $\lossprox$ in the order described in Section~\ref{sec:ideas_behind_proofs}. Formally, let $X^t_j\equiv X_{tS+j}$ denote the the $j$-th sample in buffer $t$. We denote, for $0\leq i\leq B-1$, $X^t_{-i}\equiv X^t_{(S-1)-i}$ i.e., the $i$-th processed sample in buffer $t$. We use similar notation for noise samples i.e., $\eta^t_j\equiv\eta_{tS+j}$ and $\eta^t_{-j} \equiv \eta^t_{(S-1)-j}$.

The algorithm iterates are denoted by the sequence $(A^t_i: 0\leq t\leq N-1,\, 0\leq i\leq B-1)$ where $A^t_i$ denotes the iterate obtained after processing $i$-th (reversed) sample in buffer $t$ and $N=T/S$ is the total number of buffers. Note that we enumerate buffers from $0,1,\cdots N-1$. Formally
\begin{equation}
\label{eq:SGD-RER_iterate}
A^{t-1}_{i+1}=A^{t-1}_{i}-2\gamma\left(\phi(A^{t-1}_i X^{t-1}_{-i})-X^{t-1}_{-(i-1)}\right)X^{t-1,\top}_{-i}
\end{equation}
for $1\leq t\leq N$, $0\leq i\leq B-1$ and we set $A^t_0=A^{t-1}_{B}$ with $A^0_0=A_0$. 

The algorithm outputs the tail-averaged iterate at the end of each buffer $t$: $\hat A_{t_0,t}=\frac{1}{t-t_0}\sum_{\tau=t_0+1}^{t}A^{\tau-1}_B $ where $1\leq t\leq N$ and $0\leq t_0\leq t-1$.

\subsection{Proof Strategy}
\label{subsec:proof_strategy}
The proof of Theorem~\ref{thm:sgd_rer_ub} involves many intricate steps. Therefore, we give a detailed overview about the proof below.
\begin{enumerate}
\item In Section~\ref{subsec:coupling}  we first construct a fictitious coupled process $\tilde{X}_{\tau}$ such that for every data point within a buffer $t$, $\|\tilde{X}_{\tau} - X_{\tau}\|\lesssim \frac{1}{T^{\alpha}}$ for some fixed $\alpha > 0$ chosen arbitrarily beforehand. We then show that the iterates $\tilde{A}^{t}_{i}$ which are generated with $\sgdber$ is run with the coupled process $\tilde{X}_{\tau}$ is very close to the actual iterate $A^{t}_i$. The coupled process has the advantage that the data in the successive buffers are independent. We then only deal with the coupled iterates $\tilde{A}^{t}_i$ and appeal to Lemma~\ref{lem:coupled_iterate_replacement} to obtain bounds for $A^{t}_i$.

\item In Section~\ref{subsec:bias_variance_decomp}, we give the bias variance decomposition as is standard in the linear regression literature. We extend it to the non-linear case using the mean value theorem and treat the buffers as independent data samples. Here, the matrices $\Htt{s}{0}{B-1}$ defined on the data in buffer $s$ `contracts' the norm of $A^{s}_{i} - \A$ giving the `bias term' whereas the noise $\eta^{s}_i$ presents the `variance' term which is due to the inherent uncertaintly in the estimation problem.

\item We refer to Section~\ref{sec:grammian_well_conditioned} where we develop the contraction properties of the matrices $\prod_{s=0}^{t}\Htt{s}{0}{B-1}$  where we show that $\|\prod_{s=0}^{t}\Htt{s}{0}{B-1}\| \lesssim (1-\zeta \gamma B \lmin{G})^t$ in Theorem~\ref{thm:buffer_norm_upper_bound} after developing some probabilistic results regarding $\glvar(\A,\mu,\phi)$. This allows us show exponential decay of the bias.

\item We then turn to the squared variance term in Section~\ref{subsec:last_var}. We decompose it  into `diagonal terms' with non-zero expectation and `cross terms' with a vanishing expectation. Bounding the diagonal term is straight forward using standard recursive arguments and we give the bound in Claim~\ref{claim:dg_bound}. 

\item The `cross terms' which vanish in expectation in the linear case, do not because of the coupling introduced by the non-linearities through the iterates (see Section~\ref{sec:ideas_behind_proofs} for a short description). However, we establish `algorithmic stability' in Section~\ref{subsec:algorithmic_stability} where we show that the iterates depend only weakly on each of the noise vectors and hence the cross terms have expectation very close to zero. More specifically, we use the novel idea of re-sampling the whole trajectory ($\tilde{X}_{\tau}$) by re-sampling one noise vector only and show that the iterates of $\sgdber$ are not affected much. 

\item We use the `algorithmic stability' bounds to bound the cross terms in Sections~\ref{subsec:cross_bound}. We then combine the bounds to obtain the bound on the `variance term'

\item Finally, we analyze the tail averaged output in Sections~\ref{subsec:avg_decomp},~\ref{subsec:avg_var} and~\ref{subsec:avg_bias} and then combine these ingredients to prove Theorem~\ref{thm:sgd_rer_ub}.

\end{enumerate}

\subsection{Basic Notations and Coupled Process}
\label{subsec:coupling}

\begin{define}[Coupled process]
\label{def:coupled_proc}
Given the co-variates $\{X_\tau:\tau=0,1,.\cdots T\}$ and noise $\{\eta_\tau:\tau=1,2,\cdots,T\}$, we define $\{\tilde{X}_\tau:\tau=0,1,\cdots,T\}$ as follows:
\begin{enumerate}
\item For each buffer $t$ generate, independently of everything else, $\tilde{X}^t_0\distas{}\pi$, the stationary distribution of the $\glvar(\A,\mu,\phi)$ model.

\item Then, each buffer has the same recursion as eq \eqref{eq:master_equation}:
\begin{equation}
\label{eq:coupling}
\tilde{X}^t_{i+1}=\phi(\A \tilde{X}^t_i)+\eta^t_i,\, i=0,1,\cdots S-1, 
\end{equation}
where the noise vectors as same as in the actual process $\{X_\tau\}$.
\end{enumerate}
\end{define} 

\begin{lemma}[Coupling Lemma]\label{lem:coupling_lemma}
Under Assumption~\ref{assump:4}, for any buffer $t$, we have $\norm{X^t_i-\tilde X^t_i}\leq C_{\rho}\rho^i\norm{X^t_0-\tilde X^t_0},\, a.s$. Hence 
\begin{equation}
\label{eq:xx^T_coup_contract}
\norm{X^{t}_{i}X^{t,\top}_{i}-\tilde{X}^t_i \tilde X^{t,\top}_{i}}\leq 2(\sup_{\tau\leq T}\norm{X_{\tau}})\norm{X^{t}_{i}-\tilde X^{t}_{i}}\leq 4\sup_{\tau\leq T}\norm{X_\tau}^2 C_{\rho}\rho^i
\end{equation}
\end{lemma}

With the above notation, we can write \eqref{eq:SGD-RER_iterate} in terms of a generic row (say row $r$) $a^{t-1,\top}_{i+1}$ of $A^{t-1}_{i+1}$ as follows. Let $\varepsilon^{t-1}_{-i}$ denote the element of $\eta^{t-1}_{-i}$ in row $r$. Similarly let $a^{*,\top}\equiv \left(\a\right)^{\top}$ denote the row $r$ of $\A$.  Then

\begin{equation}
\label{eq:SGD-RER_iterate_row}
a^{t-1,\top}_{i+1}=a^{t-1,\top}_{i}-2\gamma\left(\phi( X^{t-1,\top}_{-i}a^{t-1}_i)-\phi(X^{t-1,\top}_{-i}\a)\right)X^{t-1,\top}_{-i} +2\gamma \varepsilon^{t-1}_{-i}X^{t-1,\top}_{-i}
\end{equation}

Now, by the mean value theorem we can write 
\begin{equation}
\label{eq:SGD-RER_mvt}
\phi( X^{t-1,\top}_{-i}a^{t-1}_i)-\phi(X^{t-1,\top}_{-i}\a)=\phi'(\xi^{t-1}_{-i})(a^{t-1}_{-i}-\a)^{\top}X^{t-1}_{-i}
\end{equation}
where $\xi^{t-1}_{-i}$ lies between $X^{t-1,\top}_{-i}a^{t-1}_i$ and $X^{t-1,\top}_{-i}\a$. Hence we obtain

\begin{equation}
\label{eq:SGD-RER_iterate_row_1}
(a^{t-1}_{i+1}-\a)^{\top}=(a^{t-1}_{i}-\a)^{\top}(I-2\gamma\phi'(\xi^{t-1}_{-i})X^{t-1}_{-i}X^{t-1,\top}_{-i})+2\gamma \varepsilon^{t-1}_{-i}X^{t-1,\top}_{-i}
\end{equation}


 \begin{define}[Coupled SGD Iteration]
 Consider the process described in Defintion~\ref{def:coupled_proc}. We define $\sgdber$ iterates run with the coupled process ($\tilde{X}_i^{t}$) as follows:
 $$\tilde{A}^{0}_0 = A^{0}_0$$
 \begin{equation}
\label{eq:coupled_iterate}
\tilde{A}^{t-1}_{i+1}=\tilde{A}^{t-1}_{i}-2\gamma\left(\phi(\tilde{A}^{t-1}_i \tilde{X}^{t-1}_{-i})-\tilde{X}^{t-1}_{-(i-1)}\right)\tilde{X}^{t-1,\top}_{-i}
\end{equation}
 \end{define}
Using Lemma~\ref{lem:coupling_lemma}, we can show that $\tilde{A}^{t}_i \approx A^{t}_i$. Note that successive buffers for the iterates $\tilde{A}^{t}_i$ are actually independent.  We state the following lemma which shows that we can indeed just analyze $\tilde{A}^{t}_i$ and then from this obtain error bounds for $A^{t}_i$. We refer to Section~\ref{sec:technical_proofs} for the the proof.

%

\begin{lemma}
\label{lem:coupled_iterate_replacement}
 Suppose $\gamma < \frac{1}{2R_{\max}}$. we have for every $t \in [N]$ and $i \in [B]$.
$$\|a^{t}_i - \tilde{a}^{t}_i\| \leq (16\gamma^2R_{\max}^2T^2 + 8\gamma R_{\max}T) \rho^u$$
\end{lemma}

We note that we can just analyze the iterates $\tilde{A}^{t}_{i}$ and then use Lemma~\ref{lem:coupled_iterate_replacement} to infer error bounds for $A^{t}_i$. Henceforth, we will only consider $\tilde{A}^{t}_{i}$.

Before proceeding, we will set up some notation. 
\subsection{Notations and Events}
\label{subsec:notations_events}
We define the following notations. Let $R>0$ to be decided later. 
\begin{align*}
&X^t_{-i}=X^t_{(S-1)-i},\, 0\leq i\leq S-1, \quad
 \quad \phi'(\tilde \xi^t_{-i})=\frac{\phi(
\tilde{a}^{t,\top}_{-i}\tilde X^t_{-i})-\phi(\atr \tilde X^t_{-i})}{\left(\tilde{a}^{t}_{-i}-\a\right)^{\top}\tilde X^t_{-i}} \\
&\Ppt{t}{-i}=\Ptt{t}{-i},\quad 
\Htt{t}{i}{j}=\begin{cases}\prod_{s=i}^{j}\Ppt{t}{-s} & i\leq j\\
	I & i>j 
\end{cases},\\
&\hat{\gamma}=4\gamma(1-\gamma R),\quad 
\cc^t_{-j}=\left\{\norm{X^{t}_{-j}}^2\leq R\right\},\quad
\cct^{t}_{-j}=\left\{\norm{\tilde X^{t}_{-j}}^2\leq R\right\},\\
&\cd^t_{-j}=\left\{\norm{X^{t}_{-i}}^2\leq R:\,j\leq i\leq B-1\right\}=\bigcap_{i=j}^{B-1}\cc^t_{-i},\\
&\cd^{s,t}=\begin{cases} \bigcap_{r=s}^t \cd^r_{-0} & s\leq t\\
\Omega & s>t 
\end{cases},\quad
\tilde\cd^t_{-j}=\left\{\norm{\tilde X^{t}_{-i}}^2\leq R:\,j\leq i\leq B-1\right\}=\bigcap_{i=j}^{B-1}\cct^t_{-i},\\
&\tilde \cd^{s,t}=\begin{cases} \bigcap_{r=s}^t \tilde \cd^r_{-0} & s\leq t\\
\Omega & s>t 
\end{cases},\quad
\cdh^t_{-j}=\cd^{t}_{-j}\cap \cdt^t_{-j},\quad
\cdh^{s,t}=\cd^{s,t}\cap \cdt^{s,t}.
\end{align*}

To execute algorithmic stability arguments, we will need to independently resample individual noise co-ordinates. To that end, define $(\bar{\eta}_{\tau})_{\tau}$ drawn i.i.d from the noise distribution $\mu$ and independent of everything else defined so far. We denote their generic rows by $\bar{\varepsilon}$. We use the following events which correspond to a generic row
\begin{\Ieee}{LLL}
\ce^t_{i,j}=\left\{\| \varepsilon^{t}_{-k}\|^2\leq \beta,\| \bar{\varepsilon}^{t}_{-k}\|^2\leq \beta  :\,i\leq k\leq j\right\}
\end{\Ieee}

\subsection{Setting the Parameter Values:}
\label{subsubsec:parameter_setting}
 We make Assumptions~\ref{assump:3},~\ref{assump:8},~\ref{assump:7} and ~\ref{assump:6} throughout. 
We set the parameters for $\sgdber$ as follows for the rest of the analysis. We note that some of these parameter values were set in Section~\ref{sec:sgd_rer_results}. 
\begin{enumerate}
\item $\alpha \geq 10$
\item $\beta = 4C_{\eta}\sigma^2(\alpha+2)\log{2T}$.
\item $R \geq \frac{16 (\alpha + 2) d C_{\eta}\sigma^2 \log{T}}{1-\rho}$
\item $\delta=1/(2T^{\alpha+1})$
\item $u \geq \frac{2\alpha\log{T}}{\log (\tfrac{1}{\rho})} = O(\tmix \log T)$
\item $B \geq \max\left(\bar{C}_1\frac{d}{(1-\rho)(1-\rho^2)},10u\right)$ where $\bar{C}_1$ depends only on $C_{\eta}$(see Theorem~\ref{thm:buffer_norm_upper_bound}) 
\item $\gamma \leq \min\left(\frac{\zeta}{4BR(1+\zeta)},  1/2R\right)$ (see Theorem~\ref{thm:buffer_norm_upper_bound})
\end{enumerate}
From Assumption~\ref{assump:7} and Theorem~\ref{thm:process_concentration}, we conclude that for this choice of $R$ and $\beta$, we must have:

\begin{\Ieee}{LLL}
\label{eq:crude_bnd_cor_1}
\Pb{\left(\cdh^{0,N-1}\cap \cap_{r=0}^{N-1}\ce^r_{0,B-1} \right)^C}\leq \prbndH \Ieeen
\end{\Ieee}

\subsection{Bias-variance decomposition}
\label{subsec:bias_variance_decomp}
Using the above notation we can unroll the recursion in \eqref{eq:SGD-RER_iterate_row_1} as follows. We will only focus on the algorithmic iterated at the end of each buffer, i.e., we set $i=B-1$ in \eqref{eq:SGD-RER_iterate_row_1}. 

\begin{align}
\label{eq:SGD-RER_iterate_row_2}
(\tilde{a}^{t-1}_{B}-\a)^{\top} = (a_0-\a)^{\top}\prod_{s=0}^{t-1}\Htt{s}{0}{B-1}+ 2\gamma\sum_{r=1}^{t}\sum_{j=0}^{B-1}\Nt{t-r}{-j}\Xtttr{t-r}{-j}\Htt{t-r}{j+1}{B-1}{\prod_{s=r-1}^{1}\Htt{t-s}{0}{B-1}} 
\end{align}
We call the above the \emph{bias-variance} decomposition where
\begin{equation}
\label{eq:SGD-RER_bias}
(\tilde{a}^{t-1,b}_{B}-\a)^{\top}=(a_0-\a)^{\top}\prod_{s=0}^{t-1}\Htt{s}{0}{B-1}
\end{equation}
is the bias, and 
\begin{equation}
\label{eq:SGD-RER_variance}
(\tilde{a}^{t-1,v}_{B})^{\top}=2\gamma\sum_{r=1}^{t}\sum_{j=0}^{B-1}\Nt{t-r}{-j}\Xtttr{t-r}{-j}\Htt{t-r}{j+1}{B-1}{\prod_{s=r-1}^{1}\Htt{t-s}{0}{B-1}} 
\end{equation}
is the variance. We have the following simple lemma on bias-variance decomposition.
\begin{lemma}
\label{lem:SGD-RER_bias_var}
\begin{\Ieee}{LLL}
\norm{\tilde{a}^{t-1}_B-\a}^2 &\preceq & 2\left(\norm{\tilde{a}^{t-1,b}_B-\a}^2+ \norm{\tilde{a}^{t-1,v}_{B}}^2\right)\Ieeen\label{eq:SGD-RER_bias_var}
\end{\Ieee}
\end{lemma}

\subsection{Variance of last iterate - Diagonal Terms}
\label{subsec:last_var}
In this section our goal is to decompose $\norm{\tilde{a}^{t-1,v}_B}^2$ into diagonal terms and cross terms. We will then proceed to bound the diagonal terms. First, we have a preliminary lemma, which can be shown via a simple recursion.

\begin{lemma}
\label{lem:var_diag_recursion}
For $k\leq t$ define $S_k^t$ as 
\begin{\Ieee}{LLL}
\label{eq:var_Skt}
S_k^t=\sum_{r=k}^t\sum_{j=0}^{B-1}\phi'(\tilde \xi^{t-r}_{-j})\prodHtttr{t}{r-1}\Htttr{t-r}{j+1}{B-1}\Xtt{t-r}{-j}\Xtttr{t-r}{-j}\Htt{t-r}{j+1}{B-1}\prodHtt{t}{r-1}\Ieeen
\end{\Ieee}

Then, on the event $\cdt^{0,t-1}$, we have
\begin{\Ieee}{LLL}
\label{eq:var_S1t_bnd}
S_1^t\preceq \frac{1}{\gammah}\left(I-\prodHtttr{t}{t}\prodHtt{t}{t}\right)\Ieeen
\end{\Ieee}
where $\gammah=4\gamma(1-\gamma R)$
\end{lemma}

\begin{proof}
The proof is similar to that of \cite[Claim 1]{jain2021streaming}. 
\end{proof}

Next, we write $\norm{\tilde{a}^{t-1,v}_B}^2$ as 
\begin{\Ieee}{LLL}
\label{eq:var_last_iter_3}
\norm{\tilde{a}^{t-1,v}_B}^2=\sum_{r=1}^t\sum_{j=0}^{B-1}\dg(t,r,j)+\sum_{r_1,r_2}\sum_{j_1,j_2}\cro(t,r_1,r_2,j_1,j_2)\Ieeen
\end{\Ieee}
where the second sum is over $(r_1,j_1)\neq (r_2,j_2)$ and
\begin{align}
\dg(t,r,j)= 4\gamma^2 |\varepsilon^{t-r}_{-j}|^2\cdot\Xtttr{t-r}{-j}\Htt{t-r}{j+1}{B-1}\prodHtt{t}{r-1}\prodHtttr{t}{r-1}\Htttr{t-r}{j+1}{B-1}\Xtt{t-r}{-j}\label{eq:var_dg}
\end{align}
and
\begin{\Ieee}{LLL}
\cro(t,r_1,r_2,j_1,j_2)&=& 4\gamma^2 \Nt{t-r_2}{-j_2}\Xtttr{t-r_2}{-j_2}\Htt{t-r_2}{j_2+1}{B-1}\prodHtt{t}{r_2-1}\cdot\\
&&\prodHtttr{t}{r_1-1}\Htttr{t-r_1}{j_1+1}{B-1}\Xtt{t-r_1}{-j_1}\Nt{t-r_1}{-j_1} \Ieeen\label{eq:var_cro}
\end{\Ieee}

Finally, we bound the diagonal term:

\begin{claim}
\label{claim:dg_bound}
\begin{\Ieee}{LLL}
\label{eq:dg_bound}
\Ex{\sum_{r=1}^{t}\sum_{j=0}^{B-1}\dg(t,r,j)\indt{0}{t-1}}&\leq &  \frac{\gamma d}{\zeta(1-\gamma R)}\beta  +16 C_{\eta}\sigma^2\gamma^2 R T\prbndsq  \Ieeen
\end{\Ieee}

\end{claim}

\begin{proof}
Notice that we can write
\begin{\Ieee}{LLL}
\label{eq:dg_bound_1}
\dg(t,r,j) &\leq & 4\gamma^2\left(\beta +|\Nt{t-r}{-j}|^2 1[|\Nt{t-r}{-j}|^2>\beta]\right)\cdot\\
&& \Xtttr{t-r}{-j}\Htt{t-r}{j+1}{B-1}\prodHtt{t}{r-1}\prodHtttr{t}{r-1}\Htttr{t-r}{j+1}{B-1}\Xtt{t-r}{-j} \Ieeen
\end{\Ieee}

Further
\begin{\Ieee}{LLL}
\label{eq:dg_bound_2}
\Xtttr{t-r}{-j}\Htt{t-r}{j+1}{B-1}\prodHtt{t}{r-1}\prodHtttr{t}{r-1}\Htttr{t-r}{j+1}{B-1}\Xtt{t-r}{-j}\indt{0}{t-1}\leq R\Ieeen
\end{\Ieee}

Combining the above two we obtain
\begin{\Ieee}{LLL}
\label{eq:dg_bound_3}
\sum_{r=1}^{t}\sum_{j=0}^{B-1}\dg(t,r,j)\indt{0}{t-1} &\leq & 4\gamma^2\beta \frac{\tr S_1^t}{\zeta}+R\sum_{r=1}^t\sum_{j=0}^{B-1}|\Nt{t-r}{-j}|^2 1[|\Nt{t-r}{-j}|^2>\beta] \Ieeen
\end{\Ieee}
where $S_1^t$ is defined in~\eqref{eq:var_Skt}. Now taking expectation, and using lemma~\ref{lem:var_diag_recursion} and Caucy-Schwarz inequality for the first and second terms, respectively, in \eqref{eq:dg_bound_3} we obtain the claim. Here we use the fact that $\Ex{|\Nt{t-r}{-j}|^4}\leq 16 C_{\eta}^2\sigma^4$ from \cite[Theorem 2.1]{boucheron2013concentration}
\end{proof}

\subsection{Algorithmic stability}
\label{subsec:algorithmic_stability}
In order to bound the cross terms in the variance, we need the notion of algorithmic stability. Here the idea is that if $\phi$ was identity, then $\Ex{\cro(t,r_1,r_2,j_1,j_2)}$ would vanish. But in the non-linear setting, this does not happen due to dependencies between $\Nt{t-r}{-j}$ and $\Htt{t-r}{j+1}{B-1}\prodHtt{t}{r-1}$ through the algorithmic iterates. We can still show that $\Ex{\cro(t,r_1,r_2,j_1,j_2)} \approx 0$ by showing that the iterates depend very weakly on each of the noise co-ordinates $\Nt{t-r}{-j}$. So our idea is to use algorithmic stability: we re-sample the whole trajectory of $X$ by re-sampling a single noise co-ordinate independently. We then show that the iterates are not affected much by such a re-sampling, which shows that the iterates are only weakly coupled to each individual noise vector.

To that end, we need some additional notation. We have the data $(X_{\tau})_{\tau}$ and the coupled process $(\tilde{X}_{\tau})_{\tau}$. Let the corresponding (coupled) algorithmic iterates be $(\tilde{a}^s_i: 0\leq s\leq N-1,\,0\leq i\leq B-1)$. Now $\tilde{a}_i^{s}$ are functions of $X_0$ and noise vectors $\{\eta^s_{-i}:0\leq s\leq N-1,\, 0\leq i\leq S-1\}$.  Suppose we re-sample the noise $\eta^r_{-j}$ independently of everything else to get $\bar \eta^r_{-j}$. So the new noise samples are:
$$\left(\eta^0_0,\eta^0_1,\cdots, \eta^r_0,\cdots,\eta^r_{(S-1)-(j+1)},\bar\eta^r_{(S-1)-j},\eta^r_{(S-1)-(j-1)},\cdots\right)\,.$$

We then run the dynamics in Equation~\eqref{eq:master_equation} with the new noise samples to obtain $(\bar X_{\tau})_{\tau}$ and the new coupled process $(\bar{\tilde X}_{\tau})_{\tau}$ obtained through the new noise sequence (but same stationary renewal given in Definition~\ref{def:coupled_proc}), and they satisfy the following:

\begin{\Ieee}{LLL}
\bar X^s_{-i}=\begin{cases}
X^s_{-i}, & s<r,\, 0\leq i\leq S-1\\
X^r_{-i}, & s=r,\, j\leq i\leq S-1\\
\end{cases}\\
\bar{\tilde X}^{s}_{-i}=\begin{cases}
\tilde X^s_{-i}, & s\in\{1,\cdots,r-1,r+1,\cdots,N-1\},\, 0\leq i\leq S-1\\
\tilde X^r_{-i}, & s=r,\, j\leq i\leq S-1\\
\end{cases}
\end{\Ieee}

We obtain the iterates $\bar{\tilde{a}}^s_i$ by running the update Equation~\eqref{eq:SGD-RER_iterate} with the data $\bar{\tilde{X}}_{\tau}$ instead of $X_{\tau}$. Accordingly, the algorithmic iterates change to $(\bar{\tilde{a}}^s_i: 0\leq s\leq N-1,\,0\leq i\leq B-1)$ that satisfy
\begin{\Ieee}{LLL}
\bar{\tilde{a}}^s_i =\tilde{a}^s_i \quad\mathrm{for}\, s<r,\, 0\leq i\leq B-1
\end{\Ieee}
This is because, resampling $\eta_{\tau}$ does note change the value of data $\tilde{X}_{\tau^{\prime}}$ for $\tau^{\prime} \leq \tau$.
Under the setting we have the following lemma:
\begin{lemma}
\label{lem:algo-stability}

Let $\cA^{t-1}$ be the following event
\begin{\Ieee}{LLL}
\label{eq:event_A}
\cA^{t-1}=\bigcap_{r=0}^{t-1}\bigcap_{j=0}^{B-1}\left\{\norm{\tilde{a}^{r}_{j}-\a}\leq \norm{a_0-\a}+\bar{C}\frac{\sqrt{R\beta}}{\zeta\lambda_{\min}}\right\}\Ieeen
\end{\Ieee}
For some constant $\bar{C}$ depending only on $C_{\eta}$, we have for any $1\leq t\leq N$
\begin{\Ieee}{LLL}
\label{eq:event_A_probab}
\Pb{ \cdh^{0,N-1}\cap\cA^{N-1} \cap\cap_{r=0}^{N-1} \ce^{r}_{0,B-1}}\geq 1-\prbnd\Ieeen 
\end{\Ieee}

Further more, on the event $\ce^{r}_{0,j}\cap \cdt^{r,N-1}\cap \cA^{r} $ we have:
\begin{\Ieee}{LLL}
& \bar{\tilde{a}}^s_i=\tilde{a}^s_i, \, 0\leq s<r,\, 0\leq i\leq B-1 \Ieeen\label{eq:algo_stab_1}\\
& \bar{\tilde{a}}^r_0=\tilde{a}^r_0 \Ieeen\label{eq:algo_stab_2}
\end{\Ieee}
and for $s\geq r$ we have
\begin{equation}
\label{eq:crude_bnd_4}
\norm{\bar{\tilde{a}}^s_i-\tilde{a}^s_i}\leq \bar{C_2}\gamma RB\frac{\sqrt{R\beta}}{\zeta\lambda_{\min}}+8\gamma RB\norm{a_0-\a}\leq \bar{C}_2\gamma RB\frac{\sqrt{R\beta}}{\zeta\lambda_{\min}}
\end{equation}

\end{lemma}

We give the proof in Section~\ref{sec:technical_proofs}

\begin{remark}
\label{rem:crude_bnd_3}
In expression~\eqref{eq:crude_bnd_4}, we have suppressed the dependence of $\norm{a_0-\a}$ for the ease of exposition with the rationale being that since $a_0 = 0$, it sould be lower order compared to $\sqrt{R\beta} $. 
\end{remark}
Hence we see from the above lemma that changing a particular noise sample in a particular buffer perturbs the algorithmic iterates by $O(\gamma \mathrm{poly}(RB))$. \\

Let $\cR^r_{-j}$ denote the re-sampling operator corresponding to re-sampling $\eta^r_{-j}$. That is, for any function $f((a_\tau),(X_{\tau}),(\tilde X_{\tau}))$ we have 
\begin{equation}
\label{eq:resampling_operator}
\cR^r_{-j}\left(f((a_\tau),(\tilde a_\tau),(X_{\tau}),(\tilde X_{\tau}))\right)=f((\bar a_\tau),(\bar{\tilde a}_\tau),(\bar X_{\tau}),(\bar{\tilde X}_{\tau}))
\end{equation}

We will drop the subscripts and superscripts on $\cR$ when there is no ambiguity on which noise is re-sampled. First we will prove a lemma that bounds the effect of re-sampling.

\begin{lemma}
\label{lem:noise_resample_bound}
 On the event $\ce^r_{0,j}\cap\cdt^{0,t-1}\cap \cA^{t-1}$, for some constant $C$ depending only on $C_{\eta}$:
\begin{\Ieee}{LLL}
\label{eq:noise_resample_bound}
\norm{\Htt{t-r}{j+1}{B-1}\prodHtt{t}{r-1}-\cR^{t-r}_{-j}\Htt{t-r}{j+1}{B-1}\left(\prod_{s=r-1}^1 \cR^{t-r}_{-j}\Htt{t-s}{0}{B-1}\right)}\\
\leq \bar{C} \frac{Bt \norm{\phi''}\gamma^2 R^3 B\sqrt{\beta}}{\zeta\lambda_{\min}}\Ieeen
\end{\Ieee}
\end{lemma}
\begin{proof}
First, note that since we are re-sampling $\eta^{t-r}_{-j}$, the only difference between $\cR^{t-r}_{-j}\Htt{t-r}{j+1}{B-1}\left(\prod_{s=r-1}^1 \cR^{t-r}_{-j}\Htt{t-s}{0}{B-1}\right)$ and $\Htt{t-r}{j+1}{B-1}\prodHtt{t}{r-1}$ is that the algorithmic iterates $\tilde a^{s}_j$ that appear in the latter (through $\phi'(\cdot)$) are replaced by $\bar{\tilde a}^s_j$ in the former, but the covariates remain the same in both. 

Now, the matrix $\Htt{t-r}{j+1}{B-1}\prodHtt{t}{r-1}$ is of the form $\prod_{l=1}^{k}A_l$ where $\|A_l\| \leq 1$ under the conditioned events and is of the form  $I-2\gamma\phi'(\tilde \xi^{t-s}_{-i})\Xtt{t-s}{-j}\Xtttr{t-s}{-i}$. Similarly, we write: $\cR^{t-r}_{-j}\Htt{t-r}{j+1}{B-1}\left(\prod_{s=r-1}^1 \cR^{t-r}_{-j}\Htt{t-s}{0}{B-1}\right) = \prod_{l=1}^{k}\bar{A}_l$ where $\bar{A}_{l} = \cR^{t-r}_{-j}A_{l}$. Now consider the simple inequality under the condition that $\|A_l\|,\|\bar{A}_l\| \leq 1$

\begin{equation}\label{eq:component_decomposition}
\norm{\prod_{l=1}^{k}A_l - \prod_{l=1}^{k}\bar{A}_l} \leq \sum_{l=1}^{k}\norm{A_l - \bar{A}_l}
\end{equation}
Therefore, we will just bound each of the component differences $\norm{A_l - \bar{A}_l}$. To this end, consider a typical term $I-2\gamma\phi'(\tilde \xi^{t-s}_{-i})\Xtt{t-s}{-j}\Xtttr{t-s}{-i}$. We have
\begin{\Ieee}{LLL}
\label{eq:noise_resample_bound_1}
\left(I-2\gamma\phi'(\tilde \xi^{t-s}_{-i})\Xtt{t-s}{-i}\Xtttr{t-s}{-i}\right) -\cR^{t-r}_{-j}\left(I-2\gamma\phi'(\tilde \xi^{t-s}_{-i})\Xtt{t-s}{-i}\Xtttr{t-s}{-i}\right) \\
=2\gamma (\phi'(\tilde \xi^{t-s}_{-i})-\cR^{t-r}_{-j}\phi'(\tilde \xi^{t-s}_{-i}))\Xtt{t-s}{-i}\Xtttr{t-s}{-i}\Ieeen
\end{\Ieee}

Now 
\begin{\Ieee}{LLL}
\label{eq:noise_resample_bound_2}
\phi'(\tilde \xi^{t-s}_{-i})-\cR^{t-r}_{-j}\phi'(\tilde \xi^{t-s}_{-i})=\frac{\phi(\tilde a^{t-s,\top}_i \Xtt{t-s}{-i})-\phi(\atr \Xtt{t-s}{-i})}{(\tilde a^{t-s}_i-\a)^{\top}\Xtt{t-s}{-i}}-\frac{\phi(\bar{\tilde a}^{t-s,\top}_i \Xtt{t-s}{-i})-\phi(\atr \Xtt{t-s}{-i})}{(\bar{\tilde a}^{t-s}_i-\a)^{\top}\Xtt{t-s}{-i}}\\
\Ieeen
\end{\Ieee}

Now we can use the following simple result from calculus. Suppose $f$ is a real valued twice continuously differentiable function with bounded second derivative (denoted by $\norm{f''}$). Fix $x_0\in\mathbb{R}$. Let $g(x)=\frac{f(x)-f(x_0)}{x-x_0}$. By the mean value theorem, there exists $\xi$ such that:

$$g'(x)=\frac{f(x_0)-(f(x)+(x_0-x)f'(x)}{(x-x_0)^2}=\frac{1}{2}f''(\xi)$$

Now for any $x,y$, we have
$$|g(x)-g(y)|=|g'(\xi_1)(x-y)|\leq \frac{1}{2}\norm{f''}|x-y|$$
for some $\xi_1$ between $x$ and $y$. Here again we use the mean value theorem in the equality above.  Now we will apply this result to $\phi$ with $x=\tilde a^{t-s,\top}_i \Xtt{t-s}{-i}$, $y=\bar{\tilde a}^{t-s,\top}_i \Xtt{t-s}{-i}$ and $x_0=\atr \Xtt{t-s}{-i}$ to get
\begin{\Ieee}{LLL}
\label{eq:noise_resample_bound_3}
\left|\phi'(\tilde \xi^{t-s}_{-i})-\cR^{t-r}_{-j}\phi'(\tilde \xi^{t-s}_{-i})\right|\leq \frac{1}{2}\norm{\phi''}\norm{\tilde a^{t-s}_i-\bar{\tilde a}^{t-s}_i}\norm{\Xtt{t-s}{-i}}\Ieeen
\end{\Ieee}

Now we appeal to lemma~\ref{lem:algo-stability}. In particular, using equation~\eqref{eq:crude_bnd_4} we see that, on the event $\ce^r_{0,j}\cap\cdt^{0,t-1}\cap \cA^{t-1}$,  
\begin{align}
\left|\phi'(\tilde \xi^{t-s}_{-i})-\cR^{t-r}_{-j}\phi'(\tilde \xi^{t-s}_{-i})\right|&\leq \frac{1}{2}\norm{\phi''}128C\gamma R B\frac{\sqrt{R\beta}}{\zeta\lambda_{\min}}\sqrt{R}\nonumber\\
&= 64C\norm{\phi''}\gamma R^2 B\frac{\sqrt{\beta}}{\zeta\lambda_{\min}}\label{eq:noise_resample_bound_4} \\
\implies 
\norm{2\gamma (\phi'(\tilde \xi^{t-s}_{-i})-\cR^{t-r}_{-j}\phi'(\tilde \xi^{t-s}_{-i}))\Xtt{t-s}{-i}\Xtttr{t-s}{-i}}&\leq 128C\norm{\phi''}\gamma^2 R^3 B\frac{\sqrt{\beta}}{\zeta\lambda_{\min}} \label{eq:noise_resample_bound_5}
\end{align}

We now use Equation~\eqref{eq:component_decomposition} with $k \leq Bt$ along with Equation~\eqref{eq:noise_resample_bound_5} to conclude the statement of the lemma. 
%
%
\end{proof}

\subsection{Bound $\cro(t,r_1,r_2,j_1,j_2)$}

\label{subsec:cross_bound}

Next we will bound $\sum_{r}\sum_{j_1\neq j_2}\cro(t,r,r,j_1,j_2)$

\begin{claim}
\label{claim:cr_bnd_comm_buf}
\begin{\Ieee}{LLL}
\label{eq:cr_bnd_comm_buf}
\abs{\Ex{\sum_{r=1}^t\sum_{j_1\neq j_2}\cro(t,r,r,j_1,j_2)\indt{0}{t-1}}}&\leq & \bar{C}\left[ \frac{\sigma^2\gamma^2 RB}{T^{\alpha/2 - 1}}
+ \frac{\norm{\phi''} \gamma^4 T^2R^4 B^2\sigma^2\sqrt{\beta}}{\zeta\lambda_{\min}} \right]\Ieeen
\end{\Ieee}
Where $\bar{C}$ is a constant depending only on $C_{\eta}$
\end{claim}
\begin{proof}
Let $j_1<j_2$. 	We will suppress the arguments of $\cro$ for brevity. First, we re-sample the noise which is ahead in the time, i.e., $\eta^{t-r}_{-j_1}$ (and hence the entry $\Nt{t-r}{-j_1}$ in the row under consideration). 

Let $\cro'$ denote the resampled version of  $\cro$ as defined below 
\begin{\Ieee}{LLL}
\label{eq:cr_bnd_comm_buf_1}
\cro'(t,r,r,j_1,j_2)&:=& 4\gamma^2\Nt{t-r}{-j_1}\Nt{t-r}{-j_2} \cR^{t-r}_{-j_1}\left[\Xtttr{t-r}{-j_2}\Htt{t-r}{j_2+1}{B-1}\prodHtt{t}{r-1}\cdot\right.\\
&&\left. \prodHtttr{t}{r-1}\Htttr{t-r}{j_1+1}{B-1}\Xtt{t-r}{-j_1}\right]\\
&=& 4\gamma^2\Nt{t-r}{-j_1}\Nt{t-r}{-j_2} \Xtttr{t-r}{-j_2}\cR^{t-r}_{-j_1}\left(\Htt{t-r}{j_2+1}{B-1}\right)\cR^{t-r}_{-j_1}\prodHtt{t}{r-1}\cdot\\
&& \cR^{t-r}_{-j_1}\prodHtttr{t}{r-1} \cR^{t-r}_{-j_1}\left(\Htttr{t-r}{j_1+1}{B-1}\right) \Xtt{t-r}{-j_1}\Ieeen
\end{\Ieee}
where we have used the fact that $\cR^{t-r}_{-j_1}$ has no effect on the items from the process $(\tilde X_{\tau})_{\tau}$ \emph{that appear} in the expression above. Note the this is \textbf{not} $\cR^{t-r}_{-j_1} \cro $, since in $\cR^{t-r}_{-j_1}\cro$ we would have $\bar{\epsilon}^{t-r}_{-j_1}$ instead. Now, since the new algorithmic iterates $(\bar{\tilde a}^s_i)$ depend on $\bar\eta^{t-r}_{-j_1}$ but not on $\eta^{t-r}_{-j_1}$, it is immediate that 
$$\Ex{\cro'(t,r,r,j_1,j_2)}=0$$

For convenience, we introduce some notation which is only used in this proof.  $\cro'(t,r,r,j_1,j_2)$ can be written in the form $4\gamma^2\Nt{t-r}{-j_1} \Nt{t-r}{-j_2} K_1$ for some random variable $K_1$ independent of $\Nt{t-r}{-j_1}$. Under the event $\cdt^{0,t-1}$, we can easily show that $|K_1| \leq R$ almost surely. Let $\mathcal{F}_{K} = \sigma(K_1,\Nt{t-r}{-j_2})$. Let $\mathcal{M} := \{|K_1| \leq R\}$. Clearly, $\cdt^{0,t-1}\subseteq \mathcal{M}$ and $\Nt{t-r}{-j_1}\perp \mathcal{F}_{K}$. We conclude:
\begin{align}
\biggr|\Ex{\cro'\indt{0}{t-1}}\biggr| &= \biggr|\Ex{\cro'\indt{0}{t-1}1\left[\mathcal{M}\right]}\biggr| \nonumber \\
&= 4\gamma^2\biggr|\Ex{\Ex{\Nt{t-r}{-j_1}\indt{0}{t-1}\bigr|\mathcal{F}_K}K_1 \Nt{t-r}{-j_2}1\left[\mathcal{M}\right]}\biggr| \nonumber \\
&\leq 4\gamma^2\Ex{\bigr|\Ex{\Nt{t-r}{-j_1}\indt{0}{t-1}\bigr|\mathcal{F}_K}\bigr| \cdot |K_1| \cdot|\Nt{t-r}{-j_2}|1\left[\mathcal{M}\right]}  \label{eq:restriction_expectation}
\end{align}

We note that: $\biggr|\Ex{\Nt{t-r}{-j_1}\indt{0}{t-1}\bigr|\mathcal{F}_K}\biggr| = \biggr|\Ex{\Nt{t-r}{-j_1}\indt{0}{t-1,C}\bigr|\mathcal{F}_K}\biggr| \leq \sigma^2 \sqrt{\mathbb{P}\left(\indt{0}{t-1,C}|\mathcal{F}_K\right)}$. Using this in Equation~\eqref{eq:restriction_expectation}, and that under event $\mathcal{M}$, $|K_1| \leq R$ we apply Cauchy-Schwarz inequality again to conclude:
\begin{equation}\label{eq:cr_bnd_comm_buf_4}
\biggr|\Ex{\cro'\indt{0}{t-1}}\biggr| \leq 4\gamma^2\sigma^2 R \sqrt{\mathbb{P}(\cdt^{0,t-1,C})} \leq \frac{4\gamma^2\sigma^2 R}{T^{\alpha/2}}\end{equation}

Using similar technique as lemma~\ref{lem:noise_resample_bound}, we have that on the event $\ce^r_{0,j_1}\cap\cdt^{0,t-1}\cap \cA^{t-1}$, 
\begin{\Ieee}{LLL}
\label{eq:cr_bnd_comm_buf_5}
\norm{\Htt{t-r}{j_2+1}{B-1}\prodHtt{t}{r-1}-\cR^{t-r}_{-j_1}\Htt{t-r}{j_2+1}{B-1}\left(\prod_{s=r-1}^1 \cR^{t-r}_{-j_2}\Htt{t-s}{0}{B-1}\right)}\\
\leq \bar{C} \frac{T \norm{\phi''}\gamma^2 R^3 B\sqrt{\beta}}{\zeta\lambda_{\min}}\Ieeen
\end{\Ieee}

Therefore, on the event $\ce^r_{0,j_1}\cap\ce^r_{0,j_2}\cap\cdt^{0,t-1}\cap \cA^{t-1}$, we have
\begin{align}
&\abs{\cro-\cro'}\leq  \gamma^4\bar{C} R^4\left|\Nt{t-r}{-j_1}\Nt{t-r}{-j_2}\right|  T \norm{\phi''} B\frac{\sqrt{\beta}}{\zeta\lambda_{\min}} \nonumber\\
\implies
&\Ex{\abs{\cro-\cro'}1\left[\ce^r_{0,j_1}\cap\ce^r_{0,j_2}\cap\cdt^{0,t-1}\cap \cA^{t-1}\right]} 
\leq \bar{C}\norm{\phi''} \gamma^4 TR^4 B\frac{\sigma^2\sqrt{\beta}}{\zeta\lambda_{\min}} \label{eq:cross_term_bound_1}
\end{align}
We note that over the event $\cdt^{0,t-1}$, we must have $|\cro -\cro'|  \leq 2R|\Nt{t-r}{-j_1}\Nt{t-r}{-j_2}| $. Combining this with Equation~\eqref{eq:cross_term_bound_1} and noting that $\mathbb{P}\left(\ce^r_{0,j_1}\cap\ce^r_{0,j_2}\cap\cdt^{0,t-1}\cap \cA^{t-1}\right)\geq 1- \frac{1}{T^{\alpha}}$, we conclude:
\begin{\Ieee}{LLL}
\label{eq:cr_bnd_comm_buf_7}
&&\Ex{\abs{\sum_{r=1}^t\sum_{j_1\neq j_2}\cro(t,r,r,j_1,j_2)-\cro'(t,r,r,j_1,j_2) }\indt{0}{t-1}} \\
& \leq & \bar{C}\left[\norm{\phi''} \gamma^4 T^2R^4 B^2\frac{\sigma^2\sqrt{\beta}}{\zeta\lambda_{\min}} +\sigma^2\gamma^2 R TB\prbndsq\right] \Ieeen
\end{\Ieee}
Hence combining \eqref{eq:cr_bnd_comm_buf_4} and \eqref{eq:cr_bnd_comm_buf_7} we conclude the statement of the claim.
\end{proof}

Next we want to bound $\cro(t,r_1,r_2,j_1,j_2)$ for $r_2>r_1$ and arbitrary $j_1$ and $j_2$. Recall the definition of $\tilde{a}^{t-1,v}_B$ from \eqref{eq:SGD-RER_variance}. Via simple rearrangement of summation, we can express $\sum_{r_2>r_1}\sum_{j_1,j_2}\cro(t,r_1,r_2,j_1,j_2)$ in terms of $\tilde a^{t-r_1-1,v}_B$ as follows.

\begin{lemma}
\label{lem:cr_base_recursion}
\begin{\Ieee}{LLL}
\label{eq:cr_base_recursion}
\sum_{r_2>r_1}\sum_{j_1,j_2}\cro(t,r_1,r_2,j_1,j_2) \\
=2\gamma \sum_{r_1=1}^{t-1}\sum_{j_1=0}^{B-1}(\tilde{a}^{t-r_1-1,v}_B)^{\top}\prodHtt{t}{r_1}\prodHtttr{t}{r_1-1}\Htttr{t-r_1}{j_1+1}{B-1}\Xtt{t-r_1}{-j_1}\Nt{t-r_1}{-j_1}\Ieeen
\end{\Ieee}
\end{lemma}

\begin{claim}
\label{claim:cr_bnd_diff_buf}
 
\begin{\Ieee}{LLL}
\label{eq:cr_bnd_diff_buf}
\abs{\Ex{\sum_{r_1\neq r_2}\sum_{j_1,j_2}\cro(t,r_1,r_2,j_1,j_2)\indt{0}{t-1}}}\leq \bar{C}\gamma^2 R(Bt)^2\sigma^2 \prbndsq+\\
\bar{C}
\left(\norm{\phi''}\gamma^3 T^2 R^3 B\frac{\sqrt{\beta}}{\zeta\lambda_{\min}} +\gamma^2 TRB\right)\sqrt{R\sigma^2}\sqrt{\sup_{s\leq N-1}\Ex{\norm{\tilde{a}^{s,v}_B}^2\indt{0}{s}}}\\
\Ieeen
\end{\Ieee}
\end{claim}
The proof of the claim essentially proceeds similar to that of Claim~\ref{claim:cr_bnd_comm_buf} but with additional complications. We refer to Section~\ref{sec:technical_proofs} for the proof.

Combining everything in this section we have the following proposition.

\begin{proposition}
\label{prop:var_main_recur}  Let 
\begin{equation}
\label{eq:cal_v_tilde}
\cVt_{t-1}=\Ex{\norm{\tilde{a}^{t-1,v}_B}^2\indt{0}{t-1}}
\end{equation}

Then for some constant $\bar{C}$ which depends only on $C_{\eta}$:

\begin{\Ieee}{LLL}
\label{eq:var_main_recur}
\sup_{s\leq N-1}\cVt_{s} &\leq & \frac{2\gamma d}{\zeta(1-\gamma R)}\beta  + \bar{C}\norm{\phi''} \gamma^4 T^2R^4 B^2\frac{\sigma^2\sqrt{\beta}}{\zeta\lambda_{\min}} + \bar{C}\sigma^2\gamma^2 R T^2 \prbndsq  +\\
&& \bar{C}R\sigma^2\left(\norm{\phi''}^2\gamma^6 T^4 R^6 B^2\frac{\beta}{\zeta^2\lambda_{\min}^2} +\gamma^4 T^2 R^2 B^2\right)
\Ieeen
\end{\Ieee}

\end{proposition}

\begin{proof}
In the whole proof, we will denote any large enough constant depending on $C_{\eta}$ by $\bar{C}$. From claims~\ref{claim:dg_bound}, \ref{claim:cr_bnd_comm_buf} and \ref{claim:cr_bnd_diff_buf} along with equation~\eqref{eq:var_last_iter_3} we have 
\begin{\Ieee}{LLL}
\label{eq:var_main_recur_1}
\Ex{\norm{\tilde{a}^{t-1,v}_B}^2\indt{0}{t-1}} \\
\leq  \frac{\gamma d}{\zeta(1-\gamma R)}\beta  + \bar{C}\norm{\phi''} \gamma^4 T^2R^4  B^2\frac{\sigma^2\sqrt{\beta}}{\zeta\lambda_{\min}} + \bar{C}\sigma^2\gamma^2 R T\prbndsq (1+B+T) +\\
\left(\bar{C}\norm{\phi''}\gamma^3 T^2 R^3 B\frac{\sqrt{\beta}}{\zeta\lambda_{\min}} +\gamma^2 TRB\right)\sqrt{R\sigma^2}\sqrt{\sup_{s\leq N-1}\Ex{\norm{\tilde a^{s,v}_B}^2\indt{0}{s}}}\\
\Ieeen
\end{\Ieee}

Thus
\begin{\Ieee}{LLL}
\label{eq:var_main_recur_2}
\sup_{s\leq N-1}\cVt_{s}&\leq &  \frac{\gamma d}{\zeta(1-\gamma R)}\beta  + \bar{C}\norm{\phi''} \gamma^4 T^2R^4 B^2\frac{\sigma^2\sqrt{\beta}}{\zeta\lambda_{\min}} + \bar{C}\sigma^2\gamma^2 R T^2 \prbndsq  +\\
&&\bar{C}\left(\norm{\phi''}\gamma^3 T^2 R^3 B\frac{\sqrt{\beta}}{\zeta\lambda_{\min}} +\gamma^2 TRB\right)\sqrt{R\sigma^2}\sqrt{\sup_{s\leq N-1}\cVt_s}\\
\Ieeen
\end{\Ieee}

Finally, we need to solve the above recursive relation. We note a simple fact: Let $c_1,c_2>0$ be constants and let $x>0$ satisfy
\begin{equation}
x^2\leq c_1+c_2 x
\end{equation}
then 
\begin{equation}
x^2\leq \frac{1}{4}\left(c_2+\sqrt{c_2^2+4c_1}\right)^2\leq c_2^2+2c_1
\end{equation}
where in the last inequality above we used the fact that $(a+b)^2\leq 2(a^2+b^2)$.

Thus,
\begin{\Ieee}{LLL}
\label{eq:var_main_recur_4}
\sup_{s\leq N-1}\cVt_{s}&\leq &  
 \frac{2\gamma d}{\zeta(1-\gamma R)}\beta  + \bar{C}\norm{\phi''} \gamma^4 T^2R^4 B^2\frac{\sigma^2\sqrt{\beta}}{\zeta\lambda_{\min}} + \bar{C}\sigma^2\gamma^2 R T^2 \prbndsq  +\\
&& \bar{C}R\sigma^2\left(\norm{\phi''}^2\gamma^6 T^4 R^6 B^2\frac{\beta}{\zeta^2\lambda_{\min}^2} +\gamma^4 T^2 R^2 B^2\right)
\Ieeen
\end{\Ieee}

\end{proof}

%
%
%

\subsection{Bias of last iterate}
In this part, we will bound the expectation of the bias term $\norm{\tilde{a}^{t-1,b}_B-\a}^2$. 

\begin{theorem}
\label{thm:bias_last_iter_main}

For some universal constant $c_0$:

\begin{\Ieee}{LLL}
\label{eq:bias_last_iter_main}
\Ex{\norm{\tilde{a}^{t-1,b}_B-\a}^2\indt{0}{t-1}}&\leq &\norm{a_0-\a}^2 (1-c_0\zeta\gamma B\lambda_{\min})^{ t}\Ieeen
\end{\Ieee}
where $\tilde{a}^{t-1,b}-\a$ is defined in~\eqref{eq:SGD-RER_bias}
\end{theorem}

\begin{proof}
 Define $x_v = \prod_{s=0}^{v-1} \Htt{s,\top}{0}{B-1} (a_0-a^{*})$. Here, we consider the event $\mathcal{G}_v$ considered in Claim~\ref{claim:contraction_imporbability} in the proof of Theorem~\ref{thm:buffer_norm_upper_bound}, and show that for some universal constant $q_0 >0$, 
\begin{equation}\label{eq:non_triv_contraction}
 \mathbb{P}(\|\Htt{v,\top}{0}{B-1} x_v\|^2 \geq (1-\zeta \gamma \lmin{G})\|x_v\|^2\bigr|\cdt^{0,t-1},x_v) \geq q_0
\end{equation}

From Theorem~\ref{thm:buffer_norm_upper_bound} we also note that conditioned on $\cdt^{0,t-1}$, almost surely: 
$$\|\Htt{v,\top}{0}{B-1} x_v\|^2 \leq 1 $$ 
 
We let  $\mathcal{G}_v$ be the event lower bounded in Equation~\eqref{eq:non_triv_contraction}. 
 
\begin{align}
\mathbb{E}\left[\| x_{v+1}\|^2\bigr| \cdt^{0,t-1}\right]&= 
\mathbb{E}\left[\|\Htt{v,\top}{0}{B-1} x_v\|^2\bigr| \cdt^{0,t-1}\right] \nonumber \\ &=
\mathbb{E}\left[\|\Htt{v,\top}{0}{B-1} x_v\|^2 1\left[ \mathcal{G}_v\right]+\|\Htt{v,\top}{0}{B-1} x_v\|^21\left[ \mathcal{G}^C_v\right]\bigr| \cdt^{0,t-1}\right]  \nonumber \\
&\leq \mathbb{E}\left[(1-\gamma\zeta\lmin{G})\|x_v\|^2 1\left[ \mathcal{G}_v\right] + \|x_v\|^21\left[ \mathcal{G}^C_v\right]\bigr| \cdt^{0,t-1}\right]  \nonumber \\
&=\mathbb{E}\left[\|x_v\|^2\left[1-\gamma\zeta\lmin{G}\mathbb{P}(\mathcal{G}_v\bigr|\cdt^{0,t-1}\, x_v)\right]\bigr| \cdt^{0,t-1}\right]  \nonumber \\
&\leq \mathbb{E}\left[\|x_v\|^2\left[1-\gamma\zeta\lmin{G}q_0\right]\bigr| \cdt^{0,t-1}\right]
\label{eq:bias_recursion}
\end{align}

Unrolling the recursion given by Equation~\eqref{eq:bias_recursion}, and noting that $\tilde{a}^{t-1,b}_B-\a = x_{t}$, we conclude
$$\Ex{\norm{\tilde{a}^{t-1,b}_B-\a}^2\indt{0}{t-1}} \leq (1-c_0\gamma B \lambda_{\min}\zeta)^{t} \,.$$

Hence we have the theorem.

\end{proof}

\subsection{Average iterate: bias-variance decomposition}
\label{subsec:avg_decomp}
In the part, we will consider the tail-averaged iterate where a generic row is given by 
\begin{\Ieee}{LLL}
\label{eq:avg_iterate_row_1}
\ahtto=\frac{1}{N-t_0}\sum_{t=t_0+1}^{N}\tilde{a}^{t-1}_B\Ieeen
\end{\Ieee}
where $t_0\in\{0,1,\cdots,N-1\}$.

Thus we can write $\ahto-\a$ as
\begin{\Ieee}{LLL}
\label{eq:avg_iter_bias_var_decomp}
\ahtto-\a=(\ahttv)+(\ahttb-\a)\Ieeen
\end{\Ieee}
where
\begin{\Ieee}{LLL}
\ahttv=\frac{1}{N-t_0}\sum_{t=t_0+1}^{N}(\tilde{a}^{t-1,v}_B)\Ieeen\label{eq:avg_iter_var_def}\\
\ahttb-\a=\frac{1}{N-t_0}\sum_{t=t_0+1}^{N}(\tilde{a}^{t-1,b}_B-\a)\Ieeen\label{eq:avg_iter_bias_def}\\
\end{\Ieee}

\subsection{Variance of average iterate}
\label{subsec:avg_var}

\begin{remark}
From now on we will use the following notation: 
$$\sum_{t}\equiv \sum_{t=t_0+1}^N$$
$$\sum_{t_1,t_2}\equiv \sum_{t_1,t_2=t_0+1}^N$$
$$\sum_{t_1\neq t_2}\equiv \sum_{\substack{t_1,t_2=t_0+1\\ t_1\neq t_2}}^N$$
$$\sum_{t_2>t_1}\equiv \sum_{t_1=t_0+1}^{N-1}\sum_{t_2=t_1+1}^{N} $$
\end{remark}

Next we expand $\norm{\ahttv}^2$ 
\begin{\Ieee}{LLL}
\label{eq:avg_iter_var_expand}
\norm{\ahttv}^2&=&\frac{1}{(N-t_0)^2}\sum_t \norm{\tilde{a}^{t-1,v}_B}^2+\\
&&\frac{1}{(N-t_0)^2}\sum_{t_1\neq t_2}(\tilde{a}^{t_2-1,v}_B)^{\top}(\tilde{a}^{t_1-1,v}_B)\Ieeen
\end{\Ieee}

\begin{claim}
\label{claim:avg_iter_var_aux1}
 For $t_2>t_1$
\begin{\Ieee}{LLL}
\label{eq:avg_iter_var_aux1}
\abs{\Ex{\left[(\tilde{a}^{t_2-1,v}_B)^{\top}\left((\tilde{a}^{t_1-1,v}_B)-\left(\prod_{s=t_2-t_1}^1 \Htt{t_2-s}{0}{B-1}\right)(\tilde{a}^{t_1-1,v}_B)\right)\right]\indt{0}{N-1}}}\\
\leq C_1\poly(R,B,\beta,1/\zeta,1/\lambda_{\min},\norm{\phi''})\left(\gamma^{7/2}T^2+\gamma^5 T^3+\gamma^6 T^4\right)\Ieeen
\end{\Ieee}
\end{claim}
\begin{proof}
From~\eqref{eq:SGD-RER_variance} we can write
\begin{\Ieee}{LLL}
\label{eq:avg_iter_var_aux1_1}
(\tilde{a}^{t_2-1,v}_B)^{\top}&=&(\tilde{a}^{t_1-1,v}_B)^{\top}\left(\prod_{s=t_2-t_1}^1 \Htt{t_2-s}{0}{B-1}\right)+\\
&&2\gamma\sum_{r=1}^{t_2-t_1}\sum_{j=0}^{B-1}\Nt{t_2-r}{-j}\Xtttr{t_2-r}{-j}\Htt{t_2-r}{j+1}{B-1}\prodHtt{t_2}{r-1}\Ieeen
\end{\Ieee}

Hence 
\begin{\Ieee}{LLL}
\label{eq:avg_iter_var_aux1_2}
(\tilde{a}^{t_2-1,v}_B)^{\top}(\tilde{a}^{t_1-1,v}_B)=(\tilde{a}^{t_1-1,v}_B)^{\top}\left(\prod_{s=t_2-t_1}^1 \Htt{t_2-s}{0}{B-1}\right)(\tilde{a}^{t_1-1,v}_B)+\\
2\gamma \sum_{r=1}^{t_2-t_1}\sum_{j=0}^{B-1}\Nt{t_2-r}{-j}\Xtttr{t_2-r}{-j}\Htt{t_2-r}{j+1}{B-1}\prodHtt{t_2}{r-1}(\tilde{a}^{t_1-1,v}_B)\Ieeen
\end{\Ieee}

Now recall the noise re-sampling operator $\cR^{t_2-r}_{-j}$ from~\eqref{eq:resampling_operator}. It is easy to see that 
\begin{\Ieee}{LLL}
\label{eq:avg_iter_var_aux1_3}
\Ex{2\gamma \sum_{r=1}^{t_2-t_1}\sum_{j=0}^{B-1}\Nt{t_2-r}{-j}\cR^{t_2-r}_{-j}\left[\Xtttr{t_2-r}{-j}\Htt{t_2-r}{j+1}{B-1}\prodHtt{t_2}{r-1}(\tilde{a}^{t_1-1,v}_B)\right]}=0\\
\Ieeen
\end{\Ieee}
(Note that $\cR^{t_2-r}_{-j}\Xtttr{t_2-r}{-j}=\Xtttr{t_2-r}{-j}$ )

Thus, using the decomposition 
$$\cdt^{0,N-1}=\cdt^{0,t_2-r-1}\cap \cdt^{t_2-r+1,N-1}\cap \cdt^{t_2-r}_{-j}\cap \cap_{i=0}^{j-1}\cct^{t_2-r}_{-i}$$
we get
\begin{\Ieee}{LLL}
\label{eq:avg_iter_var_aux1_4}
&&\abs{\Ex{2\gamma \sum_{r=1}^{t_2-t_1}\sum_{j=0}^{B-1}\Nt{t_2-r}{-j}\Xtttr{t_2-r}{-j}\cdot \right.\right.\\
&& \left.\left.\cR^{t_2-r}_{-j}\left[\Htt{t_2-r}{j+1}{B-1}\prodHtt{t_2}{r-1}(\tilde{a}^{t_1-1,v}_B)\right]\indt{0}{N-1}}}\\
&\leq & 4\gamma^2 R (Bt_1)(B(t_2-t_1))C_{\eta}\sigma^2\prbndsq \\
&\leq & 4\gamma^2 R C_{\eta}\sigma^2 T^2\prbndsq\Ieeen
\end{\Ieee}

Next we need to bound the effect due to noise re-sampling. On the event $\cdt^{0,N-1}\cap\cA^{N-1}\cap\cap_{r=0}^{N-1}\ce^r_{0,B-1}$, we have
\begin{\Ieee}{LLL}
\label{eq:avg_iter_var_aux1_5}
\norm{\Htt{t_2-r}{j+1}{B-1}\prodHtt{t_2}{r-1}-\cR^{t_2-r}_{-j}\left[\Htt{t_2-r}{j+1}{B-1}\prodHtt{t_2}{r-1}\right]}\\
\leq C\norm{\phi''}\gamma^2 TR^3B\frac{\sqrt{\beta}}{\zeta\lambda_{\min}}\Ieeen
\end{\Ieee} 

Thus 
\begin{\Ieee}{LLL}
\label{eq:avg_iter_var_aux1_6}
&&2\gamma \abs{\Ex{ \sum_{r=1}^{t_2-t_1}\sum_{j=0}^{B-1}\Nt{t_2-r}{-j}\Xtttr{t_2-r}{-j}\indt{0}{N-1}\cdot \right.\right.\\
&& \left.\left.\left(\Htt{t_2-r}{j+1}{B-1}\prodHtt{t_2}{r-1}-\cR^{t_2-r}_{-j}\left[\Htt{t_2-r}{j+1}{B-1}\prodHtt{t_2}{r-1}\right]\right)(\tilde{a}^{t_1-1,v}_B)}}\\
&\leq & \bar{C}\gamma \left(\norm{\phi''}\gamma^2T R^3B\frac{\sqrt{\beta}}{\zeta\lambda_{\min}}\right)B(t_2-t_1)\sqrt{R\sigma^2}\Ex{\norm{(\tilde{a}^{t_1-1,v}_B)}\indt{0}{t_1-1}} \\
&& + \bar{C}\gamma^2 (2R) \sigma^2 (Bt_1)(B(t_2-t_1))\prbndsq\Ieeen
\end{\Ieee}

Now from proposition~\ref{prop:var_main_recur}, there is a constant $C_1$ such that 

\begin{\Ieee}{LLL}
\label{eq:avg_iter_var_aux1_7}
\left(\Ex{\norm{\tilde{a}^{t_1-1,v}}\indt{0}{t_1-1}}\right)^2\leq \\
C_1\left(\frac{\gamma d}{\zeta(1-\gamma R)}\beta  + \norm{\phi''} \gamma^4 T^2 R^4 B^2\frac{\sigma^2\sqrt{\beta}}{\zeta\lambda_{\min}} + \sigma^2\gamma^2 R T^2 \prbndsq  +\right.\\
\left.  R\sigma^2\left( \norm{\phi''}^2\gamma^6 T^4 R^6 B^2\frac{\beta}{\zeta^2\lambda_{\min}^2} +\gamma^4 T^2 R^2 B^2\right)\right)\Ieeen
\end{\Ieee}

So
\begin{\Ieee}{LLL}
\label{eq:avg_iter_var_aux1_8}
&&2\gamma \abs{\Ex{ \sum_{r=1}^{t_2-t_1}\sum_{j=0}^{B-1}\Nt{t_2-r}{-j}\Xtttr{t_2-r}{-j}\indt{0}{N-1}\cdot \right.\right.\\
&& \left.\left.\left(\Htt{t_2-r}{j+1}{B-1}\prodHtt{t_2}{r-1}-\cR^{t_2-r}_{-j}\left[\Htt{t_2-r}{j+1}{B-1}\prodHtt{t_2}{r-1}\right]\right)(\tilde{a}^{t_1-1,v}_B)}}\\
&\leq & \poly(R,B,\beta,1/\zeta,1/\lambda_{\min},\norm{\phi''})\left(\gamma^{7/2}T^2+\gamma^5 T^3+\gamma^6 T^4\right)\Ieeen
\end{\Ieee}

where we absorbed terms involving $\prbnd$ since $\alpha$ is taken to be large.

\end{proof}

\begin{claim}
\label{claim:avg_iter_geom}

\begin{\Ieee}{LLL}
\label{eq:avg_iter_geom}
\abs{\Ex{(\tilde{a}^{t_1-1,v}_B)^{\top}\sum_{t_2>t_1}\left(\prod_{s=t_2-t_1}^1 \Htt{t_2-s}{0}{B-1}\right)(\tilde{a}^{t_1-1,v}_B) \indt{0}{N-1}}}\\
\leq \cV_{t_1-1}\frac{C}{\zeta\gamma B\lambda_{\min}}+16(N-t_1)\gamma^2 RC_{\eta}\sigma^2 T^2\prbndsq \Ieeen
\end{\Ieee}
where $\cV_{t_1-1}$ is defined in \eqref{eq:cal_v_tilde}.
\end{claim}
\begin{proof}
Note that 
$$\prod_{s=t_2-t_1}^1 \Htt{t_2-s}{0}{B-1}=\prod_{s=t_1}^{t_2-1} \Htt{s}{0}{B-1}$$

From theorem~\ref{thm:prod_buff_norm}, there is a universal constant $C$  such that with $\delta=\prbndH$ we have
\begin{\Ieee}{LLL}
\label{eq:avg_iter_geom_1}
\Pb{\norm{\sum_{t_2>t_1}\prod_{s=t_1}^{t_2-1} \Htt{s}{0}{B-1}}>C\left(d+\log \frac{N}{\delta}+\frac{1}{\zeta\gamma B\lambda_{\min}}\right)|\cdt^{t_1,N-1}}\leq \prbndH\Ieeen
\end{\Ieee}

Since $\Pb{\cdt^{t_1,N-1}}\leq\prbndH$ we obtain
\begin{\Ieee}{LLL}
\label{eq:avg_iter_geom_2}
\Pb{\norm{\sum_{t_2>t_1}\prod_{s=t_1}^{t_2-1} \Htt{s}{0}{B-1}}>C\left(d+\log \frac{N}{\delta}+\frac{1}{\zeta\gamma B\lambda_{\min}}\right)}\leq \prbnd\Ieeen
\end{\Ieee}

Choosing $\gamma$ such that
$$d+\log\frac{N}{\delta}\leq \frac{1}{\zeta\gamma B\lambda_{\min}}$$ 
we get
\begin{\Ieee}{LLL}
\label{eq:avg_iter_geom_3}
\Pb{\norm{\sum_{t_2>t_1}\prod_{s=t_1}^{t_2-1} \Htt{s}{0}{B-1}}>\frac{C}{\zeta\gamma B\lambda_{\min}}}\leq \prbnd\Ieeen
\end{\Ieee}

Thus conditioning on the event 
$$\norm{\sum_{t_2>t_1}\prod_{s=t_1}^{t_2-1} \Htt{s}{0}{B-1}} \leq \frac{C}{\zeta\gamma B\lambda_{\min}}$$
we obtain
\begin{\Ieee}{LLL}
\label{eq:avg_iter_geom_4}
\abs{\Ex{(\tilde{a}^{t_1-1,v}_B)^{\top}\sum_{t_2>t_1}\left(\prod_{s=t_2-t_1}^1 \Htt{t_2-s}{0}{B-1}\right)(\tilde{a}^{t_1-1,v}_B) \indt{0}{N-1}}}\\
\leq \Ex{\norm{\tilde{a}^{t_1-1,v}_B}^2\indt{0}{t-1}}\frac{C}{\zeta\gamma B\lambda_{\min}}+(N-t_1)4\gamma^2 R(4C_{\eta}\sigma^2)(Bt_1)^2\prbndsq\\
\leq \cV_{t_1-1}\frac{C}{\zeta\gamma B\lambda_{\min}}+16(N-t_1)\gamma^2 RC_{\eta}\sigma^2 T^2\prbndsq
\Ieeen
\end{\Ieee}

\end{proof}

Thus combining everything we have the following theorem
\begin{theorem}
\label{thm:avg_iter_var_main}
 Suppose $\gamma \gtrsim \frac{1}{T}$. Then
\begin{\Ieee}{LLL}
\label{eq:avg_iter_var_main_a}
\Ex{\norm{\ahttv}^2\indt{0}{N-1}} \leq C_1\frac{d\beta}{\zeta^2\lambda_{\min} B(N-t_0)}+\bar{P}\cdot
\left(\gamma^{7/2}T^2+\gamma^6 T^4\right)\Ieeen
\end{\Ieee}
$\bar{P} = \poly(R,B,\beta,1/\zeta,1/\lambda_{\min},\norm{\phi''},C_{\eta})$ and $C_1 > 0$ is some constant.

\end{theorem}

\subsection{Bias of average iterate}
\label{subsec:avg_bias}
\begin{theorem}
\label{thm:avg_iter_bias_main}
There are constants $C,c_1,c_2$ such that :
\begin{align}
\Ex{\norm{\ahttb-\a}^2\indt{0}{N-1}}& \leq & C\norm{a_0-a}^2\left[e^{-c_2\zeta\gamma B\lambda_{\min}t_0}\min\left\{1,\frac{1}{(N-t_0)\zeta\gamma B\lambda_{\min}}\right\}\right]
\end{align}
\end{theorem}
The proof follows from an application of Theorem~\ref{thm:bias_last_iter_main}.

\subsection{Proof of Theorem~\ref{thm:sgd_rer_ub}}
\label{sec:sgd_rer_proof}
\begin{proof}
Let $\bar{P}$ denote the polynomial in Theorem~\ref{thm:avg_iter_var_main}.
Theorems~\ref{thm:avg_iter_bias_main} and ~\ref{thm:avg_iter_var_main}, imply for every row of the coupled iterate:
\begin{align}
\label{eq:sgd_rer_proof_1}
\mathbb{E}\left[\|\ahtto-\a\|^2\indt{0}{N-1}\right] &\leq 2\mathbb{E}\left[\|\ahttv\|^2\indt{0}{N-1}\right]+2\mathbb{E}\left[\|\ahttb-\a\|^2\indt{0}{N-1}\right] \nonumber\\
&\leq C_2 \frac{d\beta}{\zeta^2\lmin{G} B(N-t_0)} + \bar{P}\cdot \left(\gamma^{7/2}T^2+\gamma^6 T^4\right) \nonumber \\&+  C\norm{a_0-a}^2\left[e^{-c_2\zeta\gamma B\lambda_{\min}t_0}\min\left\{1,\frac{1}{(N-t_0)\zeta\gamma B\lambda_{\min}}\right\}\right]
\end{align} 

Thus for the actual process we can use the following decomposition
\begin{\Ieee}{LLL}
\label{eq:sgd_rer_proof_2}
\Ex{\norm{\ahto-\a}^2\ind{0}{N-1}}&\leq & \Ex{\norm{\ahto-\a}^2\indh{0}{N-1}}+\\
&&\Ex{\norm{\ahto-\a}^2\ind{0}{N-1}\indtc{0}{N-1}}\\
&\leq & \Ex{\norm{\ahto-\a}^2\indh{0}{N-1}} +C\gamma^2 T^2 R C_{\eta}\sigma^2\prbndsq\\
&& +2\norm{a_0-\a}^2\prbnd
\end{\Ieee}
where we used the fact on the event $\cd^{0,N-1}$
\begin{\Ieee}{LLL}
\label{eq:sgd_rer_proof_3}
\norm{\ahto-\a}^2 \leq \frac{1}{N-t_0}\sum_{t=t_0+1}^{N}\left(2\norm{a_0-\a}^2+2(Bt)(4\gamma^2 R \sum_{r=1}^t\sum_{j=0}^{B-1}|\Nt{t-r}{-j}|^2)\right)\Ieeen
\end{\Ieee}
and then used Cauchy-Schwarz inequality for the expectation over $\cdt^{0,N-1,C}$. 

Now using lemma~\ref{lem:coupled_iterate_replacement} we get
\begin{\Ieee}{LLL}
\label{eq:sgd_rer_proof_4}
\Ex{\norm{\ahto-\a}^2\indh{0}{N-1}}\leq \Ex{\norm{\ahtto-\a}^2\indh{0}{N-1}}+C\gamma^2 R^2 T^2 \prbnd\Ieeen
\end{\Ieee}
since we are choosing $u$ such that $\rho^u\leq \prbnd$. Using $\cdh^{0,N-1}\subset \cdt{0,N-1}$ we get
\begin{\Ieee}{LLL}
\label{eq:sgd_rer_proof_5}
\Ex{\norm{\ahto-\a}^2\ind{0}{N-1}}\leq \Ex{\norm{\ahtto-\a}^2\indt{0}{N-1}} +C\gamma^2 R \sigma^2 \frac{1}{T^{\alpha/2-2}}\Ieeen
\end{\Ieee}
where we absorbed all terms of order $\prbnd$ (including those depending on $\norm{a_0-\a}$) to the last term in the above display.

Thus
\begin{\Ieee}{LLL}
\label{eq:sgd_rer_proof_6}
\Ex{\|\ahtto-\a\|^2\indt{0}{N-1}} &\leq & C_1 \frac{d\beta}{\zeta^2\lmin{G} B(N-t_0)} + \bar{P}\cdot \left(\gamma^{7/2}T^2+\gamma^6 T^4\right) \nonumber \\
&&+  C_2\norm{a_0-a}^2\left[\frac{e^{-c_2\zeta\gamma B\lambda_{\min}t_0}}{(N-t_0)\zeta\gamma B\lambda_{\min}}\right]\\
&&+C_3\gamma^2 R \sigma^2 \frac{1}{T^{\alpha/2-2}}\Ieeen
\end{\Ieee}

Summing over all the rows we get a bound on the Frobenius norm. Lastly if the event $\cd^{0,N-1}$ does not occur, the $\hat A_{t_0,N}$ is the zero matrix and hence 
\begin{\Ieee}{LLL}
\label{eq:sgd_rer_proof_7}
\Ex{\|\hat{A}_{t_0,N}-\A\|_{\mathsf{F}}^2\indc{0}{N-1}}\leq \norm{\A}^2\prbnd\Ieeen. 
\end{\Ieee}

Therefore:

\begin{\Ieee}{LLL}
\label{eq:sgd_rer_proof_8}
\Ex{\|\hat{A}_{t_0,N}-\A\|_{\mathsf{F}}^2}&\leq & \bar{C} \frac{d^2\beta}{\zeta^2\lmin{G} B(N-t_0)} + \bar{P}\cdot d\left(\gamma^{7/2}T^2+\gamma^6 T^4\right)\\
&& + \bar{C}\norm{A_0-\A}^2_{\mathsf{F}}\left[\frac{e^{-c_2\zeta\gamma B\lambda_{\min}t_0}}{(N-t_0)\zeta\gamma B\lambda_{\min}}\right] \\
&&+ \bar{C}\gamma^2 R \sigma^2 d \frac{1}{T^{\alpha/2-2}}\Ieeen
\end{\Ieee}

\end{proof}

\section{Concentration Under Stationary Measure}

\label{sec:concentration}
In this section, we will consider the process $\glvar(\A,\mu,\phi)$ and the concentration of measure under its stationary distribution. In what follows, we will use the fact that $\phi$ is $1$-Lipschitz as in the definition of $\glvar$, even when we don't explicitly use Assumption~\ref{assump:3}.
\begin{proposition}\label{prop:stationarity_existence}
Under Assumption~\ref{assump:4} with $\rho < 1$, the process is exponentially Ergodic and has a stationary distribution $\pi$. Suppose $X \sim \pi$ then $\mathbb{E}\|X\|^4 < \infty$.
\end{proposition}
 The result follows from a technique similar to the one used for Proposition \cite{dieuleveut2020bridging} by considering the process in the space of measures endowed with the Wasserstein metric. 

\subsection{Sub-Gaussian Case: Stable Systems}
We will first consider the process with $X_0 = 0$ and prove concentration for $X_t$ for arbitrary $t$, and then use distributional convergence results to prove the concentration results at stationarity. First, we prove some preparatory lemmas.

\begin{lemma}\label{lem:tail_integration}
Suppose $Y$ is a $\nu^2$ sub-Gaussian random variable with zero mean. Then, for any $\lambda \leq \frac{1}{4\nu^2}$, we have:
$$\mathbb{E}\exp(\lambda Y^2) \leq 1 + 8\lambda\nu^2\,.$$
\end{lemma}
We refer Section~\ref{sec:technical_proofs} for the proof. Now, consider the random variable $Z_{t+1} = \|X_{t+1}\|^{2} - \sum_{s=0}^{t}\rho^{t-s}\|\eta_s\|^2$. By assumption, we have $X_0 = 0$. Therefore we must have $Z_{0} = 0$. 
We have the following lemma:
\begin{lemma}\label{lem:super_martingale_type_bound}
Suppose that Assumptions~\ref{assump:7} and ~\ref{assump:6} hold and $\rho$ be as given in Assumption~\ref{assump:6}. For any $\lambda$ such that $0 \leq \lambda \leq \frac{1-\rho}{2\rho C_{\eta}\sigma^2}$, we have:
$$\mathbb{E}\exp(\lambda Z_{t+1}) \leq 1\,.$$
\end{lemma}
\begin{proof}
First by mean value theorem, we must have:
 $\phi(\A X_t) = \phi(\A X_t) - \phi(0) = D \A X_t $ for some diagonal matrix $D$ with entries lying in $[0,1]$. Therefore, $\|\phi(\A X_t) \| \leq \|D\|\|\A\|\|X_t\| \leq \rho \|X_t\|$. Using this in Equation~\eqref{eq:master_equation}, we conclude:
\begin{align}
\|X_{t+1}\|^2 -\|\eta_t\|^2 &= \|\phi(\A X_t)\|^2 + 2\langle \eta_t, \phi(\A X_t) \rangle \nonumber \\
&\leq \rho^{2}\|X_t\|^2 +  2\langle \eta_t, \phi(\A X_t)\rangle \label{eq:norm_contraction_pointwise}
\end{align}

Let $\mathcal{F}_s = \sigma(X_0,\eta_0,\dots,\eta_s)$. It is clear that $X_{s} \in \mathcal{F}_{s-1}$. Using Equation~\eqref{eq:norm_contraction_pointwise}, we conclude:

\begin{align}
\mathbb{E}\left[\exp(\lambda Z_{t+1})\bigr|\mathcal{F}_{t-1}\right] &= \mathbb{E}\left[\exp\left(\lambda \|X_{t+1}\|^2 - \lambda \|\eta_t\|^2\right)\bigr|\mathcal{F}_{t-1}\right]\exp\left(-\lambda \sum_{s=0}^{t-1}\rho^{t-s}\|\eta_s\|^2\right) \nonumber \\
&\leq \mathbb{E}\left[\exp\left(\lambda \rho^2\|X_{t}\|^2 + 2\lambda\langle \eta_t, \phi(\A X_t)\rangle\right)\bigr|\mathcal{F}_{t-1}\right]\exp\left(-\lambda \sum_{s=0}^{t-1}\rho^{t-s}\|\eta_s\|^2\right) \nonumber \\
&\leq \exp\left(\lambda \rho^2\|X_{t}\|^2 + 2\lambda^2C_{\eta}\sigma^2\|\phi(\A X_t)\|^2\right)\exp\left(-\lambda \sum_{s=0}^{t-1}\rho^{t-s}\|\eta_s\|^2\right) \nonumber \\
&\leq \exp\left(\lambda \rho^2\|X_{t}\|^2 + 2\lambda^2\rho^2C_{\eta}\sigma^2\|X_t\|^2\right)\exp\left(-\lambda \sum_{s=0}^{t-1}\rho^{t-s}\|\eta_s\|^2\right) \nonumber \\
&\leq  \exp\left(\lambda \rho\|X_{t}\|^2 \right)\exp\left(-\lambda \sum_{s=0}^{t-1}\rho^{t-s}\|\eta_s\|^2\right) \nonumber \\
&= \exp\left(\lambda\rho Z_{t}\right) \label{eq:super_martingale_type}
\end{align}
In the fourth step, we have used the fact that $\|\phi(\A X_t)\| \leq \rho
\|X_t\|$. In the fifth step we have used the assumption that $\lambda \leq \frac{1-\rho}{2\rho C_{\eta}\sigma^2}$ to show $\lambda \rho^2 + 2\lambda^2\rho^2C_{\eta}\sigma^2 \leq \rho \lambda$. In the last step, we have used the definition of $Z_t$.  We iterate over Equation~\eqref{eq:super_martingale_type} and use the fact that $Z_0 = 0$ almost surely to conclude that whenever $\lambda  \leq \frac{1-\rho}{2\rho C_{\eta}\sigma^2}$, we must have:

$$\mathbb{E}\exp(\lambda Z_{t+1}) \leq \mathbb{E}\exp(\lambda Z_0) = 1\,.$$

\end{proof}

Now, let $Y_{t+1} = \sum_{s=0}^{t}\rho^{t-s}\|\eta_t\|^2$. We will now use Lemma~\ref{lem:tail_integration} to bound $\mathbb{E}\exp(\lambda Y_{t+1})$ for $\lambda > 0$ small enough.

\begin{lemma}\label{lem:geometric_noise_concentration}
Suppose that Assumptions~\ref{assump:7} and ~\ref{assump:6} hold and $\rho$ be as given in Assumption~\ref{assump:6}.
For any $\lambda$ such that $ 0 \leq \lambda \leq \frac{1}{4dC_{\eta}\sigma^2}$, we have:
$$\mathbb{E}\exp(\lambda Y_{t+1}) \leq \exp\left(8\tfrac{\lambda d C_{\eta}\sigma^2}{1-\rho}\right) $$
\end{lemma}

\begin{proof}
Let $N(\beta) := \mathbb{E}\exp(\beta \|\eta_s\|^2)$
By independence of the noise sequence, we have: 
\begin{equation}\label{eq:noise_independence_exp}
\mathbb{E}\exp(\lambda Y_{t+1}) = \prod_{s = 0}^{t}N(\rho^{t-s}\lambda)
\end{equation}
For $\beta \leq \frac{1}{4dC_{\eta}\sigma^2}$
\begin{align}
N(\beta) &= \mathbb{E}\exp(\beta \|\eta_s\|^2) = \mathbb{E}\exp(\beta\sum_{i=1}^{d} \langle e_i,\eta_s\rangle^2) \nonumber\\
&\leq \frac{1}{d} \sum_{i=1}^{d} \mathbb{E}\exp(\beta d \langle e_i,\eta_s\rangle^2) \leq 1 + 8\beta d C_{\eta} \sigma^2
\end{align}
In the last step, we have used Jensen's inequality for the function $x \to \exp(x)$ and then invoked the Lemma~\ref{lem:tail_integration}. Plugging this into Equation~\eqref{eq:noise_independence_exp}, we conclude:

\begin{align}
\mathbb{E}\exp(\lambda Y_{t+1}) &\leq \prod_{s = 0}^{t}\left(  1+ 8\lambda d \rho^{t-s} d C_{\eta} \sigma^2\right) \leq \exp(\sum_{s=0}^{t} 8\lambda d\rho^{t-s} C_{\eta}\sigma^2) \nonumber \\
&\leq \exp\left(8\tfrac{\lambda d C_{\eta}\sigma^2}{1-\rho}\right) 
\end{align}

\end{proof}

Based on Lemmas~\ref{lem:geometric_noise_concentration} and~\ref{lem:super_martingale_type_bound}, we will now state the following concentration inequality:

\begin{theorem}\label{thm:process_concentration}
Suppose Assumptions~\ref{assump:7} and ~\ref{assump:6} hold and $\rho$ be as given in Assumption~\ref{assump:6}. Let $X$ be distributed according $\pi$, the stationary distribution of $\glvar(\A,\mu,\phi)$. Then, for any $0<\lambda  \leq \lambda^{*}:= \min(\tfrac{1}{8dC_{\eta} \sigma^2}, \tfrac{1-\rho}{4\rho C_{\eta}\sigma^2})$, we have:
$$\mathbb{E}\exp\left(\lambda \|X\|^2\right) \leq \exp(\tfrac{8\lambda d C_{\eta} \sigma^2}{1-\rho}) \,.$$
We conclude:
\begin{enumerate}
\item Applying Chernoff bound with $\lambda = \lambda^{*}$, we conclude:
$$\mathbb{P}\left(\|X\|^2 > \frac{8dC_{\eta}\sigma^2}{1-\rho} + \beta\right) \leq \exp(-\lambda^{*}\beta) \,.$$
\item $$\mathbb{E}\|X\|^2 \leq \frac{8 d C_{\eta}\sigma^2}{1-\rho}$$
\end{enumerate}
The conclusions still hold when $X$ is replaced by $X_t$ for any $t \in \mathbb{N}$ for the process started at $0$. 
\end{theorem}
\begin{proof}
We first note that $\|X_{t+1}\|^2 = Z_{t+1} + Y_{t+1}$. Therefore, by Cauchy-Schwarz inequality, we must have:
\begin{align}
\mathbb{E}\exp\left(\lambda \|X_{t+1}\|^2\right) &= \mathbb{E}\exp\left(\lambda (Z_{t+1}+Y_{t+1})\right) \nonumber \\
&\leq \sqrt{\mathbb{E}\exp(2\lambda Z_{t+1})}\sqrt{\mathbb{E}\exp(2\lambda Y_{t+1})} \nonumber \\
&\leq \exp(\tfrac{8\lambda d C_{\eta} \sigma^2}{1-\rho})
\end{align}
Here we have used Lemmas~\ref{lem:geometric_noise_concentration} and~\ref{lem:super_martingale_type_bound} and the appropriate bounds on $\lambda$. Recall that we started the chain $(X_t)$ with $X_0 = 0$. Denote the law of $X_t$ by $\pi_t$. By proposition~\ref{prop:stationarity_existence}, we show that $\pi_t$ converges weakly to the stationary distribution $\pi$. We invoke Skhorokhod representation theorem to show that there exist random variables $\bar{X}_t \sim \pi_t$ and $X \sim \pi$ for $t\in \mathbb{N}$, defined on a common probability space such that $\bar{X}_t \to X$ almost surely. Now, we have shown that:

$$\mathbb{E}\exp\left(\lambda \|\bar{X}_{t+1}\|^2\right) \leq \exp(\tfrac{8\lambda d C_{\eta} \sigma^2}{1-\rho}) \,.$$

Now, applying Fatou's Lemma to the equation above as $t \to \infty$, we conclude:
\begin{equation}\label{eq:exp_moment_bound}
\mathbb{E}\exp\left(\lambda \|X\|^2\right) \leq \exp(\tfrac{8\lambda d C_{\eta} \sigma^2}{1-\rho}) \,.
\end{equation}

The concentration inequality follows from an application of Chernoff bound and the second moment bound follows from Jensen's inequality to Equation~\eqref{eq:exp_moment_bound} (i.e, $\mathbb{E}\exp(Y) \geq \exp(\mathbb{E}Y)$).
\end{proof}
%
%
%
%
%
%
%

\subsection{Sub-Gaussian Case: Possibly Unstable Systems}
We consider the case with $(C_{\rho},\rho)$ regularity, but we allow $\rho > 1$. 
\begin{lemma}\label{lem:unstable_concentration}
Under Assumption~\ref{assump:4}, we have:

\begin{equation}\label{eq:almost_sure_bound_process}
\|X_{t}\| \leq C_{\rho}\sum_{s=0}^{t-1}\rho^{t-s-1}\|\eta_s\|\,.
\end{equation}
No suppose Assumption~\ref{assump:7} also holds. Let $\delta \in (0,1/2)$. Then with probability atleast $1-\delta$, we must have:
$$\sup_{0\leq t \leq T}\|X_{t}\| \leq C C_{\rho}\sqrt{C_{\eta}}S(\rho,T)\sigma \sqrt{d\log(\tfrac{T}{\delta})}\,.$$
Where $S(\rho,T) := \sum_{t=0}^{T-1}\rho^{T-t-1}$ and $C$ is some universal constant. 
\end{lemma}
\begin{proof}
We consider the notations established in Assumption~\ref{assump:4}. We will define the process $X_{t}^{(s)}$ by $X_0^{(s)} = \dots = X_s^{(s)} = 0$ and $X_{t+1}^{(s)} = \phi(\A X_{t}^{(s)}) +\eta_t$ for $t \geq s$, where $\eta_t$ is the same noise sequence driving the process $X_0,X_1,\dots,X_T$. Note that $X_t^{(s)} = h_{t-s}(0,\eta_s,\eta_{s+1},\dots,\eta_{t-1})$.
\begin{align}
X_t - 0 &= X_t - X_t^{(1)} + X_{t}^{(1)} - X_t^{(2)} + \dots + X_t^{(t)} - 0 \nonumber \\
\implies
\|X_t\| &\leq  \sum_{s=0}^{t-1} \|X_t^{(s)} -X_t^{(s+1)}\| \nonumber \\
&= \sum_{s=0}^{t-1}\|h_{t-s-1}(\eta_s,\dots,\eta_{t-1})-h_{t-s-1}(0,\eta_{s+1},\dots,\eta_{t-1})\| \nonumber \\
&\leq \sum_{s=0}^{t-1}C_{\rho}\rho^{t-s-1}\|\eta_s\|
\end{align}
In the last step, we have used Assumption~\ref{assump:4}. To prove the high probability bound, we note that $\mathbb{P}(\sup_{0\leq s \leq T-1}\leq \|\eta_s\| > C \sqrt{C_{\eta}}\sigma \log(\tfrac{T}{\delta})  ) \leq \delta$ for some universal constant $C$.
\end{proof}
\subsection{Heavy Tailed Case: Stable Systems}

\begin{theorem}\label{thm:heavy_tail_concentration}
Suppose Assumption~\ref{assump:4} holds with $\rho < 1$. Suppose that $X$ is distributed as the stationary distribution $\pi$ of the system $\glvar(\A,\mu, \phi)$. Then, we have:
\begin{enumerate}
\item $$\mathbb{E}\|X\|^4 \leq \frac{C_{\rho}^4M_4}{(1-\rho)^4}\,.$$
Where we recall $M_4 = \mathbb{E}\|\eta_t\|^4$.
\item $$\mathbb{E}\|X\|^2 \leq \frac{C_{\rho}^2d\sigma^2}{(1-\rho)^2}\,.$$
\end{enumerate}

\end{theorem}
\begin{proof}
We use Equation~\eqref{eq:almost_sure_bound_process} to conclude the desired bound for $X_T$ when the process is started with $X_0 =0$. We then use Fatou's lemma along with Skhorokhod Representation theorem like in Theorem~\ref{thm:process_concentration} to conclude the bound at stationarity. 
\end{proof}

\section{Well Conditioned Second Moment Matrices}
\label{sec:grammian_well_conditioned}

In this section we will consider a stationary sequence $X_0,\dots, X_T$ derived from the process $\glvar(\A,\mu,\phi)$, with the corresponding noise sequence $\eta_0,\dots,\eta_T$. We want to show that the matrix $\frac{1}{B}\sum_{t=0}^{B-1}X_tX_t^{\intercal}$ behaves similar to $G := \mathbb{E}X_tX_t^{\intercal}$. To do this, we will first to control the quantity:
$\mathbb{E}\langle X_t,x \rangle^2\langle X_s,x\rangle^2 $ for arbitrary fixed vector $x \in \mathbb{R}^d$. Clearly, $\mathbb{E}\langle X_t,x\rangle^2 = x^{\intercal}Gx$. 

\begin{lemma}\label{lem:square_correlation}
Without loss of generality, we suppose that $t > s$. Suppose $X_0,\dots,X_T$ be a stationary sequence from $\glvar(\A,\mu,\phi)$.
\begin{enumerate}
\item 
Suppose Assumptions~\ref{assump:7} and~\ref{assump:6} hold. Then we have:
$$\mathbb{E}\langle X_t,x \rangle^2\langle X_s,x\rangle^2  \leq 2 (x^{\intercal}Gx)^2 + \bar{C}_1 \rho^{2(t-s)}  \frac{d\sigma^2}{1-\rho}x^{\intercal}Gx \log\left(\frac{d}{1-\rho}\right)$$
Where $\bar{C}_1$ depends only on $C_{\eta}$.

\item Suppose Assumption~\ref{assump:4} holds with $\rho < 1$. Recall that $M_4 = \mathbb{E}\|\eta_t\|^2$. We have:
$$\mathbb{E}\langle X_t,x \rangle^2\langle X_s,x\rangle^2  \leq 2 (x^{\intercal}Gx)^2 + 8\|x\|^2 C_{\rho}^6 \rho^{2(t-s)}\frac{M_4}{(1-\rho)^4}$$
\end{enumerate}
\end{lemma}

We will give the proof of this lemma in Section~\ref{sec:technical_proofs}. Now, consider the random matrix $\hat{G}_{B} := \frac{1}{B}\sum_{t=0}^{B-1}X_tX_t^{\intercal}$. Clearly, $\hat{G}_{B} \succeq 0$ and because of stationarity, $\mathbb{E}\hat{G}_B = G$. We write down the following lemma:
\begin{lemma}\label{lem:probable_contraction}
Suppose $X_0,\dots,X_T$ be a stationary sequence from $\glvar(\A,\mu,\phi)$.
\begin{enumerate}
\item Suppose Assumptions~\ref{assump:7} and~\ref{assump:6} hold. 
 Let $\bar{C}_1$ be as in Lemma~\ref{lem:square_correlation}. Suppose $B \geq \bar{C}_1 \frac{d}{(1-\rho)(1-\rho^2)}\log\left(\frac{d}{1-\rho}\right) $. Then, for any fixed vector $x \in \mathbb{R}^{d}$,
$$\mathbb{P}\left( x^{\intercal}\hat{G}_B x \geq \frac{1}{2}x^{\intercal}Gx\right) \geq p_0 >0\,.$$
Where $p_0$ is a universal constant which can be taken to be $\tfrac{1}{16}$.
Furthermore, for any event $\mathcal{A}$ such that $\mathbb{P}(\mathcal{A}) > 1-p_0 $, we must have:

$$\mathbb{P}\left( x^{\intercal}\hat{G}_B x \geq \frac{1}{2}x^{\intercal}Gx\biggr| \mathcal{A}\right) \geq q_0 := \frac{p_0 -\mathbb{P}(\mathcal{A}^c)}{\mathbb{P}(\mathcal{A}) } >0\,.$$

\item Suppose Assumption~\ref{assump:4} holds with $\rho < 1$. Whenever $B \geq \frac{8C_{\rho}^6M_4}{(1-\rho)^4(1-\rho^2)\sigma^4}$
$$\mathbb{P}\left( x^{\intercal}\hat{G}_B x \geq \frac{1}{2}x^{\intercal}Gx\right) \geq p_0 >0\,.$$
\end{enumerate}

\end{lemma}
\begin{proof}
Without loss of generality, take $\|x\|  = 1$. We start with the Paley-Zygmund inequality. Let $Z$ be any random variable such that $Z \geq 0$ almost surely and $\mathbb{E}Z^2 <\infty$. For any $\theta \in [0,1]$ we must have:
$$\mathbb{P}(Z \geq \theta \mathbb{E}Z ) \geq (1-\theta)^2 \frac{(\mathbb{E}Z)^2}{\mathbb{E}Z^2} \,.$$
Now consider $Z =  x^{\intercal}\hat{G}_B x $ and $\theta = \frac{1}{2}$. 
\begin{enumerate}
\item 

The simple calculation shows that:

\begin{align}
\mathbb{P}\left( x^{\intercal}\hat{G}_B x \geq \frac{1}{2}x^{\intercal}Gx\right)  &\geq \frac{1}{4} \frac{B^2(x^{\intercal} G x)^2}{\sum_{s,t=0}^{B-1}\mathbb{E}\langle X_{t},x\rangle^2\langle X_s,x\rangle^2} \nonumber \\
&\geq \frac{1}{4} \frac{B^2(x^{\intercal} G x)^2}{2B^2(x^{\intercal} G x)^2 + \sum_{s,t=0}^{B-1}\bar{C}_1 \rho^{2|t-s|}  \frac{d\sigma^2}{1-\rho}x^{\intercal}Gx \log\left(\frac{d}{1-\rho}\right)} \nonumber \\
&\geq \frac{1}{4} \frac{B^2(x^{\intercal} G x)^2}{2B^2(x^{\intercal} G x)^2 + 2\sum_{t=0}^{B-1}\bar{C}_1  \frac{d\sigma^2}{(1-\rho)(1-\rho^2)}x^{\intercal}Gx \log\left(\frac{d}{1-\rho}\right)} \nonumber \\
&=\frac{1}{4} \frac{B^2(x^{\intercal} G x)^2}{2B^2(x^{\intercal} G x)^2 + 2B\bar{C}_1  \frac{d\sigma^2}{(1-\rho)(1-\rho^2)}x^{\intercal}Gx \log\left(\frac{d}{1-\rho}\right)} \nonumber\\
&= \frac{1}{8} \frac{1}{1+ \tau_B}
\end{align}

Here, $\tau_B := \frac{\bar{C}_1}{x^{\intercal}Gx} \frac{d\sigma^2}{B(1-\rho)(1-\rho^2)}\log\left(\frac{d}{1-\rho}\right) $. In the second step we have used item 1 of Lemma~\ref{lem:square_correlation}. In the third step, we have summed the infinite series $\sum_{s\geq t}\rho^{2(t-s)}$. Using the hypothesis that $B \geq \bar{C}_1 \frac{d}{(1-\rho)(1-\rho^2)}\log\left(\frac{d}{1-\rho}\right) $ and $G \succeq \sigma^2 I$, we conclude the result.
\item We proceed similarly as above, but use item 2 in Lemma~\ref{lem:square_correlation} instead. 
\end{enumerate}
\end{proof}

We will now follow the method used to prove~\cite[Lemma 31]{jain2021streaming}. We now consider the matrix $\Htt{s}{0}{B-1}$  under the event $\cdt^s_{-0}$ in order to prove Theorem~\ref{thm:buffer_norm_upper_bound}, where the terms are as defined in Section~\ref{subsec:notations_events}. For the sake of clarity, we will drop the superscript $s$.
\begin{remark}
We prove the results below for $\Htt{s}{0}{B-1}$ but they hold unchanged when the matrices are all replaced with $\Htt{s,\top}{0}{B-1}$ given that we reverse the order of taking products whenever they are encountered.
\end{remark}

\begin{lemma}\label{lem:contraction}
Suppose Assumption~\ref{assump:3} holds. Suppose that $\gamma RB < \frac{1}{4}$. Then, for any buffer $s$,  under the event $\cdt^{s}_{-0}$, we have: 

$$I - 4\gamma\left(1+\tfrac{2\gamma BR}{1-4\gamma B R}\right)\sum_{i=0}^{B-1}\Xtt{s}{-i}\Xtttr{s}{-i} \preceq \Htt{s}{0}{B-1}\Htttr{s}{0}{B-1}\preceq I - 4\gamma\left(\zeta-\tfrac{2\gamma BR}{1-4\gamma B R}\right)\sum_{i=0}^{B-1}\Xtt{s}{-i}\Xtttr{s}{-i}$$

In particular, whenever we have $\gamma B R \leq \frac{\zeta}{4(1+\zeta)}$, we must have:
$$I - 4\gamma\left(1+\frac{\zeta}{2}\right)\sum_{i=0}^{B-1}\Xtt{s}{-i}\Xtttr{s}{-i} \preceq \Htt{s}{0}{B-1}\Htttr{s}{0}{B-1}\preceq I - 2\gamma\zeta\sum_{i=0}^{B-1}\Xtt{s}{-i}\Xtttr{s}{-i}$$
\end{lemma}
\begin{proof}
The proof follows from the proof of~\cite[Lemma 28]{jain2021streaming} with minor modifications to account for the fact that $\phi^{\prime}(\beta) \in [\zeta,1]$. 
\end{proof}

Combining Lemma~\ref{lem:contraction} with Lemma~\ref{lem:square_correlation} we will show that $\Htt{s}{0}{B-1}$ contracts any given vector with probability at-least $p_0 > 0$. 

\begin{theorem}\label{thm:buffer_norm_upper_bound}
Suppose Assumptions~\ref{assump:3},~\ref{assump:7} and~\ref{assump:6} hold.  Assume that $B$ and $\gamma$ are such that:
$B \geq \bar{C}_1 \frac{d}{(1-\rho)(1-\rho^2)}\log\left(\frac{d}{1-\rho}\right) $ and $\gamma B R \leq \frac{\zeta}{4(1+\zeta)}$ where $\bar{C}_1$ is as given in Lemma~\ref{lem:probable_contraction}. We also assume that $\mathbb{P}(\cdh^{b,a}) > \max(\tfrac{1}{2},1-\frac{p_0}{2})$, where $p_0$ is as given in Lemma~\ref{lem:probable_contraction}. Let $a \geq b$. 
 Let $\lmin{G}$ denote the smallest eigenvalue of $G$. Conditioned on the event $\cdt^{b,a}$,
\begin{enumerate}[label=(\arabic*)]
\item $\|\prod_{s=a}^{b}\Htttr{s}{0}{B-1}\| \leq 1$ almost surely
\item  Whenever $b-a+1$ is larger than some universal constant $C_0$,
$$\mathbb{P}\left(\|\prod_{s=a}^{b}\Htttr{s}{0}{B-1}\| \geq 2(1-\zeta\gamma B\lmin{G})^{c_4(a-b+1)}\biggr|\cdt^{b,a}\right)
\leq \exp(-c_3 (a-b+1)+c_5d)$$
Where $c_3,c_4$ and $c_5$ are universal constants.
\end{enumerate}
\end{theorem}

\begin{proof}
The proof of (1) above follows from an application of Lemma~\ref{lem:contraction}. So we will just prove (2). We will prove this with an $\epsilon$ net argument over the unit $\ell^2$ sphere in $\mathbb{R}^d$. 

Suppose we have arbitrary $x\in \mathbb{R}^d$ such that $\|x\| = 1$.  
Let $K_{v} := \prod_{s=v}^{b}\Htttr{s}{0}{B-1}$. When $v \leq b$, we take this product to be identity. Now, define $\hat{G}_B^{v} := \frac{1}{B}\sum_{j=0}^{B-1}X^{v}_{j}X_{j}^{v,\intercal}$

Consider the class of events indexed by $v$: $\mathcal{G}_v := \{\|\Htttr{v}{0}{B-1}K_{v-1}x\|^2 \leq \|K_{v-1}x\|^2 (1-\gamma\zeta B\lmin{G} \}$. From Lemma~\ref{lem:probable_contraction}, we will prove the following claim:
\begin{claim}\label{claim:contraction_imporbability}
 Whenever $v \in [b,a]\cap \mathbb{Z}$:
\begin{equation}\label{eq:contraction_improbability}
\mathbb{P}(\mathcal{G}^{c}_v|\tilde{\mathcal{D}}^{b,a},\Htttr{s}{0}{B-1}: s < v) \leq 1-q_0
\end{equation}
Where $q_0 > 0$ is as given in Lemma~\ref{lem:probable_contraction} and can be taken to be a universal constant under the present hypotheses.
\end{claim}
\begin{proof}
We will denote $K_{v-1}x$ by $x_v$ for the sake of convenience. We note that when conditioned on $\Htttr{s}{0}{B-1}$ for $s<v$, $x_v$ is fixed. Using Lemma~\ref{lem:contraction}, we note that: $$ \mathbb{P}(\mathcal{G}^{c}_v|\tilde{\mathcal{D}}^{b,a},\Htttr{s}{0}{B-1}: s < v) \leq \mathbb{P}(x_v^{\intercal}\hat{G}_B^{v}x_v < \tfrac{1}{2}x_v^{\intercal}Gx_v|\tilde{\mathcal{D}}^{b,a},\Htttr{s}{0}{B-1}: s < v) $$ 
We note that $\hat{G}^{v}_B$ is independent of $\Htttr{s}{0}{B-1}$ for $s\leq v$ (eventhough $\Htt{v}{0}{B-1}$ is not necessarily). Now we also note that $\hat{G}^{v}_B$ is independent of $\tilde{\mathcal{D}}^{s}$ for $s \neq v$. 
Therefore, we can apply Lemma~\ref{lem:probable_contraction} to conclude the claim. 
\end{proof}

Let $D \subseteq \{b,\dots,a\}$ such that $|D| = r$. It is also clear from item 1 and the definitions above that whenever the event $\cap_{v \in D}\mathcal{G}_v$ holds, we have:
\begin{equation}\label{eq:norm_contraction}
\|\prod_{s=a}^{b}\Htttr{s}{0}{B-1}x\| \leq (1-\gamma B\lmin{G})^\frac{r}{2}\,.
\end{equation}
Therefore, whenever Equation~\eqref{eq:norm_contraction} is violated, we must have a set $D^{c} \subseteq \{b,\dots,a\}$ such that $|D^{c}| \geq b-a-r$ and the event $\cap_{v \in D^{c}} \mathcal{G}^{c}_v$ holds. We will union bound all such events indexed by $D^{c}$ to obtain an upper bound on the probability that Equation~\eqref{eq:norm_contraction} is violated. Therefore, using Equation~\eqref{eq:contraction_improbability} along with the union bound, we have:

$$\mathbb{P}\left(\|\prod_{s=a}^{b}\Htttr{s}{0}{B-1}x\| \geq (1-\gamma B\lmin{G})^\frac{r}{2}\biggr|\tilde{\mathcal{D}}^{b,a}\right) \leq {{a-b+1}\choose{a-b-r}}(1-q_0)^{a-b-r}$$
Whenever $a-b+1$ is larger than some universal constant, we can pick $r = c_2(b-a+1)$ for some constant $c_2 > 0$ small enough such that:
$$\mathbb{P}\left(\|\prod_{s=a}^{b}\Htttr{s}{0}{B-1}x\| \geq (1-\gamma B\lmin{G})^\frac{r}{2}\biggr|\tilde{\mathcal{D}}^{b,a}\right) \leq \exp(-c_3(b-a+1))$$

Now, let $\mathcal{N}$ be a $1/2$-net of the sphere $\mathcal{S}^{d-1}$. Using Corollary 4.2.13 in \cite{vershynin2019high}, we can choose $|\mathcal{N}| \leq 6^d$. By Lemma 4.4.1 in \cite{vershynin2019high} we show that:

\begin{equation}\label{eq:constant_vector_contraction}
\|\prod_{s=a}^{b}\Htttr{s}{0}{B-1}\| \leq 2 \sup_{x\in \mathcal{N}}\|\prod_{s=a}^{b}\Htttr{s}{0}{B-1}x\|
\end{equation}

By union bounding Equation~\eqref{eq:constant_vector_contraction} for every $x \in \mathcal{N}$, we conclude that:

\begin{align}
&\mathbb{P}\left(\|\prod_{s=a}^{b}\Htttr{s}{0}{B-1}\| \geq 2(1-\zeta\gamma B\lmin{G})^{c_4(b-a+1)}\biggr|\tilde{\mathcal{D}}^{b,a}\right)\leq |\mathcal{N}|\exp(-c_3(a-b+1)) \nonumber \\
&= \exp(-c_3 (a-b+1)+c_5d)
\end{align}

\end{proof}

We will now state the equivalent of~\cite[Lemma 32]{jain2021streaming}. The proof proceeds similarly, but using Theorem~\ref{thm:buffer_norm_upper_bound} instead. Consider the following operator:

\begin{equation}\label{eq:geometric_stochastic_op}
F{a,N}:= \sum_{t=a}^{N-1} \prod_{s=t}^{a+1}\Htttr{s}{0}{B-1}
\end{equation}
Here we choose the convention that whenever $s > t$, then in any product involving $\Htttr{s}{0}{B-1}$ and $\Htt{t}{0}{B-1}$, $s$ appears to the right of $t$. Hence, we use the take $\prod_{s = a}^{a+1}\Htttr{s}{0}{B-1} = I$

\begin{theorem}
\label{thm:prod_buff_norm}
Suppose all the conditions in Theorem~\ref{thm:buffer_norm_upper_bound} hold. Then, for any $\delta \in (0,1)$, we have:
$$\mathbb{P}\left(\|F_{a,N}\| \geq C\left(d + \log \frac{N}{\delta} + \frac{1}{\zeta\gamma B \lmin{G}}\right)\biggr|\tilde{\mathcal{D}}^{a,N}\right) \leq \delta $$
Where $C$ is a universal constant. 
\end{theorem} 

\section{Self Normalized Noise Concentration}
\label{sec:self_normalized_noise}

We recall the events defined in Section~\ref{sec:gen_newton_analysis} :
\begin{enumerate}
\item $\mathcal{D}_{T}(R) := \{\sup_{0\leq t\leq T}\|X_t\|^2 \leq R\}$ 
\item $\mathcal{E}_T(\kappa) := \{\hat{G} \succeq \frac{\sigma^2I}{\kappa}\}$
\item $\mathcal{D}_T(R,\kappa):= \mathcal{D}_T(R)\cap\mathcal{E}_T(\kappa)$
\end{enumerate}
From Lemma~\ref{lem:unstable_concentration}, we conclude that taking $R \geq C^2_{\rho}C_{\eta}(S(\rho,T))^2\sigma^2\log (\tfrac{2T}{\delta})$ ensures that $\mathbb{P}\left(\mathcal{D}_{T}(R)\right) \geq 1-\frac{\delta}{2}$. Only in this section, we define the following:
\begin{enumerate}
\item
$\bar{X}_{t} := \phi(\A X_{t})$
\item 
$\bar{K}_X := \frac{1}{T}\sum_{t=0}^{T-2} \bar{X}_t\eta^{\intercal}_t + \eta_t\bar{X}_t^{\intercal}$
\item 
$\bar{G}:=\frac{1}{T}\sum_{t=0}^{T-2} \bar{X}_t\bar{X}_t^{\intercal}$
\item 
$\bar{K}_{\eta} :=\frac{1}{T}\sum_{t=0}^{T-2}\eta_t\eta_t^{\intercal}$
\end{enumerate}
\begin{lemma}\label{lem:dominant_event}
Let $\delta \in (0,\frac{1}{2})$
Take $R = C^2_{\rho}C_{\eta}(S(\rho,T))^2d\sigma^2\log (\tfrac{2T}{\delta})$, $\kappa = 2$ and suppose $T \geq \bar{C}_3 \left(d\log\left(\frac{R}{\sigma^2}\right) + \log\tfrac{1}{\delta}\right) $ for some constant $\bar{C}_3$ depending only on $C_{\eta}$. Then, we have:
$$\mathbb{P}\left(\mathcal{D}_{T}(R,\kappa)\right) \geq 1-\delta $$
\end{lemma}

\begin{proof}
  Consider $\hat{G} = \frac{1}{T}\sum_{t=0}^{T-1} X_tX_t^{\intercal} = \frac{1}{T}\sum_{t=0}^{T-2}\bar{X}_t\bar{X}_t^{\intercal} + \bar{X}_t\eta^{\intercal}_t + \eta_t\bar{X}_t^{\intercal}+\eta_t\eta_t^{\intercal}$. For this proof only, we will define, 
To show the result, we will prove that $\bar{K}_X$ is not too negative with high probability and that $\bar{K}_{\eta}$ dominates identity with high probability. Let $x \in \mathcal{S}^{d-1}$ and $\lambda \in\mathbb{R}$
Note that due to the sub-Gaussianity of $\eta_t$ and the definition of the process, 
 $$M_s := \exp\left(\sum_{s=0}^{t}\lambda \langle x,\eta_s\rangle\langle  x, \bar{X}_s\rangle - \frac{C_{\eta}\sigma^2\lambda^{2}}{2}\langle \bar{X}_s,x\rangle^2\right) \,.$$
is a super martingale with respect to the filtration $\mathcal{F}_t := \sigma(X_0,\eta_0,\dots,\eta_t)$, we conclude that $\mathbb{E}M_{T-1} \leq 1$. An application of Chernoff bound shows that for every $\lambda , \beta > 0$, we must have:

\begin{equation}\label{eq:self_norm_conc_1}
\mathbb{P}\left(|\langle x,\bar{K}_X x\rangle| \geq 2C_{\eta}\sigma^2\lambda x^{\intercal}\bar{G}x + \frac{\beta}{T}\biggr|\mathcal{D}_T(R)\right) \leq \frac{2}{1-\delta}\exp(-\lambda\beta)
\end{equation}

We will now invoke Theorem 5.39 in \cite{vershynin2010introduction} to conclude that for some constant $\bar{C}_2$ which depends only on $C_{\eta}$:
\begin{equation}\label{eq:noise_lower_isometry}
\mathbb{P}\left( \bar{K}_{\eta} \preceq \left(1-\bar{C}_2\left(\sqrt{\tfrac{d}{T}}+ \sqrt{\tfrac{\log\tfrac{1}{\delta}}{T}}\right)\right)\sigma^2I  \biggr|\mathcal{D}_T(R)\right) \leq \frac{\delta}{4}
\end{equation}



Consider any $\epsilon$ net $\mathcal{N}_{\epsilon}$ over $\mathcal{S}^{d-1}$. By Corollary 4.2.13 in \cite{vershynin2010introduction}, we can take $|\mathcal{N}_{\epsilon}| \leq (1+\tfrac{2}{\epsilon})^d$.  From Equations~\eqref{eq:self_norm_conc_1} and~\eqref{eq:noise_lower_isometry}, we conclude that conditioned on $\mathcal{D}_T(R)$, with probability at-least $1-\frac{\delta}{4} - |\mathcal{N}_{\epsilon}|\frac{\exp(-\lambda\beta)}{1-\delta}$ we have:

\begin{align}
\inf_{x\in \mathcal{S}^{d-1}}x^{\intercal}\hat{G}x  &\geq \inf_{y \in \mathcal{N}_{\epsilon}} y^{\intercal}\hat{G}y - 2\|\hat{G}\|\epsilon \nonumber\\
&\geq \inf_{y \in \mathcal{N}_{\epsilon}} y^{\intercal}\bar{G}y - |y^{\intercal}\bar{K}_{X}y|+y^{\intercal}\bar{K}_{\eta}y - 2\|\hat{G}\|\epsilon \nonumber \\
&\geq  \inf_{y \in \mathcal{N}_{\epsilon}} y^{\intercal}\bar{G}y - |y^{\intercal}\bar{K}_{X}y|+y^{\intercal}\bar{K}_{\eta}y - 2R\epsilon \nonumber\\
&\geq \inf_{y \in \mathcal{N}_{\epsilon}} y^{\intercal}\bar{G}y(1-2\lambda \sigma^2C_{\eta}) -\frac{\beta}{T} + \sigma^2\left(1-\bar{C}_2\left(\sqrt{\tfrac{d}{T}}+ \sqrt{\tfrac{\log\tfrac{1}{\delta}}{T}}\right) \right) - 2R\epsilon
\end{align}
In the third step, we have used the fact that under the event  $\mathcal{D}_T(R)$, $\|\hat{G}\| \leq R$.  Take $\lambda = \frac{1}{2\sigma^2C_{\eta}}$ and $\epsilon = \frac{1}{8R\sigma^2}$ and $\beta = 2\sigma^2 dC_{\eta}\log(16\tfrac{R}{\sigma^2}+1) + 2\sigma^2 C_{\eta}\log{\tfrac{8}{\delta}}$. We conclude that whenever $T \geq \bar{C}_3 \left(d\log\left(\frac{R}{\sigma^2}\right) + \log\tfrac{1}{\delta}\right)$ for some constant $\bar{C}_3$ depending only on $C_{\eta}$, with probability at-least $1-\frac{\delta}{2}$ conditioned on $\mathcal{D}_T(R)$, we have:
$\hat{G} \succeq \frac{\sigma^2}{2} I$. In the definition of $\mathcal{E}_{T}(\kappa)$, we take $\kappa  = 2$. Therefore, we must have:
$$\mathbb{P}(\mathcal{E}_{T}(\kappa)\cap \mathcal{D}_T(R)) = \mathbb{P}(\mathcal{E}_{T}(\kappa)| \mathcal{D}_T(R))\mathbb{P}(\mathcal{D}_T(R)) \geq (1-\tfrac{\delta}{2})^2 \geq 1-\delta \,.$$
 We conclude the result from the equation above. 
\end{proof}

We now give the proof of Lemma~\ref{lem:normalized_noise}.

\subsection{Proof of Lemma~\ref{lem:normalized_noise}}

\begin{proof}
We invoke Theorem 1 in \cite{abbasi2011online} with $S_t = T\hat{N}_i$, $V = T\sigma^2 I $, $\bar{V}_t = V + T \hat{G}$. We know that $\langle\eta,e_i\rangle$ is $C_{\eta}\sigma^2$ sub-Gaussian. So, we take `$R$' in the reference to be $C_{\eta}\sigma^2$. Therefore, we conclude that with probability at least $1-\delta$:

\begin{equation}\label{eq:self_normalized_ref}
 \hat{N}_i^{\intercal}\bar{V}_t^{-1}\hat{N}_i \leq \frac{2C_{\eta}\sigma^2}{T^2} \log \left( \frac{\det(\bar{V}_t)^{1/2}\det(V)^{-1/2}}{\delta}\right) \,.
 \end{equation}

Under the event $\mathcal{D}_{T}(R,\kappa)$, we must have: $\bar{V}_t \preceq \sigma^2 T I + T R I $. This implies: 
\begin{equation}\label{eq:det_bound}
\det(\bar{V}_t)^{1/2}\det(V)^{-1/2} \leq (1+ \tfrac{R}{\sigma^2})^{\frac{d}{2}}
\end{equation}
Now, observe that under the event $\mathcal{D}_{T}(R,\kappa)$, $\hat{G}\succeq \frac{\sigma^2 I}{2}$. Therefore, $\bar{V}_{t} \preceq 3T\hat{G}$. This implies:
\begin{equation}\label{eq:PSD_bound}
\frac{1}{3T}\hat{N}_i^{\intercal}\hat{G}^{-1}\hat{N}_i \leq \hat{N}_i^{\intercal}\bar{V}_t^{-1}\hat{N}_i
\end{equation}

Combining Equations~\eqref{eq:self_normalized_ref},~\eqref{eq:det_bound} and~\eqref{eq:PSD_bound} and using Lemma~\ref{lem:dominant_event}, we conclude that with probability at-least $1-2\delta$, we have:

$$\hat{N}_i^{\intercal}\hat{G}^{-1}\hat{N}_i \leq \frac{6C_{\eta}\sigma^2}{T}\left[d \log(1+\tfrac{R}{\sigma^2}) + \log \tfrac{1}{\delta}\right]$$
Using union bound, we conclude the result. 
\end{proof}

\section{Parameter Recovery Lower Bounds for ReLU-AR Model}
\label{sec:relu_lb}
We will show that in the case of non-expansive activation functions, parameter recovery can be hard information theoretically. More specifically, even when $\phi = \relu$ and $\eta_t \sim \mathcal{N}(0,I)$, and $\|\A\|_2 = \rho < 1$, we will need exponentially many samples with respect to $d$ in order restimate $\A$ upto any vanishing accuracy. We do this via the two point method. Henceforth, we will assume that $d\geq 2$. Given $\epsilon > 0$, define with matrix $A(\epsilon) \in \mathbb{R}^{d\times d}$ as: 
\begin{equation}
A_{ij}(\epsilon) = \begin{cases}
\tfrac{1}{4} \quad \text{if} \quad i = j; i\leq d-1 \\
0 \quad \text{if} \quad i \neq j; i\leq d-1 \\
-\tfrac{\epsilon}{\sqrt{d-1}} \quad \text{if} \quad i \neq j; i = d \\
0 \quad \text{if} \quad i = j = d \\
\end{cases}
\end{equation}

We will consider $\glvar(A(\epsilon),\mathcal{N}(0,I),\relu)$ and $\glvar(A(0),\mathcal{N}(0,I),\relu)$ as the two points in the two point method. As usual, we will consider $X_0 = 0$ almost surely for the sake of convenience. But it will be shown that this process is rapidly mixing since we intend to pick $\epsilon$ such that $\|A(\epsilon)\| < \tfrac{1}{2}$, so all the results should easily extend to stationary sequences. We collect some useful results in the following lemma:

\begin{lemma}\label{lem:basics_lb}
\begin{enumerate}
\item $\|A(\epsilon)\| = \sqrt{\tfrac{1}{16} + \epsilon^2}$
\item Suppose $\epsilon$ is small enough such that $\|A(\epsilon)\| < 1$. Let $X_0,X_1,\dots,X_T \sim \glvar(A(\epsilon),\mathcal{N}(0,I),\relu)$ with $X_0 = 0$. For some universal constant $p_1 >0$, we have for every $i \leq d$, $t \geq 0$:
$$\mathbb{P}\left(\langle X_{t+1},e_i\rangle \geq \frac{1}{2}\biggr|X_{t}\right) \geq p_1$$ 
\end{enumerate}

\end{lemma}
\begin{proof}
\begin{enumerate}
\item Proof follows from elementary calculations.
\item  $X_{t+1} = \relu(A(\epsilon)X_t) + \eta_t$. Therefore, for every $i \leq d-1$, must have: $\langle X_{t+1},e_i\rangle \geq \langle\eta_t,e_i\rangle$. Therefore, we conclude that 

$$\mathbb{P}\left(\langle X_{t+1},e_i\rangle \geq \frac{1}{2}\biggr|X_t\right) \geq \mathbb{P}\left(\langle \eta_t,e_i\rangle \geq \tfrac{1}{2}\biggr|X_t\right) \geq p_1 >0 \,.$$
\end{enumerate}
\end{proof}

We will now show that the last co-ordinate $\langle X_{t+1},e_d\rangle$ is just noise with a large probability and hence, we cannot estimate the last row of $A(\epsilon)$ even with a large number of samples. Note that the event $\langle X_{t+1},e_d\rangle \neq \langle\eta_t,e_d\rangle$ is that same as the event $\langle a_d(\epsilon),X_t\rangle > 0$. Let $a_d(\epsilon)$ denote the last row of $A(\epsilon)$, in the form of a column vector.

\begin{lemma}\label{lem:high_probab_negativity}
Suppose $t \geq 2$. Then, for some universal constants $C_0,C_1 > 0$,
$$\mathbb{P}\left(\langle a_d(\epsilon),X_t\rangle > 0\right) \leq C_0 \exp(-C_1 d)\,.$$

\end{lemma} 

\begin{proof}

$\langle a_d,X_t\rangle > 0$ if and only if $\langle a_d, \relu(A(\epsilon)X_{t-1}) \rangle + \langle a_d, \eta_{t-1}\rangle > 0$. We first note that the first $d-1$ rows of $X_{t-1}$ are i.i.d. by the definition of $A(\epsilon)$. Therefore, we conclude using Lemma~\ref{lem:basics_lb} that:
$$\mathbb{P}\left(\langle a_d, \relu(A(\epsilon)X_{t-1}) \rangle \geq -c_0\epsilon \sqrt{d}\right) \leq \exp(-c_1 d)\,.$$ for some universal constants $c_0,c_1$. Now, $\langle a_d,\eta_{t-1} \rangle$ is distributed as $\mathcal{N}(0,\epsilon^2)$ and is independent of $X_{t-1}$. Therefore, with a large probability, $\langle a_d, \relu(A(\epsilon)X_{t-1}) \rangle$ takes a large negative value and by Gaussian concentration, $\langle a_d, \eta_{t-1}\rangle$ concentrates near $0$ . Therefore, using elementary calculations we conclude the result.
\end{proof}

In what follows, by  $\mathbf{X}^{\epsilon}_T := (X^{\epsilon}_0,\dots,X_T^{\epsilon})$ we denote the process such that $X_0^{\epsilon}  = 0$ and $(X^{\epsilon}_0,\dots,X_T^{\epsilon}) \sim \glvar(A(\epsilon),\mathcal{N}(0,I),\relu)$.
\begin{lemma}\label{lem:kl_calculation}
When $\epsilon < \tfrac{1}{4}$, for some universal constants $c_0,c_1 > 0$, we have:
$$\tv(\mathbf{X}^{\epsilon}_T, \mathbf{X}^{0}_T) \leq c_0\left( \epsilon + \epsilon \sqrt{T}\exp(-c_1 d)\right) \,.$$
\end{lemma}
\begin{proof}

 Following the proof of Lemma 36 in \cite{jain2021streaming}, we conclude that:

\begin{equation}\label{eq:KL_identity}
\kl(\mathbf{X}^{\epsilon}_T|| \mathbf{X}^{0}_T) = \frac{1}{2}\sum_{t=0}^{T-1}\mathbb{E}\|\relu(A(\epsilon)X_{t}^{\epsilon}) - \relu(A(0)X_{t}^{\epsilon})\|^2
\end{equation}
Note here that the expectation is with respect to the randomness in the trajectory $\mathbf{X}^{\epsilon}$. Applying the definition of $A(\epsilon)$ and $\relu$ to Equation~\eqref{eq:KL_identity}, we further simplify to show that:

\begin{equation}\label{eq:KL_identity_1}
\kl(\mathbf{X}^{\epsilon}_T|| \mathbf{X}^{0}_T) = \frac{1}{2}\sum_{t=0}^{T-1}\mathbb{E}|\relu(\langle a_d(\epsilon),X_{t}^{\epsilon}\rangle )|^2 = \frac{1}{2} \sum_{t=1}^{T-1}\mathbb{E}|\relu(\langle a_d(\epsilon),X_{t}^{\epsilon}\rangle )|^2  
\end{equation}

In the equation above, we have used the fact that $X^{\epsilon}_0 = 0$ almost surely. Now, when $t = 1$, we have $X^{\epsilon}_1 = \eta_0 \sim \mathcal{N}(0,I)$ almost surely. Therefore, $\mathbb{E}|\relu(\langle a_d(\epsilon),X_{1}^{\epsilon}\rangle )|^2 = \frac{\epsilon^2}{2}$. For $t \geq 2$, we will use Lemma~\ref{lem:high_probab_negativity} and Theorem~\ref{thm:process_concentration}. In this proof only, define $\mathcal{P}_t$ to be the event $\{\langle a_d,X_t \rangle > 0\}$. Therefore, we conclude that:

\begin{align}
\mathbb{E}|\relu(\langle a_d(\epsilon),X_{t}^{\epsilon}\rangle )|^2  &= \mathbb{E}\mathbbm{1}(\mathcal{P}_t)|\langle a_d(\epsilon),X_{t}^{\epsilon}\rangle|^2 \nonumber \\
&\leq \mathbb{E}\mathbbm{1}(\mathcal{P}_t)\| a_d(\epsilon)\|^2\|X_{t}^{\epsilon}\|^2 \nonumber \\
&=\mathbb{E}\mathbbm{1}(\mathcal{P}_t)\epsilon^2\|X_{t}^{\epsilon}\|^2 \nonumber \\
&\leq \epsilon^2 \sqrt{\mathbb{P}(\mathcal{P}_t)}\sqrt{\mathbb{E}\|X_t\|^4 }
\end{align}
From Lemma~\ref{lem:high_probab_negativity}, we conclude that $\mathbb{P}(\mathcal{P}_t) \leq C_0\exp(-C_1 d)$ and from Theorem~\ref{thm:process_concentration}, it is clear that the concentration inequalities hold we show that $\sqrt{\mathbb{E}\|X_t\|^4 } \leq C_2 d$ for some universal constant $C_2$. Combining the results above with Equation~\eqref{eq:KL_identity_1}, we obtain that for universal constants $C_0,C_1$:

\begin{equation}\label{eq:KL_bouond}
\kl(\mathbf{X}^{\epsilon}_T|| \mathbf{X}^{0}_T) \leq C_0 \left(\epsilon^2 d T \exp(-C_1d) +\epsilon^2\right)
\end{equation}
We then use Pinsker's inequality to conclude the result. 
\end{proof}

\subsection{Proof of Theorem~\ref{thm:relu_exp_lb}}

\begin{proof}
Let the notations below be as defined in the statement of the theorem. Since the minimax loss upper bounds any Bayesian loss, we will consider the uniform prior over $\{\glvar(A(\epsilon),\mathcal{N}(0,I),\relu),\glvar(A(0),\mathcal{N}(0,I),\relu)\}$. For any estimator $\mathcal{A}$ with input $X_{0},\dots,X_T$, we denote $\mathcal{A}(\mathbf{X})$ as its output. We use the notation $\mathbf{X}^{\epsilon}$ and $\mathbf{X}^{0}$ as defined in Lemma~\ref{lem:kl_calculation}. Therefore, 

\begin{align}
\loss(\Theta(\tfrac{1}{2}),T) &\geq \inf_{\mathcal{A}} \frac{1}{2}\loss(\mathcal{A},T,A(0)) + \frac{1}{2}\loss(\mathcal{A},T,A(\epsilon)) \nonumber \\
&= \frac{1}{2}\inf_{\mathcal{A}} \mathbb{E}\|\mathcal{A}(\mathbf{X}^{0})-A(0)\|_F^{2} + \mathbb{E}\|\mathcal{A}(\mathbf{X}^{\epsilon})-A(\epsilon)\|_F^{2} \label{eq:minimax_bayesian}
\end{align}
Where the expectation is over the respective trajectories $\mathbf{X}^{\epsilon}$ and $\mathbf{X}^{0}$. Now, from Lemma~\ref{lem:kl_calculation} and the coupling calculation of $\tv$ distance, we can define the trajectories $\mathbf{X}^{\epsilon}$ and $\mathbf{X}^{0}$ on a common probability space such that $\mathbb{P}(\mathbf{X}^{0} \neq \mathbf{X}^{\epsilon}) \leq \tv(\mathbf{X}^{0},\mathbf{X}^{\epsilon})$. Picking $\epsilon$ to be small enough such that $\tv(\mathbf{X}^{0},\mathbf{X}^{\epsilon}) \leq \frac{1}{2}$ (which is true for a choice of $\epsilon^2 = c_0\min(\frac{\exp(c_1 d)}{T},1)$ for universal constants $c_0,c_1$). Define the event $\mathcal{S}_T := \{\mathbf{X}^{0} = \mathbf{X}^{\epsilon}\} $, we conclude from Equation~\eqref{eq:minimax_bayesian} that:

\begin{align}
\loss(\Theta(\tfrac{1}{2}),T) &\geq \frac{1}{2}\inf_{\mathcal{A}} \mathbb{E}\mathbbm{1}(\mathcal{S}_T)\|\mathcal{A}(\mathbf{X}^{0})-A(0)\|_F^{2} + \mathbb{E}\mathbbm{1}(\mathcal{S}_T)\|\mathcal{A}(\mathbf{X}^{\epsilon})-A(\epsilon)\|_F^{2} 
\end{align}

Note that over the event $\mathcal{S}_T$, we must have $\mathcal{A}(\mathbf{X}^{\epsilon}) = \mathcal{A}(\mathbf{X}^{0})$ in the case of a deterministic estimator. In case of a random estimator, the proof is same once we observe that their distribution is same and hence can be coupled almost surely. By convexity, we must have: $\|a-b\|^2 + \|b-c\|^2 \geq \frac{1}{2}\|a-c\|^2$. Combining these considerations into the equation above, we have:

\begin{align}
\loss(\Theta(\tfrac{1}{2}),T) &\geq \frac{\mathbb{P}(\mathcal{S}_t)}{4} \|A(\epsilon) - A(0)\|_F^{2} \nonumber \\
&\geq \frac{\epsilon^2}{8}.
\end{align}

From the choice of $\epsilon^2$ above, we conclude the result. 

\end{proof}

\section{Proofs of Technical Lemmas}
\label{sec:technical_proofs}

\subsection{Proof of Lemma~\ref{lem:tail_integration}}
\begin{proof}
The proof follows from integrating the tails. Let $Z:= \exp(\lambda Y^2)$. For any $\gamma \in \mathbb{R}^{+}$, we have from the definition of sub-Gaussianity.

\begin{equation}
\mathbb{P}(Z \geq \gamma ) = \begin{cases} 1 \text{ if } \gamma \leq 1 \\
\mathbb{P}\left(|Y| \geq \sqrt{\frac{\log(\gamma)}{\lambda}} \right) \text{ if } \gamma > 1
\end{cases}
\end{equation}
Now,
\begin{align}
\mathbb{E}Z &= \int_{0}^{\infty} \mathbb{P}(Z \geq \gamma)d\gamma \nonumber \\
&= \int_{0}^{1} d\gamma + \int_{1}^{\infty} \mathbb{P}\left(|Y| \geq \sqrt{\tfrac{\log(\gamma)}{\lambda}} \right)d\gamma \nonumber \\
&\leq 1 + \int_{1}^{\infty} 2\exp\left(- \tfrac{\log(\gamma)}{2\nu^2\lambda}\right) d\gamma \nonumber \\
&= 1 + 2\int_{1}^{\infty}\gamma^{-\tfrac{1}{2\nu^2\lambda}}d\gamma \nonumber \\
&=1 + \frac{4\nu^2\lambda}{1-2\nu^2\lambda } \nonumber \\
&\leq 1 + 8\lambda\nu^2
\end{align}

\end{proof}

\subsection{Proof of Lemma~\ref{lem:square_correlation}}

\begin{proof}

We draw $\tilde{X}_{s} \sim \pi$, independent of $X_s$. We obtain $\tilde{X}_{s+k}$ by running the markov chain with the same noise sequence. i.e, $\tilde{X}_{s+k+1} = \phi(\A \tilde{X}_{s+k}) + \eta_{s+k}$. We then obtain $\tilde{X}_t$. Then, it is clear that:

\begin{align}
\langle X_t,x \rangle^2\langle X_s,x\rangle^2  &= \langle X_t-\tilde{X}_t + \tilde{X}_t,x \rangle^2\langle X_s,x\rangle^2 \nonumber \\
&\leq 2 \langle \tilde{X}_t,x \rangle^2\langle X_s,x\rangle^2 + 2 \langle X_t-\tilde{X}_t,x \rangle^2\langle X_s,x\rangle^2 \nonumber 
\end{align}   

Taking expectation on both sides and noting that $\tilde{X}_t$ is independent of $X_s$, we conclude:
\begin{equation}\label{eq:first_decoupling}
\mathbb{E}\langle X_t,x \rangle^2\langle X_s,x\rangle^2 \leq 2(x^{\intercal}Gx)^2 + 2\mathbb{E} \langle X_t-\tilde{X}_t,x \rangle^2\langle X_s,x\rangle^2
\end{equation}
By Assumption~\ref{assump:4}, we have: $\|X_t - \tilde{X}_t\|^2 \leq C^{2}_{\rho}\rho^{2(t-s)}\|X_s-\tilde{X}_s\|^2$. Plugging this into Equation~\eqref{eq:first_decoupling}, we conclude:

\begin{align}
\mathbb{E}\langle X_t,x \rangle^2\langle X_s,x\rangle^2 &\leq 2(x^{\intercal}Gx)^2 + 2\mathbb{E} \|x\|^2 C^{2}_{\rho}\rho^{2(t-s)}\|X_s-\tilde{X}_s\|^2\langle X_s,x\rangle^2 \nonumber \\
&\leq 2(x^{\intercal}Gx)^2 + 4\|x\|^2 C^{2}_{\rho}\rho^{2(t-s)}\mathbb{E} \left(\|X_s\|^2+\|\tilde{X}_s\|^2\right)\langle X_s,x\rangle^2 \nonumber \\
&= 2(x^{\intercal}Gx)^2 + 4\|x\|^2 C^{2}_{\rho}\rho^{2(t-s)}\left[\mathbb{E}\|X_s\|^2\langle X_s,x\rangle^2+ x^{\intercal}Gx\mathbb{E}\|\tilde{X}_s\|^2\right]\label{eq:master_fourth_moment}
\end{align}
We can evaluate $\mathbb{E}\|\tilde{X}_s\|^2$ from Theorems~\ref{thm:process_concentration} and~\ref{thm:heavy_tail_concentration}.

\begin{enumerate}
\item First we consider the Sub-Gaussian setting with $\|\A\| = \rho <1$ and $C_{\rho} = 1$. Fix $R > 0$. We will use the notation from Theorem~\ref{thm:process_concentration} below. We can then write, 

\begin{align}
\mathbb{E}\|X_s\|^2\langle X_s,x\rangle^2 &= \mathbb{E}\|X_s\|^2\langle X_s,x\rangle^2 \mathbbm{1}(\|X_s\|^2 \leq R) +  \mathbb{E}\|X_s\|^2\langle X_s,x\rangle^2 \mathbbm{1}(\|X_s\|^2 > R) \nonumber \\
&\leq \mathbb{E}R\langle X_s,x\rangle^2 \mathbbm{1}(\|X_s\|^2 \leq R) +  \mathbb{E}\|X_s\|^4 \mathbbm{1}(\|X_s\|^2 > R) \nonumber \\
&\leq Rx^{\intercal}Gx + \mathbb{E}\|X_s\|^4 \mathbbm{1}(\|X_s\|^2 > R) \nonumber \\
&\leq R x^{\intercal}Gx + \sqrt{\mathbb{E}\|X_s\|^8}\sqrt{ \mathbbm{P}(\|X_s\|^2 > R)} \label{eq:fourth_moment_bound_1}
\end{align}
From Theorem~\ref{thm:process_concentration} and Proposition 2.7.1 in \cite{vershynin2019high}, we show that $\mathbb{E}\|X_s\|^8 \leq C\left(\frac{dC_{\eta}\sigma^2}{1-\rho}\right)^4$ for some universal constant $C$. Again, taking $R = \frac{8 d C_{\eta}\sigma^2}{1-\rho} + \frac{2\log\frac{1}{\delta}}{\lambda^{*}} \leq \frac{24 d C_{\eta}\sigma^2\log(\frac{1}{\delta})}{1-\rho} $, we have: $\sqrt{ \mathbbm{P}(\|X_s\|^2 > R)}  \leq \delta$. We plug this into Equation~\eqref{eq:fourth_moment_bound_1}, take $\delta = \frac{(1-\rho)x^{\intercal}Gx}{d\sigma^2}$ after noting that $x^{\intercal}Gx \geq \sigma^2 \|x\|^2$ to show that:

\begin{equation}\label{eq:fourth_moment_bound_2}
\mathbb{E}\|X_s\|^2\langle X_s,x\rangle^2  \leq \bar{C} \frac{d\sigma^2}{1-\rho}x^{\intercal}Gx \log\left(\frac{d}{1-\rho}\right)
\end{equation}
Where $\bar{C}$ is a constant which depends only on $C_{\eta}$. Using Equation~\eqref{eq:fourth_moment_bound_2} in Equation~\eqref{eq:master_fourth_moment}, we conclude that:
\begin{equation}\label{eq:fourth_moment_final}
\mathbb{E}\langle X_t,x \rangle^2\langle X_s,x\rangle^2 \leq 2 (x^{\intercal}Gx)^2 + \bar{C}_1 C_{\rho}^2 \rho^{2(t-s)}  \frac{d\sigma^2}{1-\rho}x^{\intercal}Gx \log\left(\frac{d}{1-\rho}\right)
\end{equation}

\item 
From Equation~\eqref{eq:master_fourth_moment}, we directly conclude via Cauchy Schwarz inequality that:
$$\mathbb{E}\langle X_t,x \rangle^2\langle X_s,x\rangle^2 \leq 2 (x^{\intercal}Gx)^2 + 8\|x\|^2 C_{\rho}^2 \rho^{2(t-s)}\mathbb{E}\|X_s\|^4$$
We then use the bound on $\mathbb{E}\|X_s\|^4$ given in Theorem~\ref{thm:heavy_tail_concentration} to conclude the result. 
\end{enumerate}

\end{proof}

\subsection{Proof of Lemma~\ref{lem:well_conditioned_grammian}}

\begin{proof}
Consider $\hat{G} = \frac{1}{T}\sum_{\tau=0}^{T-1}X_{\tau}X_{\tau}^{\intercal}$ and consider the coupled process as in Defintion~\ref{def:coupled_proc}. We will divide the times into buffers of size $B$, with gaps of size $u$ and also consider related notation as given in Section~\ref{sec:sgd_rer_analysis}. Now, $\hat{G} \succeq \frac{B}{NS} \sum_{t = 1}^{N}G^{(t)}$, where $G^{(t)}$ is the empirical second moment matrix of the buffer $t$ given by $\hat{G}^{(t)} := \frac{1}{B}\sum_{i=0}^{B-1}X_i^{t,\intercal}X_i^{t}$. Now, consider the coupled second moment matrix defined on buffer $t$ given by $\tilde{G}^{(t)} := \frac{1}{B}\sum_{t=0}^{B-1}\tilde{X}_i^{t,\intercal}\tilde{X}_i^{t}$. By Lemma~\ref{lem:coupling_lemma}, we know that $\|\hat{G}^{(t)} - \tilde{G}^{(t)}\| \leq 4 N^{*} C_{\rho}\rho^{u}$. Where $M^{*} = \sup_{\tau \leq T}\max\left(\|X_{\tau}\|^2,\|\tilde{X}_{\tau}\|^2\right)$. Now, observe that by definition of the coupling, we have that $\tilde{G}_t$ are i.i.d. Combining the considerations above, and letting $B \geq u$ to conclude $\frac{B}{S} \geq \frac{1}{2}$
\begin{equation}\label{eq:basic_gramian_lower_bound}
\hat{G} \succeq \frac{B}{SN} \sum_{t = 1}^{N}\hat{G}^{(t)} \succeq \frac{1}{2N} \sum_{t = 1}^{N}\tilde{G}^{(t)} - 4M^{*}C_{\rho}\rho^{u}I
\end{equation}
Before proceeding further, we will give a high probability bound on $M^{*}$. By Markov's inequality, and Theorem~\ref{thm:heavy_tail_concentration}
$$\mathbb{P}(M^{*} > R) \leq \frac{2T\mathbb{E}\|X_{\tau}\|^2}{R} \leq \frac{2TC_{\rho}^2d\sigma^2}{(1-\rho)^2\alpha}\,.$$
Consider the event $\cdh_T(R) = \{M^{*} \leq R\}$. Letting $R = 
\frac{4TdC_{\rho}^2\sigma^2}{(1-\rho)^2 \delta}$ only in this proof, we conclude that:
$$\mathbb{P}(\mathcal{D}_T(R)) \geq 1- \frac{\delta}{2}$$
Recall that $G = \mathbb{E}X_{\tau}X_{\tau}^{\intercal}$. Now, as shown by item 2 in Lemma~\ref{lem:probable_contraction}, we have whenever $B \geq \frac{4C_{\rho}^6M_4}{(1-\rho)^4(1-\rho^2)\sigma^4}$

\begin{equation}\label{eq:buffer_wise_lb}
\mathbb{P}\left( x^{\intercal}\tilde{G}^{(t)} x \geq \frac{1}{2}x^{\intercal}Gx\right) \geq p_0 >0\,.
\end{equation}

Now, consider any arbitrary, fixed vector $x \in \mathbb{R}^d$. Using independence of $\hat{G}^{(t)}$ and Equation~\eqref{eq:buffer_wise_lb}, we conclude that for some universal constants $c_0,c_1$, we must have:
\begin{equation}
\mathbb{P}(\frac{1}{N} \sum_{t = 1}^{N}x^{\intercal}\tilde{G}^{(t)}x \leq c_0 x^{\intercal}Gx) \leq \exp(-c_1 N). 
\end{equation}
Rewriting the equation above by taking $x = G^{-1/2}y$, we have:
\begin{equation}\label{eq:fixed_vector_well_condition}
\mathbb{P}(\frac{1}{N} \sum_{t = 1}^{N}y^{\intercal}G^{-1/2}\tilde{G}^{(t)}G^{-1/2}y \leq c_0 \|y\|^2) \leq \exp(-c_1 N). 
\end{equation}
For simplicity of exposition, we will take $J := \frac{1}{N} \sum_{t = 1}^{N}G^{-1/2}\tilde{G}^{(t)}G^{-1/2} $ in the calculations below. Before proceeding with a bound on the operator norm, we will give a bound on $\|J\|$. Since $G \succeq \sigma^2 I$, we must have:
$\|J\| \leq \frac{M^{*}}{\sigma^2}$.

Now, we will apply an epsilon net argument. Let $\mathcal{N}_{\epsilon}$ be an $\epsilon$-net over $\mathcal{S}^{d-1}$. We can take $|\mathcal{N}_{\epsilon}| \leq (1+\frac{2}{\epsilon})^d$. 

$$\inf_{y \in \mathcal{S}^{d-1}}y^{\intercal}Jy \geq \inf_{y \in \mathcal{N}_{\epsilon}}y^{\intercal}Jy   - 2\|J\|\epsilon \,.$$

We let $\epsilon = \frac{c\sigma^2}{R}$ for some constant $c > 0$ small enough and $R$ as defined earlier in this proof. By union bound over $\mathcal{D}^c_{T}(R)$ and the event given in Equation~\eqref{eq:fixed_vector_well_condition}, we conclude that for some universal constant $c_2 > 0$ small enough:

\begin{equation}\label{eq:high_probab_lower_bound}
\mathbb{P}(\inf_{y\in\mathcal{S}^{d-1}} y^{\intercal}J y > c_2, M^{*} \leq R) \geq 1-\exp( C d\log(\tfrac{R}{\sigma^2}) - c_1 N) - \tfrac{\delta}{2}
\end{equation}

Now, using Equation~\eqref{eq:basic_gramian_lower_bound}, we conclude that:

$$G^{-1/2}\hat{G} G^{-1/2} \succeq  \frac{J}{2}  - 4M^{*}C_{\rho}\rho^{u}G^{-1}$$
Using the fact that $G^{-1} \preceq \frac{I}{\sigma^2}$ and using Equation~\eqref{eq:high_probab_lower_bound}, we conclude that whenever $T \geq C d\log(\tfrac{1}{\delta}) B \log(R/\sigma^2)$, $B \geq u$, $B \geq \frac{4C_{\rho}^6M_4}{(1-\rho)^4(1-\rho^2)\sigma^4}$ and $u \geq \frac{\log\left( \tfrac{R C_1 C_{\rho}}{\sigma^2}\right)}{\log\left(\tfrac{1}{\rho}\right)}$, we conclude that for some constant $c_0 > 0$ small enough, with probability atleast $1-\delta$, we have:

$$\hat{G} \succeq c_0G $$

\end{proof}

\subsection{Proof of Lemma~\ref{lem:algo-stability}}

First, we will obtain a crude upper bound on $\norm{\tilde{a}^{t-1}_{j}-\a}$ using Theorem~\ref{thm:buffer_norm_upper_bound}. That is, we want to show that $\norm{\tilde{a}^{t-1}_{j}-\a}$ does not grow too large with high probability. 
\begin{proposition}
\label{prop:crude_bnd}
 Let $\lambda_{\min}\equiv \lambda_{\min}(G)$. Conditional on $\cdt^{0,t-1}\cap \cap_{r=0}^{t-1}\ce^r_{0,B-1}$, with probability at least $1-N\delta$, for all $1\leq t\leq N$, all $1\leq j\leq B$ we have
\begin{\Ieee}{LLL}
\label{eq:crude_bnd}
\norm{\tilde{a}^{t-1}_j-\a} &\leq & \norm{a_0-a}+2\gamma B\sqrt{R\beta} C\left(d+\log\frac{N}{\delta}+\frac{1}{\zeta\gamma B\lambda_{\min}}\right)\Ieeen
\end{\Ieee}
where $C$ is constant depending only on $C_{\eta}$.

\end{proposition}

\begin{proof}
Let us start with the expression for $\tilde a^{t-1}_{j}-\a$
\begin{\Ieee}{LLL}
\label{eq:prop_crude_bnd_1}
(\tilde{a}^{t-1}_{j}-\a)^{\top}&=&(a_0-\a)^{\top}\left(\prod_{s=0}^{t-2}\Htt{s}{0}{B-1}\right)\Htt{t-1}{0}{j-1}+
2\gamma \sum_{i=0}^{j-1}\phi'(\tilde \xi^{t-1}_{-i})\Nt{t-1}{-i}\Xtttr{t-1}{-i}\Htt{t-1}{i+1}{j-1}\\
&& + 2\gamma\sum_{r=2}^t\sum_{i=0}^{B-1}\phi'(\tilde{\xi}^{t-r}_{-i})\Nt{t-r}{-i}\Xtttr{t-r}{-i}\Htt{t-r}{i+1}{B-1}\prodHtt{t}{r-1}\Ieeen
\end{\Ieee}
 We will work on the event $\cdt^{0,t-1}\cap \cap_{r=0}^{t-1}\ce^r_{0,B-1}$. It is clear from Equation~\eqref{eq:prop_crude_bnd_1} that: 
\begin{\Ieee}{LLL}
\label{eq:prop_crude_bnd_2}
\norm{\tilde{a}^{t-1}_{j}-\a}&\leq &\norm{a_0-\a}+2\gamma B\sqrt{R\beta}+2\gamma\sqrt{R\beta}B\sum_{r=2}^t\norm{\prodHtt{t}{r-1}}
\end{\Ieee}

 We use Theorem~\ref{thm:prod_buff_norm} (with appropriate constant $C>1$ to account for minor differences in indexing) to show that conditional on $\cdt^{0,t-1}\cap\cap_{r=0}^{t-1}\ce^r_{0,B-1}$, for fixed $t$, with probability at least $1-\delta$, for all $1\leq j\leq B$ 
\begin{\Ieee}{LLL}
\label{eq:prop_crude_bnd_4}
\norm{\tilde{a}^{t-1}_j-\a}\leq \norm{a_0-\a}+2\gamma B\sqrt{R\beta} C\left(d+\log\frac{N}{\delta}+\frac{1}{\zeta\gamma B\lambda_{\min}}\right)
\end{\Ieee}

Thus taking union bound we get that conditional on $\cdt^{0,t-1}\cap\cap_{r=0}^{t-1}\ce^r_{0,B-1}$  with probability at least $1-N\delta$, for all $1\leq t\leq N-1$ and all $1\leq j\leq B$
\begin{\Ieee}{LLL}
\label{eq:prop_crude_bnd_5}
\norm{\tilde{a}^{t-1}_j-\a}\leq \norm{a_0-\a}+2\gamma B\sqrt{R\beta} C\left(d+\log\frac{N}{\delta}+\frac{1}{\zeta\gamma B\lambda_{\min}}\right)
\end{\Ieee}

\end{proof}



\begin{proof}[Proof of Lemma~\ref{lem:algo-stability}]
On the event $\ce^{r}_{0,j}\cap \cdt^{r,N-1}$, we note the following inequalities
\begin{\Ieee}{LLL}
& \bar{\tilde{a}}^s_i=\tilde{a}^s_i \, 0\leq s<r,\, 0\leq i\leq B-1 \Ieeen \\
& \bar{\tilde{a}}^r_0=\tilde{a}^r_0 \Ieeen
\end{\Ieee}

\begin{\Ieee}{LLL}
 \norm{\bar{\tilde{a}}^s_i-\tilde{a}^s_i}\leq \begin{cases}
 4i\gamma\sqrt{R\beta}+\sum_{k=0}^{i-1} 4\gamma R\norm{\tilde{a}^r_{k}-\a}, & s=r,\,1\leq i\leq j\\
 4(j+1)\gamma\sqrt{R\beta}+\sum_{k=0}^{j-1} 4\gamma R\norm{\tilde{a}^r_{k}-\a}, s=r,\,& j+1\leq i\leq B-1\\
 4(j+1)\gamma\sqrt{R\beta}+\sum_{k=0}^{j-1} 4\gamma R\norm{\tilde{a}^r_{k}-\a}, r<s,\,& 0\leq i\leq B-1\\ 
 \end{cases}\Ieeen\label{eq:algo_stab_3}
\end{\Ieee}

The result then follows from an application of Proposition~\ref{prop:crude_bnd} with $\delta$ chosen as in \ref{subsubsec:parameter_setting}
\end{proof}

\subsection{Proof of Lemma~\ref{lem:coupled_iterate_replacement}}
We first state and prove the following result:
\begin{lemma}\label{lem:bounded_iterates}
Let $R_{\max} := \sup_{\tau \leq T}(\norm{X_{\tau}}^2,\norm{\tilde{X}}^2)$
and suppose $\gamma \leq \frac{1}{2R_{\max}}$.  For every $t \in [N]$ and $i \in [B]$ we have:
$$\|a^{t}_i\| \leq 2\gamma R_{\max} T  \,.$$
\end{lemma}
\begin{proof}
Let the row under consideration be the $k$-th row and $e_k$ be the standard basis vector. Consider the $\sgdber$ iteration: 
\begin{align}
 a^{t}_{i+1} &=a^{t}_{i}-2\gamma\left(\phi(\langle a^{t}_i, X^{t}_{-i}\rangle)-X^{t}_{-(i-1)}\right)X^{t}_{-i} \nonumber \\
&= (I-2\gamma \zeta_{t,i}\Xt{t}{-i}\Xt{t,\top}{-i} ) a^{t}_i + 2\gamma \langle \Xt{t}{-(i-1))},e_k\rangle\Xt{t}{-i}  
\end{align}
Where $\zeta_{t,i} := \frac{\phi(\langle a^{t}_i, X^{t}_{-i}\rangle) }{\langle a^{t}_i, X^{t}_{-i}\rangle} \in [\zeta,1]$ exists in a weak sense due to our assumptions on $\phi$. 
Observe that for our choice of $\gamma$, we have $\|(I-2\gamma \zeta_{t,i}\Xt{t}{-i}\Xt{t,\top}{-i} ) \| \leq 1$ and $\|\langle\Xt{t}{-(i-1)},e_k\rangle\Xt{t,\top}{-i} \| \leq R_{\max}$. Therefore, triangle inequality implies:
$$\|a^{t}_{i+1}\| \leq \|a^{t}_i \|+ 2\gamma R_{\max}$$
We conclude the bound in the Lemma.

\end{proof}

\begin{proof}[Proof of Lemma~\ref{lem:coupled_iterate_replacement}]

Let the row under consideration be the $k$-th row and $e_k$ be the standard basis vector.
\begin{align}
a^{t}_{i+1} &= a^{t}_i - 2 \gamma (\phi(\langle a^{t}_{i},\Xt{t}{-i}\rangle)-\langle e_k,\Xt{t}{-(i-1)}\rangle)\Xt{t}{-i} \nonumber \\
&= a^{t}_i - 2 \gamma (\phi(\langle a^{t}_{i},\Xtt{t}{-i}\rangle)-\langle e_k,\Xtt{t}{-(i-1)}\rangle)\Xtt{t}{-i} + \Delta_{t,i}
\end{align}
Where  $$\Delta_{t,i} := 2\gamma\left(\phi(\langle a^{t}_{i},\Xtt{t}{-i}\rangle) \Xtt{t}{-i}  - \phi(\langle a^{t}_{i},\Xt{t}{-i}\rangle) \Xt{t}{-i} \right) + 2\gamma\left(\langle\Xt{t}{-(i-1)},e_k\rangle\Xt{t}{-i} - \langle\Xtt{t}{-(i-1)},e_k\rangle\Xtt{t}{-i}\right) \,.$$
Using Lemmas~\ref{lem:bounded_iterates} and~\ref{lem:coupling_lemma}, we conclude that:
$$\|\Delta_{t,i}\| \leq (16\gamma^2R_{\max}^2T + 8\gamma R_{\max})\rho^u $$
Using the recursion for $\tilde{a}_{i}^{t}$, we conclude:
\begin{align}
a^{t}_{i+1} -\tilde{a}^{t}_{i+1}  &=(I - 2\gamma \tilde{\zeta}_{t,i}\Xtt{t}{i}\Xtt{t,\top}{i}) (a^{t}_i  - \tilde{a}^{t}_{i}) + \Delta_{t,i} \nonumber \\
\implies \norm{a^{t}_{i+1} -\tilde{a}^{t}_{i+1}} &\leq \norm{a^{t}_i  - \tilde{a}^{t}_{i}}\norm{(I - 2\gamma \tilde{\zeta}_{t,i}\Xtt{t}{i}\Xtt{t,\top}{i})} + (16\gamma^2R_{\max}^2T + 8\gamma R_{\max})\rho^u
 \nonumber \\
\implies \norm{a^{t}_{i+1} -\tilde{a}^{t}_{i+1}} &\leq \norm{a^{t}_i  - \tilde{a}^{t}_{i}} +  (16\gamma^2R_{\max}^2T + 8\gamma R_{\max})\rho^u \label{eq:coupling_distance_recursion}
\end{align}
In the first step, $\tilde{\zeta}_{t,i} :=\frac{ \phi(\langle a^{t}_{i},\Xtt{t}{-i}\rangle) - \phi(\langle \tilde{a}^{t}_{i},\Xtt{t}{-i}\rangle}{ \langle a^{t}_{i},\Xtt{t}{-i}\rangle - \langle \tilde{a}^{t}_{i},\Xtt{t}{-i}\rangle} \in [\zeta,1]$.
In the last step we have used the fact that under the conditions on $\gamma$, we must have $\norm{(I - 2\gamma \tilde{\zeta}_{t,i}\Xtt{t}{i}\Xtt{t,\top}{i})}\leq 1$. We conclude the statement of the lemma from Equation~\eqref{eq:coupling_distance_recursion}. 
\end{proof}

\subsection{Proof of Claim~\ref{claim:cr_bnd_diff_buf}}

\begin{proof}
Let $r_2>r_1$. As in proof of Claim~\ref{claim:cr_bnd_comm_buf}, let $\cro'$ denote the resampled version of $\cro$ obtained by re-sampling $\eta^{t-r_1}_{-j_1}$ i.e.,
\begin{\Ieee}{LLL}
\label{eq:cr_bnd_diff_buf_1}
\cro'(t,r_1,r_2,j_1,j_2)&:=& 4\gamma^2\Nt{t-r_1}{-j_1}\Nt{t-r_2}{-j_2} \cR^{t-r_1}_{-j_1}\left[\Xtttr{t-r_2}{-j_2}\Htt{t-r_2}{j_2+1}{B-1}\prodHtt{t}{r_2-1}\cdot\right.\\
&&\left. \prodHtttr{t}{r_1-1}\Htttr{t-r_1}{j_1+1}{B-1}\Xtt{t-r_1}{-j_1}\right]\\
&=& 4\gamma^2\Nt{t-r_1}{-j_1}\Nt{t-r_2}{-j_2}\Xtttr{t-r_2}{-j_2}\left(\Htt{t-r_2}{j_2+1}{B-1}\right)\left(\prod_{s=r_2-1}^{r_1+1}\Htt{t-s}{0}{B-1}\right)\cdot\\
&&\cR^{t-r_1}_{-j_1}\prodHtt{t}{r_1}  \cR^{t-r_1}_{-j_1}\prodHtttr{t}{r_1-1} \cR^{t-r_1}_{-j_1}\left(\Htttr{t-r_1}{j_1+1}{B-1}\right) \Xtt{t-r_1}{-j_1}\\
\Ieeen
\end{\Ieee}
Here we have used the fact that $\cR^{t-r_1}_{-j_1}$ does not affect the buffers up to $t-r_1-1$ and only $\tilde{X}$s that are affected are in the term $\Htt{t-r_1}{0}{j_1-1}$. Like in Claim~\ref{claim:cr_bnd_comm_buf}, notice that 
$$\Ex{\sum_{r_2>r_1}\sum_{j_1,j_2}\cro'(t,r_1,r_2,j_1,j_2)}=0$$
Applying Lemma~\ref{lem:cr_base_recursion}, we conclude that:
\begin{\Ieee}{LLL}
\label{eq:cr_bnd_diff_buf_2}
&&\sum_{r_2>r_1}\sum_{j_1,j_2}\cro'(t,r_1,r_2,j_1,j_2) \\
&=&2\gamma\sum_{r_1=1}^{t-1}\sum_{j_1=0}^{B-1}(\tilde{a}^{t-r_1-1,v}_B)^{\top}\cR^{t-r_1}_{-j_1}\left(\Htt{t-r_1}{0}{B-1}\right)\cR^{t-r_1}_{-j_1}\prodHtt{t}{r_1-1}\cdot \\
&&\cR^{t-r_1}_{-j_1}\prodHtttr{t}{r_1-1}\cR^{t-r_1}_{-j_1}\left(\Htttr{t-r_1}{j_1+1}{B-1}\right)\Xtt{t-r_1}{-j_1}\Nt{t-r_1}{-j_1}\Ieeen
\end{\Ieee}

We cannot continue our analysis like in Claim~\ref{claim:cr_bnd_comm_buf} because due to resampling of $\Nt{t-r_1}{-j_1}$, $\Htt{t-r_1}{0}{B-1}$ changes not just because of the iterates $\tilde a^{t-r_1}_i$ but also due to $\tilde{X} \to \bar{\tilde{X}}$.



Further 
\begin{\Ieee}{LLL}
\label{eq:cr_bnd_diff_buf_3}
\Ex{\sum_{r_2>r_1}\sum_{j_1,j_2}\cro'(t,r_1,r_2,j_1,j_2)\indt{0}{t-r_1-1}\indt{t-r_1+1}{t-1}1\left[\cdt^{t-r_1}_{-j_1}\right]}=0\Ieeen
\end{\Ieee}

Next we have simple lemma
\begin{lemma}
\label{lem:cr_bnd_diff_buf_cond}
Consider for each $(r_1,j_1)$, the re-sampling operator $\cR^{t-r_1}_{-j_1}$
\begin{\Ieee}{LLL}
\label{eq:cr_bnd_diff_buf_cond}
\abs{\Ex{\sum_{r_2> r_1}\sum_{j_1,j_2}\cro(t,r_1,r_2,j_1,j_2)\indt{0}{t-1}}}\leq 4\gamma^2 R \frac{(Bt)^2}{2} C_{\eta}\sigma^2\prbndsq+\\
\abs{\Ex{\sum_{r_2> r_1}\sum_{j_1,j_2}\cro(t,r_1,r_2,j_1,j_2)\indt{0}{t-1}\cR^{t-r_1}_{-j_1}1\left[\cdt^{t-r_1}_{-0}\right]}}
\Ieeen
\end{\Ieee}
\end{lemma}
\begin{proof}
We have
\begin{\Ieee}{LLL}
\label{eq:cr_bnd_diff_buf_cond_1}
\indt{0}{t-1}=\indt{0}{t-1}\cR^{t-r_1}_{-j_1}1\left[\cdt^{t-r_1}_{-0}\right]+\indt{0}{t-1}\cR^{t-r_1}_{-j_1}1\left[\cdt^{t-r_1,C}_{-0}\right]\Ieeen
\end{\Ieee}

Hence
\begin{\Ieee}{LLL}
\label{eq:cr_bnd_diff_buf_cond_2}
&&\abs{\Ex{\sum_{r_2> r_1}\sum_{j_1,j_2}\cro(t,r_1,r_2,j_1,j_2)\indt{0}{t-1}}}\\
&\leq & \abs{\Ex{\sum_{r_2> r_1}\sum_{j_1,j_2}\cro(t,r_1,r_2,j_1,j_2)\indt{0}{t-1}\cR^{t-r_1}_{-j_1}1\left[\cdt^{t-r_1}_{-0}\right]}}+\\
&& 4\gamma^2 R \frac{(Bt)^2}{2} C_{\eta}\sigma^2\prbndsq\Ieeen
\end{\Ieee}
where we used $\cR^{t-r_1}_{-j_1}1\left[\cdt^{t-r_1,C}_{-0}\right]$ is identically distributed as $1\left[\cdt^{t-r_1,C}_{-0}\right]$ and hence $\Ex{\cR^{t-r_1}_{-j_1}1\left[\cdt^{t-r_1,C}_{-0}\right]}\leq \frac{1}{T^{\alpha}}$

\end{proof}

So, based on the above lemma, we focus on bounding 
$$\abs{\Ex{\sum_{r_2> r_1}\sum_{j_1,j_2}\cro(t,r_1,r_2,j_1,j_2)\indt{0}{t-1}\cR^{t-r_1}_{-j_1}1\left[\cdt^{t-r_1}_{-0}\right]}}$$

Now notice that 
\begin{\Ieee}{LLL}
\label{eq:cr_bnd_diff_buf_4}
&&\Ex{\sum_{r_2>r_1}\sum_{j_1,j_2}\cro'(t,r_1,r_2,j_1,j_2)\cdot \right.\\
&&\left.\indt{0}{t-r_1-1}\indt{t-r_1+1}{t-1}1\left[\cdt^{t-r_1}_{-j_1}\right]\cR^{t-r_1}_{-j_1}1\left[\cdt^{t-r_1}_{-0}\right]}\\
&=&0\Ieeen
\end{\Ieee}

Hence
\begin{\Ieee}{LLL}
\label{eq:cr_bnd_diff_buf_5}
\Ex{\sum_{r_2>r_1}\sum_{j_1,j_2}\cro'(t,r_1,r_2,j_1,j_2)\indt{0}{t-1}\cR^{t-r_1}_{-j_1}1\left[\cdt^{t-r_1}_{-0}\right]}=0-\\
\Ex{\sum_{r_2>r_1}\sum_{j_1,j_2}\cro'(t,r_1,r_2,j_1,j_2)\indt{0}{t-r_1-1}\indt{t-r_1+1}{t-1}\cdot  \right.\\
\left.1\left[\cdt^{t-r_1}_{-j_1}\right]1\left[\cup_{i=0}^{j_1-1}\cct^{t-r,C}_{-i}\right]\cR^{t-r_1}_{-j_1}1\left[\cdt^{t-r_1}_{-0}\right]}\Ieeen\\
\end{\Ieee}

Thus
\begin{equation}
\label{eq:cr_bnd_diff_buf_6}
\abs{\Ex{\sum_{r_2>r_1}\sum_{j_1,j_2}\cro'(t,r_1,r_2,j_1,j_2)\indt{0}{t-1}\cR^{t-r_1}_{-j_1}1\left[\cdt^{t-r_1}_{-0}\right]}}
\leq  2\gamma^2 R \frac{(Bt)^2}{2}C_{\eta}\sigma^2\prbndsq
\end{equation}

Now, similar to lemma~\ref{lem:noise_resample_bound}, on the event $\ce^{r_1}_{0,j_1}\cap\cdt^{0,t-1}\cap \cA^{t-1}$ we have:
\begin{\Ieee}{LLL}
\label{eq:cr_bnd_diff_buf_7}
\norm{\prodHtt{t}{r_1-1}-\cR^{t-r_1}_{-j_1}\prodHtt{t}{r_1-1}}
\leq C Bt \norm{\phi''}\gamma^2 R^3 B\frac{\sqrt{\beta}}{\zeta\lambda_{\min}}\Ieeen
\end{\Ieee}

Next, similar to lemma~\ref{lem:noise_resample_bound} for $\gamma R\leq \frac{1}{2}$, on the event $\cdt^{t-r_1}_{-0}\cap \cap_{i=0}^{B-1}\left\{\norm{\cR^{t-r_1}_{j_1}\Xtt{t-r_1}{-i}}^2\leq R\right\} $ we have
\begin{\Ieee}{LLL}
\label{eq:cr_bnd_diff_buf_8}
\norm{\Htt{t-r_1}{0}{B-1}-\cR^{t-r_1}_{-j_1}\left(\Htt{t-r_1}{0}{B-1}\right)}\leq 4\gamma RB \Ieeen
\end{\Ieee}

Finally we can bound the norm of the expected difference of sums of $\cro$ and $\cro'$ using lemma~\ref{lem:cr_base_recursion} and \eqref{eq:cr_bnd_diff_buf_2} as
\begin{\Ieee}{LLL}
\label{eq:cr_bnd_diff_buf_9}
&&\abs{\Ex{\sum_{r_2>r_1}\sum_{j_1,j_2}(\cro-\cro')\indt{0}{t-1}\cR^{t-r_1}_{-j_1}1\left[\cdt^{t-r_1}_{-0}\right] 1\left[\cap_{s=0}^{t-1}\cap_{i=0}^{B-1}\ce^s_i\right]1\left[\cA^{t-1}\right]}} \\
&\leq& 2\gamma\Ex{\sum_{r_1=1}^{t-1}\sum_{j_1} \sqrt{R}|\Nt{t-r_1}{-j_1}| \left[\norm{\tilde a^{t-r_1-1,v}_B}\indt{0}{t-r_1-1}\right]\cdot\right.\\
&&\left.\left(C\norm{\phi''}\gamma^2 T R^3 B\frac{\sqrt{\beta}}{\zeta\lambda_{\min}} +C\gamma RB\right) }\\
&\leq & \left(C\norm{\phi''}\gamma^3 T^2 R^3 B\frac{\sqrt{\beta}}{\zeta\lambda_{\min}} +C\gamma^2 TRB\right)\sqrt{RC_{\eta}\sigma^2}\sqrt{\sup_{s\leq N-1}\Ex{\norm{\tilde a^{s,v}_B}^2\indt{0}{s}}}\\
\Ieeen
\end{\Ieee}

Thus
\begin{\Ieee}{LLL}
\label{eq:cr_bnd_diff_buf_10}
&&\abs{\Ex{\sum_{r_2>r_1}\sum_{j_1,j_2}(\cro-\cro')\indt{0}{t-1}\cR^{t-r_1}_{-j_1}1\left[\cdt^{t-r_1}_{-0}\right]}} \\
&\leq & C\left(\norm{\phi''}\gamma^3 T^2 R^3 B\frac{\sqrt{\beta}}{\zeta\lambda_{\min}} +\gamma^2 TRB\right)\sqrt{RC_{\eta}\sigma^2}\sqrt{\sup_{s\leq N-1}\Ex{\norm{\tilde a^{s,v}_B}^2\indt{0}{s}}}\\
&&+ C(Bt)^2 \left[\gamma^2 R C_{\eta}\sigma^2 \left(\sqrt{\Pb{\cup_{s=0}^{N-1}\cup_{i=0}^{B-1} \ce^{s,C}_i}}+\sqrt{\Pb{\cA^{t-1,C}}}\right)\right]\\
&\leq & C\left(\norm{\phi''}\gamma^3 T^2 R^3 B\frac{\sqrt{\beta}}{\zeta\lambda_{\min}} +\gamma^2 TRB\right)\sqrt{RC_{\eta}\sigma^2}\sqrt{\sup_{s\leq N-1}\Ex{\norm{\tilde a^{s,v}_B}^2\indt{0}{s}}}\\
&&+ C(Bt)^2 \gamma^2 R C_{\eta}\sigma^2 \prbndsq\\
\Ieeen
\end{\Ieee}

Combining everything we conclude the claim.
%

\end{proof}

\end{document}